%% file: main.tex
\documentclass[11pt]{article}

\usepackage{mathpazo}
\usepackage[margin=1in]{geometry}
\usepackage{amsmath,amssymb,amsthm,mathtools}
\usepackage{bm}
\usepackage{comment}
\usepackage{lipsum}
\usepackage[T1]{fontenc}
\usepackage[colorlinks=true,citecolor=blue,linkcolor=blue,urlcolor=blue]{hyperref}
\usepackage[natbib=true,minalphanames=3,maxcitenames=2,maxbibnames=99,style=alphabetic]{biblatex}
\usepackage{subcaption}
\usepackage{xcolor}
\usepackage{booktabs}
\usepackage{multirow}
\usepackage{cleveref}
\usepackage[toc,page,header]{appendix}
\usepackage{minitoc}
\usepackage{lipsum}
\usepackage{algorithm}
\usepackage{thmtools}
\usepackage[noend]{algpseudocode}
\usepackage{grffile}
\usepackage{todonotes}

\definecolor{mydarkblue}{rgb}{0,0.08,0.85}
\definecolor{mylightblue}{rgb}{0.06,0.56,1.0}
\definecolor{mylightorange}{rgb}{1.0,0.62,0.12}
\definecolor{mylightred}{rgb}{0.99,0.00,0.04}
\hypersetup{colorlinks=true,
linkcolor=mydarkblue,
citecolor=mydarkblue,
filecolor=mydarkblue,
urlcolor=mydarkblue,
bookmarksopen=true,
linktoc=page}

\newtheorem{proposition}{Proposition}
\newtheorem{definition}{Definition}
\newtheorem{theorem}{Proposition}

\newcommand{\cdmi}{\ensuremath{\Theta^{(i)}}}
\newcommand{\cdma}{\ensuremath{\Theta^{(1)}}}
\newcommand{\dma}{\ensuremath{\theta^{(1)}}}

\newcommand{\cdmb}{\ensuremath{\Theta^{(2)}}}
\newcommand{\dmb}{\ensuremath{\theta^{(2)}}}

\newcommand{\resdmab}{\ensuremath{\theta^{(1\setminus 2)}}}
\newcommand{\resdmba}{\ensuremath{\theta^{(2\setminus 1)}}}

\newcommand{\modeleval}[2]{f_{\mathcal{A}}({#1};{#2})}

\newcommand{\mask}[1]{\bm{1}_{#1}}
\newcommand{\lr}[1]{\left({#1}\right)}

\newcommand{\idt}{IDT}

\title{\textsc{ModelDiff}: A Framework for Comparing \\ Learning Algorithms}
\author{Harshay Shah\footnote{Equal contribution.},\ \
Sung Min Park\footnotemark[1],\ \ 
Andrew Ilyas\footnotemark[1],\ \ 
Aleksander M\k{a}dry \\
\texttt{\{\url{harshay},\url{sp765},\url{ailyas},\url{madry}\}@mit.edu} \\
Massachusetts Institute of Technology}
\date{}

\begin{document}
\setcounter{tocdepth}{2}
\doparttoc %
\renewcommand\ptctitle{}
\faketableofcontents %

\maketitle
\begin{abstract}
    \input{sections/abstract}
\end{abstract}

\clearpage
\section{Introduction}
\input{sections/intro}

\section{Preliminaries and Setup}
\label{sec:primer}
\input{sections/problem_setup}

\section{Comparing learning algorithms with \textsc{ModelDiff}}
\label{sec:approach}
\input{sections/approach}

\section{Case studies: Applying \textsc{ModelDiff}}
\label{sec:examples}
\input{sections/examples}

\section{Discussion}
\label{sec:more_discussion}
\input{sections/more_discussion}

\section{Related work}
\label{sec:discussion}
\input{sections/discussion}

\section{Conclusion}
\label{sec:conclusion}
\input{sections/conclusion}

\section{Acknowledgements}
\label{sec:acknowledgements}
\input{sections/ack.tex}

\clearpage
\printbibliography

\clearpage
\appendix
\addcontentsline{toc}{section}{Appendix} %
\renewcommand\ptctitle{Appendices}
\part{}
\parttoc
\clearpage

\input{appendices/app_explanation}

\clearpage

\input{appendices/setup}

\clearpage

\input{appendices/more_related_work}
\clearpage

\input{appendices/analysis}
\clearpage

\input{appendices/counterfactuals}

\clearpage

\input{appendices/baseline}

\end{document}

%% file: sections/abstract.tex
We study the problem of {\em (learning) algorithm comparison}, 
where the goal is to find differences between models trained 
with two different learning algorithms.
We begin by formalizing this goal as one of finding 
{\em distinguishing feature transformations}, i.e., 
input transformations that change the predictions of models 
trained with one learning algorithm but not the other.
We then present \textsc{ModelDiff}, a method 
that leverages the datamodels framework \citep{ilyas2022datamodels}
to compare learning algorithms based on how they use their training data.
We demonstrate \textsc{ModelDiff} through three case studies, comparing models trained 
with/without data augmentation,
with/without pre-training,
and with different SGD hyperparameters.
Our code is available at \url{https://github.com/MadryLab/modeldiff}.

%% file: sections/intro.tex
Building a machine learning model involves making a number of design choices.
Indeed,
even after choosing a dataset,
one must decide on a model architecture,
an optimization method,
and a data augmentation pipeline.
It is these design choices together that define the {\em learning algorithm}, i.e.,
the function mapping training datasets to machine learning models.

One of the key reasons why these design choices matter to us is that---even when they do not directly affect accuracy---they determine the
{\em biases} of the resulting models.
For example,
\citet{hermann2020origins} find significant variation in shape bias \citep{geirhos2019imagenet}
across a group of ImageNet models that vary in accuracy by less than 1\%.
Similarly, \citet{liu2022same} find that language models with the same loss
can have drastically different implicit biases and as a result,
different performances when used for transfer learning.

Motivated by this, we develop a framework for
comparing learning algorithms.
Our framework is {\em general} in the sense that one can use it
to study {\em any} pair of algorithms
{\em without} a prior hypothesis on what the difference between them is. As a result, it enables us to precisely dissect the impact of particular design choices.
Specifically, our two main contributions are:
\begin{itemize}
    \item[(a)] \textbf{A precise, quantitative definition of
    the algorithm comparison problem.}
    We introduce the problem of {\em (feature-based) learning algorithm comparison}.
    Given two learning algorithms, the (intuitive) goal here
    is to find features used by one learning algorithm but not by the other.
    We formalize this goal by stating it as one of
    finding transformations in input space that induce different behavior from 
    (i.e., that {\em distinguish}) models trained by the two learning algorithms
    (Section \ref{ssec:motivation}). 
    \item[(b)] \textbf{A method for comparing any two algorithms.}
    We propose
    a method, \textsc{ModelDiff},
    for comparing algorithms in terms of {\em how they use the training data}.
    The key tool we leverage here is {\em datamodeling} \citep{ilyas2022datamodels},
    which
    allows us to quantify the impact of each training example
    on a specific (single) model prediction (see Section \ref{ssec:datamodels}).
    Intuitively, two models that use different features to arrive at their final 
    predictions should differ in what training examples they rely 
    on\footnote{
        For example, suppose we were comparing a texture-biased model to a
        shape-biased one: when faced with a test example containing, e.g., a cat in a
        unique pose, the texture-biased model will rely on examples from the
        training set with similar texture (e.g., fur), whereas the shape-biased
        model will rely more on examples in a similar pose.}. 
    We translate this intuition into a method with formal guarantees 
    in Section~\ref{sec:approach}.
    The output of our method is a set of {\em distinguishing subpopulations},
    groups of test examples for which the two algorithms systematically differ
    in {\em how} they make predictions.
\end{itemize}

We tie our method (b) back to our formal definition of algorithm comparison in (a)
by looking for a pattern within each distinguishing subpopulation,
and using it to design a candidate input transformation (Figure \ref{fig:headline} middle).
We then ``verify'' this transformation (Figure \ref{fig:headline}
right) through counterfactual analysis.

\begin{figure}[h]
    \centering
    \includegraphics[width=\textwidth]{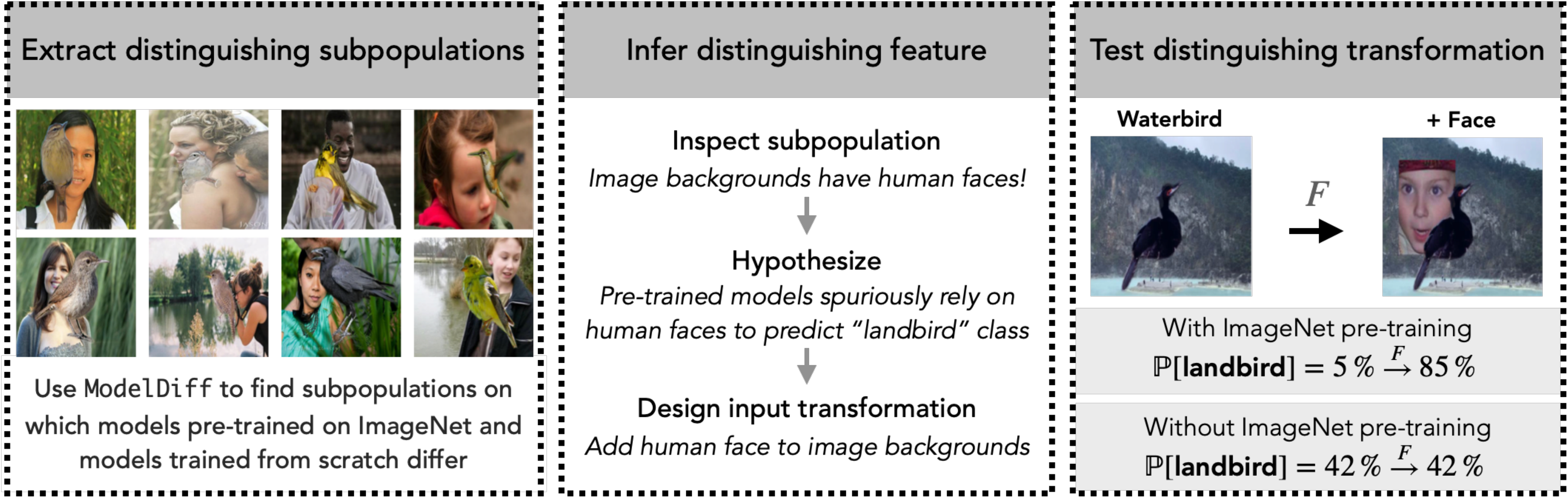}
    \caption{A demonstration of our approach. Our method
    (Section \ref{sec:approach}) automatically extracts coherent subpopulations
    of images (left) on which models differ in terms of what features they use.
    Then, an interpreter---whether human or automated---uses these subpopulations to design
    a {\em feature transformation} that they think will distinguish the models (middle).
    Finally, we verify the validity of this feature transformation via counterfactual analysis (right).}
    \label{fig:headline}
\end{figure}

We illustrate the utility of our framework through three case studies (\Cref{sec:examples}),
which are motivated by typical choices one needs to make within machine learning pipelines, namely:
\begin{itemize}
    \item {\bf Data augmentation:} We compare classifiers trained
    with and without data augmentation on the
    \textsc{Living17} \citep{santurkar2021breeds} dataset.
    We show that models trained with data augmentation
    are more prone to picking up instances of co-occurrence bias and texture bias compared to models trained without data augmentation.
    For instance,
    adding a spider web
    to random images increases the ``spider''
    confidence of models trained with (without) augmentation by 15\% (1\%) on
    average.

    \item {\bf ImageNet pre-training:} We compare classifiers first
    pre-trained on ImageNet
    \cite{deng2009imagenet,russakovsky2015imagenet}
    then fine-tuned on \textsc{Waterbirds} task \citep{sagawa2020distributionally} to classifiers
    trained from scratch on \textsc{Waterbirds}.
    We demonstrate that pre-training can either suppress or amplify specific spurious correlations.
   For example, adding a yellow patch to random images changes models' 
   average confidence in the ``landbird'' label by +12\% (-4\%) when training from scratch (pre-training) .
  Meanwhile, adding a human face to the background changes the ``landbird'' confidence of pre-trained (from-scratch) models by +4\% (-1\%).
    \item {\bf SGD hyperparameters:}
    Finally, we compare classifiers trained on \textsc{Cifar-10}
    \citep{krizhevsky2009learning}
    with different choices of learning rates and batch sizes.
    Our analysis pinpoints subtle differences in model behavior induced by small changes to these hyperparameters.
	For example, adding a small pattern that resembles windows to random images
	increases the ``truck'' confidence by 7\% on average for models trained with
	a smaller learning rate, but increases the confidence by only 2\% for models
	trained with a larger learning rate.
\end{itemize}
Across all three case studies, our framework surfaces fine-grained differences
between models trained with different learning algorithms, enabling us to
better understand the role of the design choices that make up a learning
algorithm.

%% file: sections/problem_setup.tex
We begin by laying the groundwork for both our formalization of 
algorithm comparison as well as our proposed approach.
In Section~\ref{ssec:motivation}, we state the algorithm comparison problem,
and formalize it as the task of identifying
\emph{distinguishing transformations}.
These are functions that---when applied to test examples---significantly
and consistently change
the predictions of one model class but not the other.
In Section~\ref{ssec:datamodels}, we introduce {\em datamodels} 
\citep{ilyas2022datamodels}, the key primitive behind our approach to 
algorithm comparison (which we present in Section \ref{sec:approach}). 

\subsection{Formalizing algorithm comparisons via distinguishing transformations}
\label{ssec:motivation}
\input{sections/problem_defn}

\subsection{Datamodel representations for comparison}
\label{ssec:datamodels}
\input{sections/datamodels}

%% file: sections/problem_defn.tex
The goal of algorithm comparison problem is to understand the ways in which two learning
algorithms (trained on the same dataset) differ in the models they yield.
More specifically, we are interested in comparing the {\em model classes}
induced by the two learning algorithms:
\begin{definition}[Induced model class]
    Given an input space $\mathcal{X}$,
    a label space $\mathcal{Y}$,
    and a model space $\mathcal{M} \subset \mathcal{X} \to \mathcal{Y}$,
    a \underline{learning algorithm}
    $\mathcal{A}: (\mathcal{X} \times \mathcal{Y})^* \to \mathcal{M}$
    is a (potentially random) function mapping a set of input-label pairs
    to a model.
    Fixing a data distribution $\mathcal{D}$,
    the \underline{model class} induced by algorithm
    $\mathcal{A}$ is the distribution over $\mathcal{M}$ that results from
    applying $\mathcal{A}$ to randomly sampled datasets from $\mathcal{D}$.
\end{definition}

Focusing on model classes (rather than individual models) 
isolates differences induced by the learning algorithms themselves
rather than by any inherent non-determinism in training 
\citep{zhong2021larger,damour2020underspecification,jiang2021assessing}.
The perspective we adopt here is that model classes differ insofar as
they use different features to make predictions.
Thus, our goal when comparing two algorithms should be to pinpoint features
that one model class uses but the other does not.
Rather than try to precisely define ``feature,'' however, 
we make this notion precise by defining functions that we call
\emph{distinguishing (feature) transformations}:

\begin{definition}[Distinguishing transformation]
    \label{def:informal_comparison}
    Let $\mathcal{A}_1, \mathcal{A}_2$ denote learning algorithms,
    $S$ a dataset of input-label pairs,
    and $\mathcal{L}$ a loss function (e.g., correct-class margin).
    Suppose $M_1$ and $M_2$ are models trained on dataset $\mathcal{D}$
    using algorithms $\mathcal{A}_1$ and $\mathcal{A}_2$ respectively.
    Then, a $(\epsilon, \delta)$-distinguishing feature transformation of $M_1$ with respect to $M_2$ is a function
    $F: \mathcal{X} \to \mathcal{X}$ such that for some label $y_c \in \mathcal{Y}$, 
    \begin{align*}
        \overbrace{
        \mathbb{E}[L_1(F(x), y_c) - L_1(x, y_c)]
        }^{\text{Counterfactual effect of $F$ on $M_1$}}
        \geq
        \delta \qquad \text{ and } \qquad
        \overbrace{
        \mathbb{E}[L_2(F(x), y_c)
        -
        L_2(x, y_c)]
        }^{\text{Counterfactual effect of $F$ on $M_2$}}
        \leq \epsilon,
    \end{align*}
    where $L_i(x, y) = \mathcal{L}(M_i(x), y)$ is the loss of 
    a model trained with algorithm $i$ on the pair $(x, y)$,
    and the expectations above are taken over inputs $x$ and randomness in the learning algorithm.
\end{definition}

Intuitively, a distinguishing feature transformation is just a function $F$ that,
when applied to test examples,
significantly changes the predictions of one model class---but not the other---in a consistent way.
Definition \ref{def:informal_comparison} also immediately suggests a way to {\em evaluate}  the effectiveness of a distinguishing feature transformation.  %
That is, given a hypothesis about how two algorithms differ
(e.g., that models trained with $\mathcal{A}_1$ are more sensitive to texture than those trained with $\mathcal{A}_2$),
one can design a corresponding transformation $F$
(e.g., applying style transfer, as in \citep{geirhos2019imagenet}),
and directly measure its relative effect on the two model classes.

\paragraph{{\em Informative} distinguishing transformations (IDTs).}
Not every distinguishing transformation $F$ sheds the same amount of light on model behavior.
For example, given any two non-identical learning algorithms,
we could
craft a transformation $F$ that
imperceptibly modifies its input to satisfy Definition
\ref{def:informal_comparison}
(e.g., by using adversarial examples).
Alternatively, one could craft an $F$ that arbitrarily transforms its inputs into pathological out-of-distribution examples on which the two model
classes disagree.
While these transformations satisfy Definition \ref{def:informal_comparison},
they yield no benefit in terms of qualitatively understanding the differences in \emph{salient}
features used by the model classes.
More concretely, an {\em informative} distinguishing feature transformation must
(a) capture a feature that naturally arises in the data distribution and
(b) be semantically meaningful.

%% file: sections/datamodels.tex
Before presenting our approach to comparing learning algorithms 
(Section \ref{sec:approach}),
we provide an overview of the underlying datamodeling 
framework \citep{ilyas2022datamodels}, and its applicability to learning algorithm comparison.

\paragraph{A primer on datamodels.}
Let us fix a learning algorithm $\mathcal{A}$ 
(i.e., a map from training datasets to machine learning models),
a training set $S = \{x_1,\ldots,x_d\}$, and a {\em specific} test example $x \in T$.
For any subset $S' \subset S$ of the training set, we follow
\citep{ilyas2022datamodels} and define the {\em model output function} 
\begin{equation}
    \label{eq:model_output_function}
    f(x, S') = \text{ the loss after training a model on $S'$ (using $\mathcal{A}$) and evaluating on $x$.}
\end{equation}

\citet{ilyas2022datamodels} show that
one can often approximate the model output function above 
with a much simpler model, which they call a {\em datamodel}. 
For example, if the learning algorithm $\mathcal{A}$ is training a neural network on
a standard image classification task and the output function $f$ is the 
{\em correct-class margin} of the resulting classifier, a simple {\em linear} predictor suffices, i.e.,
\begin{equation}
    \label{eq:datamodels}
    \mathbb{E}[f(x, S')] \approx \theta_x \cdot \bm{1}_{S'},
\end{equation}
where $\theta_x$ is a (learned) parameter vector
and $\bm{1}_{S'} \in \{0, 1\}^{|S|}$ is a binary {\em indicator vector} of the
set $S'$, encoding whether each example of $S$ is included in $S'$, i.e.,
\[
    (\bm{1}_{S'})_i = \mathbb{1}\{x_i \in S'\}.
\]
The vector $\theta_x$ is called the {\em datamodel} for the example 
$x$, and one can compute it by sampling a collection of random subsets $S_k \subset S$,
training a model on each subset, and solving a (regularized) linear regression problem 
of predicting $f(x, S_k)$ from $\bm{1}_{S_k}$, i.e.,
\[
    \theta_x = \arg\min_\theta 
    \sum_{k=1}^n 
        \left(f(x, S_k) - \theta_x^\top \bm{1}_{S_k}\right)^2
    + \lambda_x \|\theta\|_1,
\]
where cross-validation determines the regularization strength 
$\lambda_x \in \mathbb{R}$.
Together, the resulting datamodel vectors for each test example form the {\em
datamodel matrix} $\Theta \in \mathbb{R}^{|T| \times |S|}$.

\paragraph{Datamodel representations.}
For our purposes, it will be convenient to view the datamodel $\theta_j$ for a test example $x_j$ 
as a $|S|$-dimensional {\em representation} of $x_j$, 
where the $i$-th coordinate $\theta_{ji}$ of this representation
is the extent to which models trained with $\mathcal{A}$
rely\footnote{For readers familiar with influence functions
\citep{koh2017understanding,hampel2011robust},
an intuitive (but not quite accurate)
way to interpret the datamodel weight $\theta_{ji}$ is as the influence of the
$i$-th training example on test example $x_j$.
}
on the $i$-th training example to classify example $x_j$.
These representations have two properties that make them a useful tool for model comparison:

\begin{itemize}
    \item[(a)] \textbf{Consistent basis:}
    The $i$-th coordinate of a datamodel representation always corresponds to 
    the importance of the $i$-th training example, 
    making datamodels inherently comparable across learning algorithms
    (as long as the training set is fixed).
    In other words, datamodel
    representations are {\em automatically aligned} across
    different algorithms, 
    and even across models that lack explicit representations 
    (e.g., decision trees).

    \item[(b)] \textbf{Predictiveness:}
    Datamodels have a natural {\em causal} interpretation.
    In particular, \citet{ilyas2022datamodels} show that datamodels can
    predict the counterfactual effect of removing or
    adding different training examples on model output for a given
    test example.
    As a result, any trends we find across the datamodel
    representations come with a precise quantitative interpretation in terms of
    model outputs (we return to this point in Section \ref{sec:approach}).

\end{itemize}

%% file: sections/approach.tex
Recall (from Section \ref{ssec:motivation}) 
that we want to identify informative distinguishing 
transformations (IDTs). These are functions in input space that
(a) induce different behavior from the two learning algorithms being studied;
while 
(b) capturing a feature that arises naturally in the data and 
(c) being semantically meaningful.
Now, both the size of the space of transformations
as well as the subjective nature of ``semantically meaningful''
make it futile to search for 
\idt{s}
directly.

Instead, our method (called \textsc{ModelDiff}) outputs a set of 
{\em distinguishing subpopulations},
groups of examples with a common feature
that is used by one learning algorithm and not the other.
One can use these subpopulations to design 
\idt{s}, whether by inspection or
domain-specific analysis.
\textsc{ModelDiff} finds these subpopulations by examining
{\em how each algorithm uses the training data},
and comprises the following steps:

\paragraph{Step 1: Compute normalized datamodels for each algorithm.}
    We start by computing a datamodel representation
    \citep{ilyas2022datamodels}
    for each example in the test set $T$.
Specifically, we compute two sets of
datamodels---$\smash{\cdma}, \smash{\cdmb} \in \mathbb{R}^{|T| \times |S|}$---corresponding to the model
classes induced by learning algorithms $\mathcal{A}_1$ and $\mathcal{A}_2$
respectively.
After computing the datamodels, we {\em normalize} each one---that is,  
we divide each datamodel (or equivalently, each row of $\smash{\cdmi}$) 
by its $\ell_2$ norm.

\paragraph{Step 2: Compute residual datamodels.}
    Next, we compute a {\em residual datamodel} for each test example $x_i$,
    which is
    the projection of the datamodel $\smash{\dma_i}$
    onto the null space of datamodel $\smash{\dmb_i}$:
		\begin{align}
            \label{eq:residual_dms}
    		\resdmab_i = \dma_i - \text{proj}_{\dmb_i}\left(\dma_i\right).
		\end{align}
        Intuitively, the residual datamodels of algorithm $\mathcal{A}_1$ with respect to $\mathcal{A}_2$ correspond to the training directions that influence $\mathcal{A}_1$ after ``projecting away'' the component that also influences $\mathcal{A}_2$.
\paragraph{Step 3(i): Run PCA on residual datamodels.}
Finally, we use principal component analysis to find the
highest-variance directions in the space of residual datamodels.
That is, we run
\begin{equation}
    \label{eq:pca}
    \textsc{$\ell$-PCA}(\{\resdmab_1, \ldots, \resdmab_{|T|}\})
\end{equation}
to find the top $\ell$ principal components $\{u_1,\ldots u_\ell\}$ 
of the residual datamodels. 
Intuitively, these principal components---which we call {\em distinguishing 
training directions}---correspond to (weighted) combinations of training examples
that together are used by one algorithm and not the other.

How do we verify that these principal components 
actually help us distinguish the two learning algorithms?
\Cref{thm:informal} below leverages the predictive properties
of datamodels to establish that connection. Specifically, it shows that for a training direction to truly distinguish 
$\mathcal{A}_1$ from $\mathcal{A}_2$ (in terms of model behavior),
it suffices to explain a high (resp., low) amount of the variance in datamodel
representations of algorithm $\mathcal{A}_1$ (resp., $\mathcal{A}_2$). 
Importantly, the latter conditions can be checked directly by simply plotting the explained variances of each component---see Figure \ref{fig:headline_2}(3i).

\begin{proposition}[Informal---see Appendix \ref{app:analysis} for a formal statement and proof]
    \label{thm:informal}
    Assume for illustrative purposes that Eq. \eqref{eq:datamodels}
    holds with equality 
    (i.e., that datamodels perfectly approximate model output).
    Let $\bm{u} \in \mathbb{R}^{|S|}$ be a principal component 
    identified in \eqref{eq:pca}, and define the \underline{explained variance gap} as
    \[
        \Delta(\bm{u}) = \|\Theta^{(1)} \bm{u}\|^2 - \|\Theta^{(2)} \bm{u}\|^2.
    \]
    Then, up/down-weighting the training set $S$ (in a specific way) according 
    to $u$ will change outputs of algorithm 
    $\mathcal{A}_1$ by 
    $\Delta(\bm{u})$ 
    more than the outputs of algorithm $\mathcal{A}_2$
    on average (i.e., across test examples).
\end{proposition}

\paragraph{Step 3(ii): From principal components to subpopulations.}
We now have a set of $\ell$ distinguishing training directions $u_i$
(i.e., principal components of the residual datamodel matrix)
that we have verified by measuring $\Delta(u)$ from \Cref{thm:informal}.
The final step in our method is to translate each distinguishing direction $u_i$ 
into a corresponding {\em distinguishing subpopulation}, i.e.,
the set of test examples $x_j$ whose residual 
datamodels $\smash{\theta_j^{(1\backslash 2)}}$ 
most closely align with $u_i$ (See Figure \ref{fig:headline_2}(3ii)). 

\paragraph{Connecting back to our algorithm comparison objective.} 
The output of \textsc{ModelDiff} is the set of distinguishing subpopulations
from Step 3(ii).
We connect these subpopulations back to our formal definition of the comparison
problem (\Cref{def:informal_comparison})
by inferring a {\em distinguishing feature} 
for each one\footnote{
    We intentionally leave this step open-ended to ensure our
    method is applicable across domains and modalities.
    In vision specifically,
    we can streamline this step using text-vision models like CLIP
    \citep{radford2021learning} (see Appendix \ref{app:clip}).
} (e.g., in Figure \ref{fig:headline_2}, all the images in the distinguishing 
subpopulation have an orange tint).
We then design a transformation that adds this feature to inputs
(e.g., adding an orange tint---\Cref{fig:headline_2}(4)),
and measure its average effect
on models trained with $\mathcal{A}_1$ and $\mathcal{A}_2$.

We illustrate our end-to-end process (including the verification step) visually in Figure \ref{fig:headline_2}.

\begin{figure}[h]
    \centering
    \includegraphics[width=\textwidth]{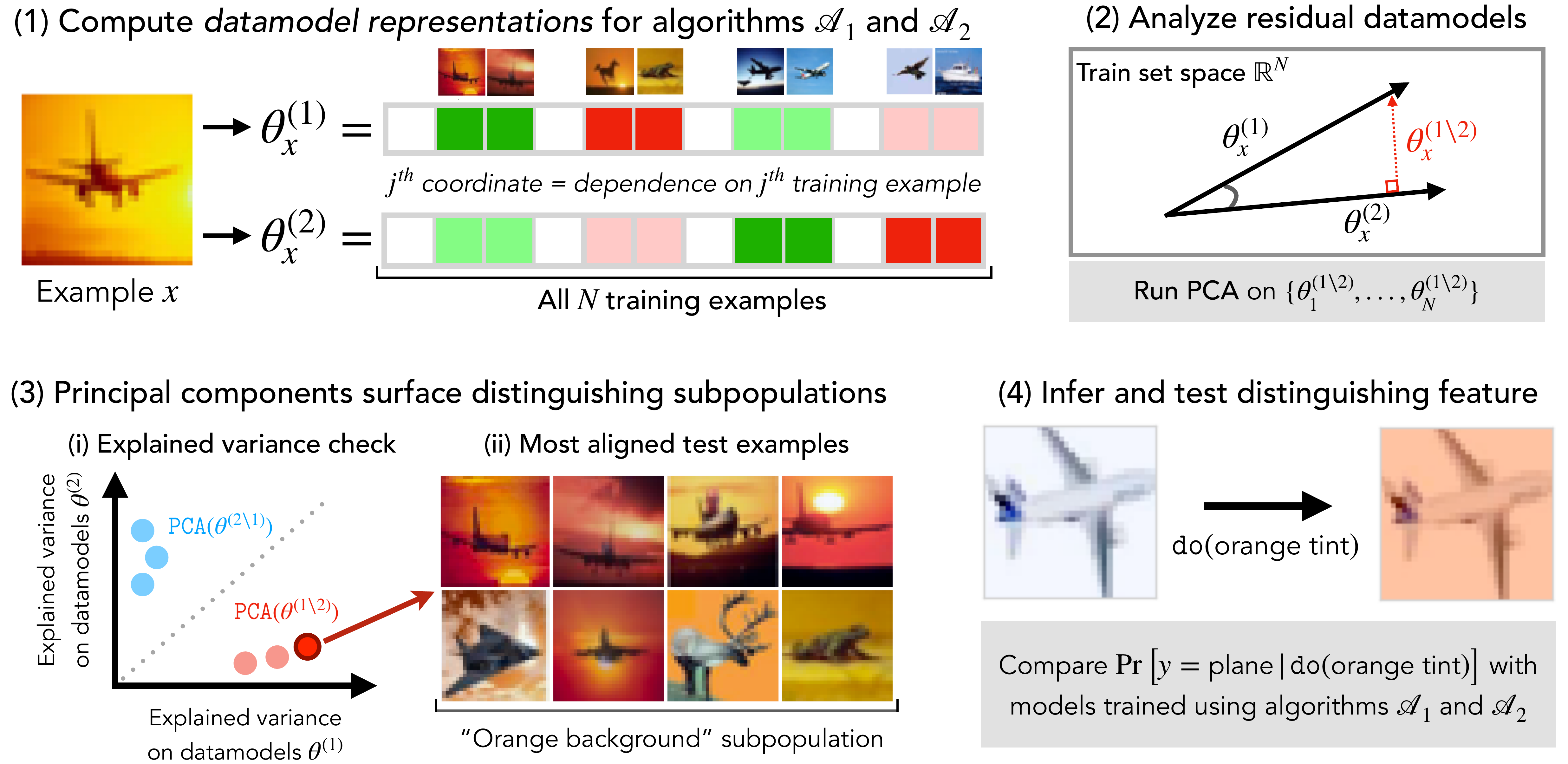}
    \caption{A visual summary of our approach.
    We (1) compute datamodel representations for each algorithm
    \citep{ilyas2022datamodels},
    then (2) compute residual datamodels \eqref{eq:residual_dms};
    next, (3) PCA on these residual datamodels yields so-called 
    ``distinguishing subpopulations.''
    Finally, in (4) simple inspection suffices to turn this subpopulation into 
    a feature transformation that we test with Definition \ref{def:informal_comparison}.
    }
    \label{fig:headline_2}
\end{figure}

%% file: sections/examples.tex
We demonstrate our comparison framework
using three case studies that each study
a key aspect of the standard training pipeline---namely, 
data augmentation, pre-training, and optimizer.
In each case study, we follow the same procedure:
(i) we use the \textsc{ModelDiff} algorithm from \Cref{sec:approach} above
to identify {\em distinguishing subpopulations};
(ii) we inspect these subpopulations to come up with a {\em candidate distinguishing transformation};
and (iii) we apply these candidate transformation to test examples and compare
its effect on models trained with the two learning algorithms.

\subsection{Case study: Data augmentation}
\label{ssec:data_aug}

\input{sections/comparisons/data_aug}

\begin{figure}[h]
	\centering
	\begin{subfigure}[t]{0.35\textwidth}
		\centering
		\includegraphics[width=\textwidth]{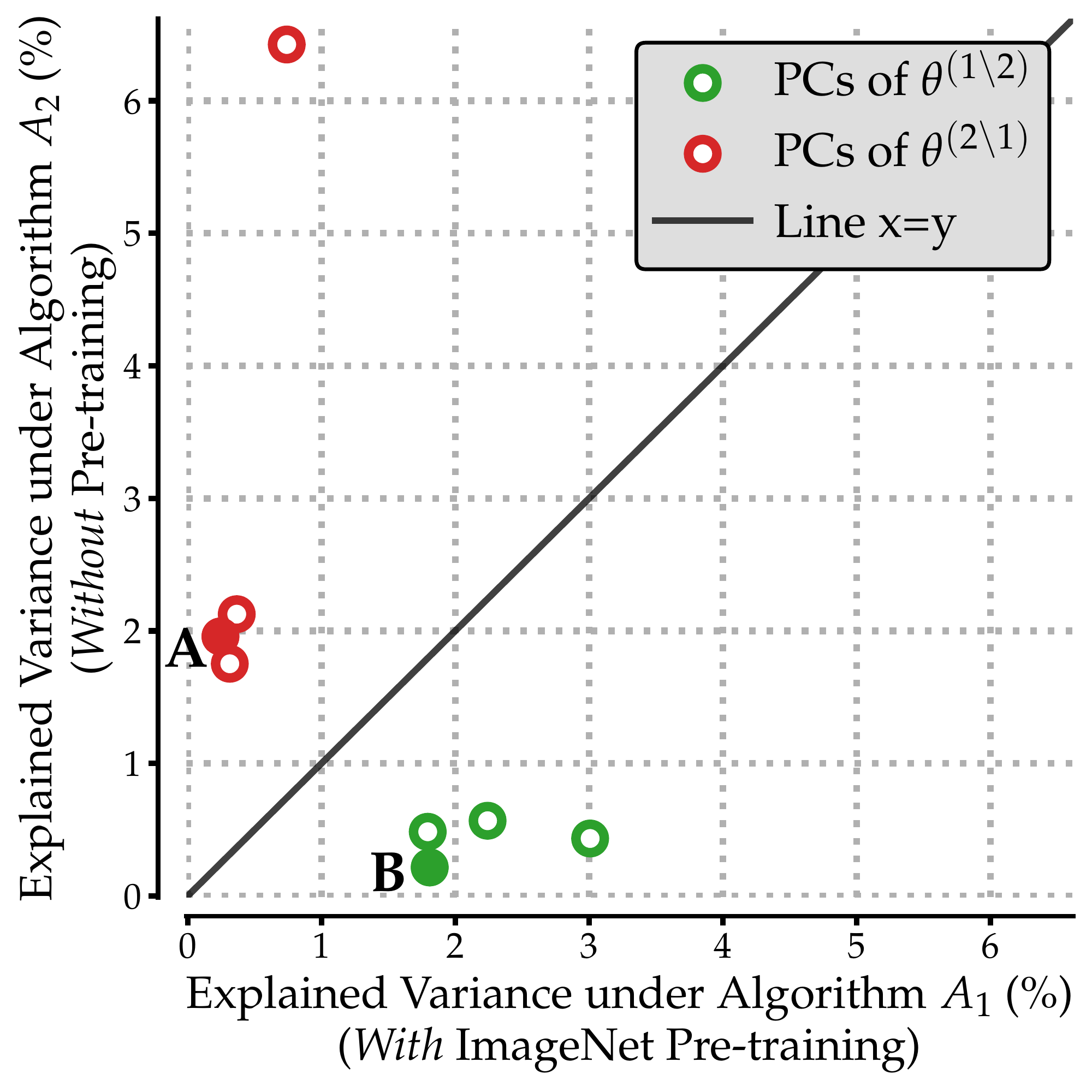}
	\end{subfigure}
	\hspace{1em}
	\begin{subfigure}[t]{0.6\textwidth}
		\includegraphics[width=\textwidth]{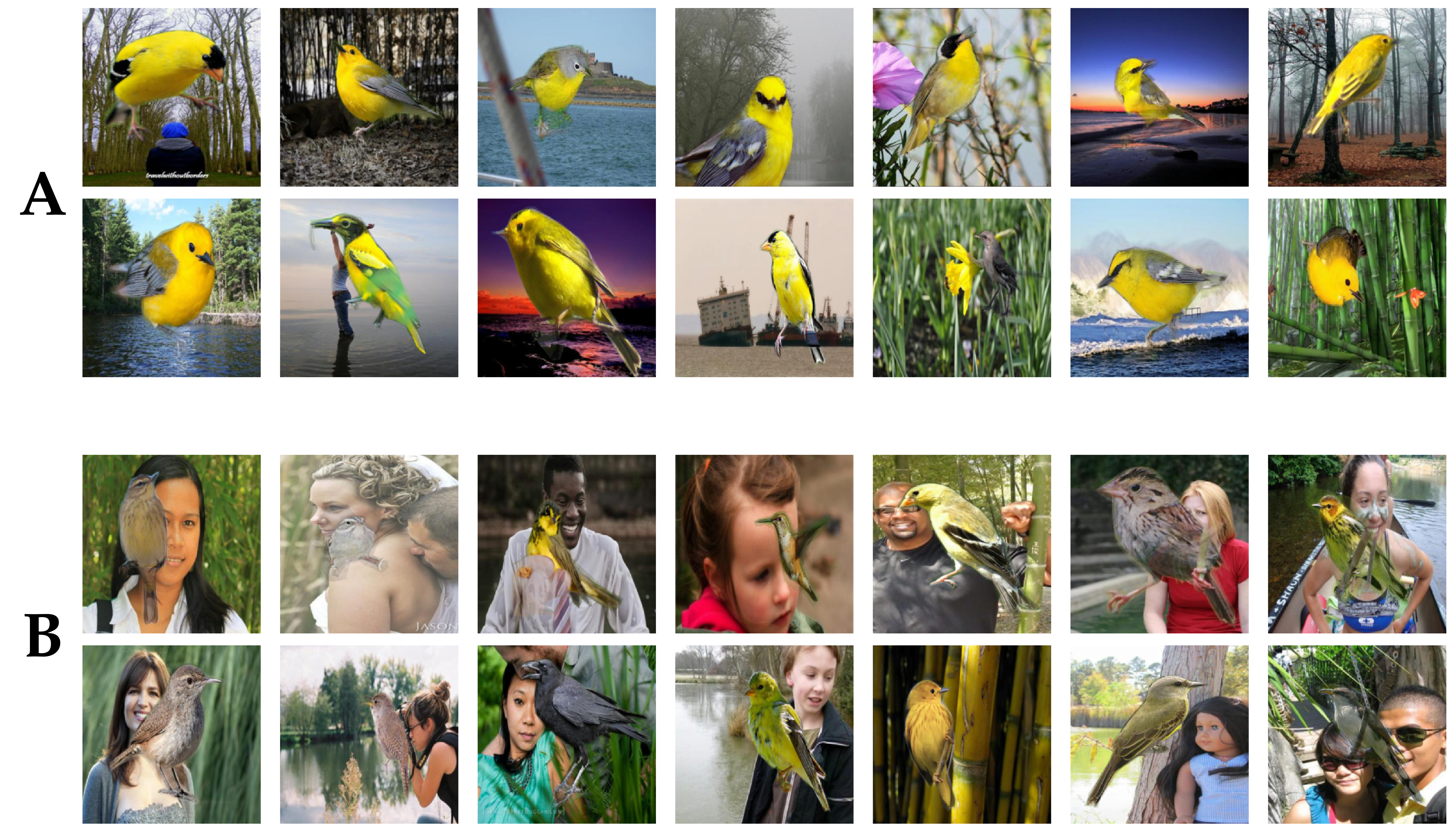}
	\end{subfigure}
	\caption{
	\textbf{Comparing \text{Waterbirds} models trained with and without ImageNet pre-training.}
	An analog to Figure \ref{fig:ev_pca-living17} for our second case study,
	see Figure \ref{fig:ev_pca-living17} for a description.
	In this case, subpopulation $\text{\bf{A}}$ seems to correspond to yellow birds and
	subpopulation $\text{\bf{B}}$ to faces in the background.
	}
	\label{fig:ev_pca-waterbirds}
\end{figure}
\subsection{Case study: Pre-training}
\label{ssec:pretrain}

\input{sections/comparisons/pretraining}

\subsection{Case study: Optimizer hyperparameters}
\label{ssec:lr}
\input{sections/comparisons/lr}

%% file: sections/comparisons/data_aug.tex
\begin{figure}
	\centering
	\begin{subfigure}[t]{0.35\textwidth}
		\centering
		\includegraphics[width=\textwidth]{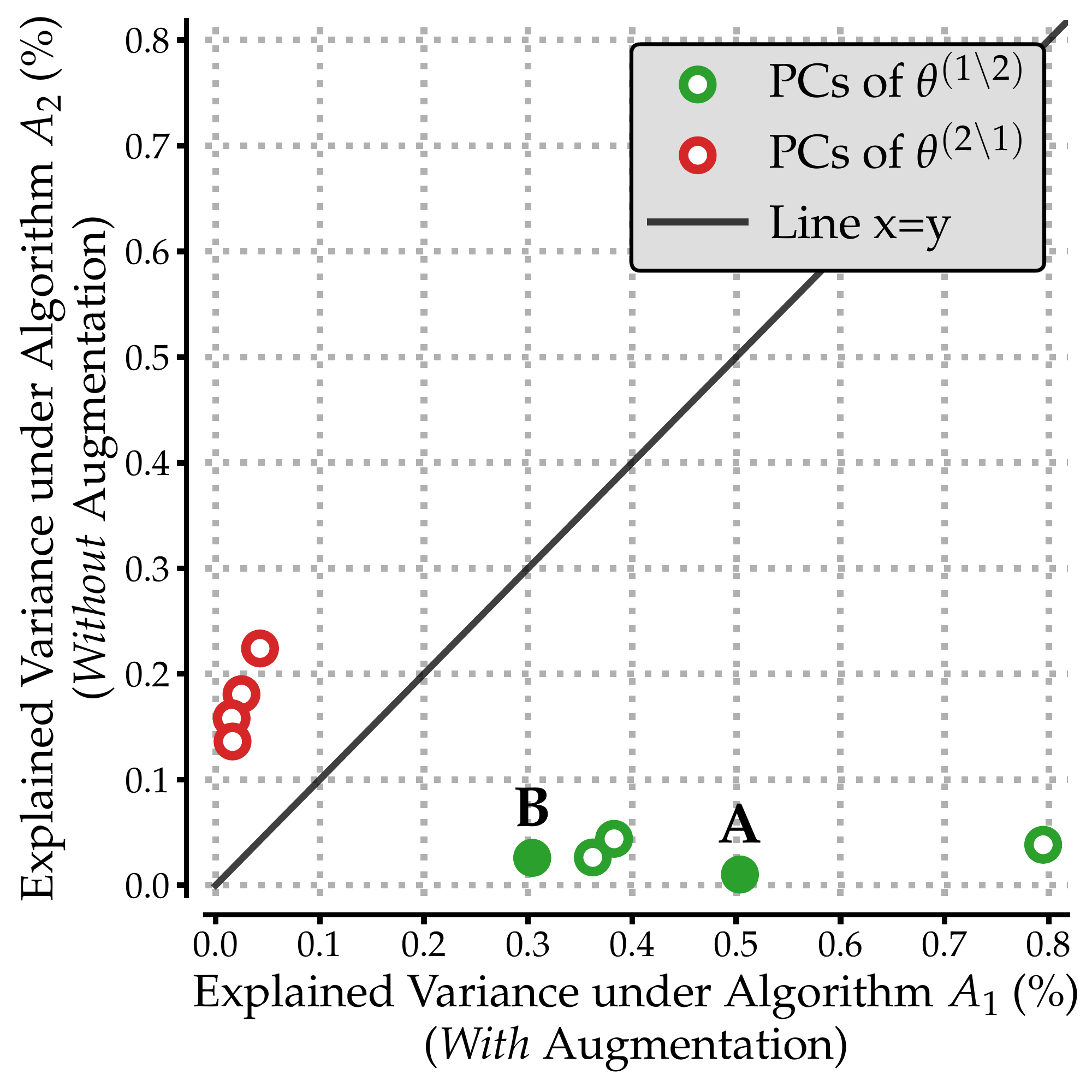}
		\label{fig:ev-living17}
	\end{subfigure}
	\hspace{1em}
	\begin{subfigure}[t]{0.6\textwidth}
		\includegraphics[width=\textwidth]{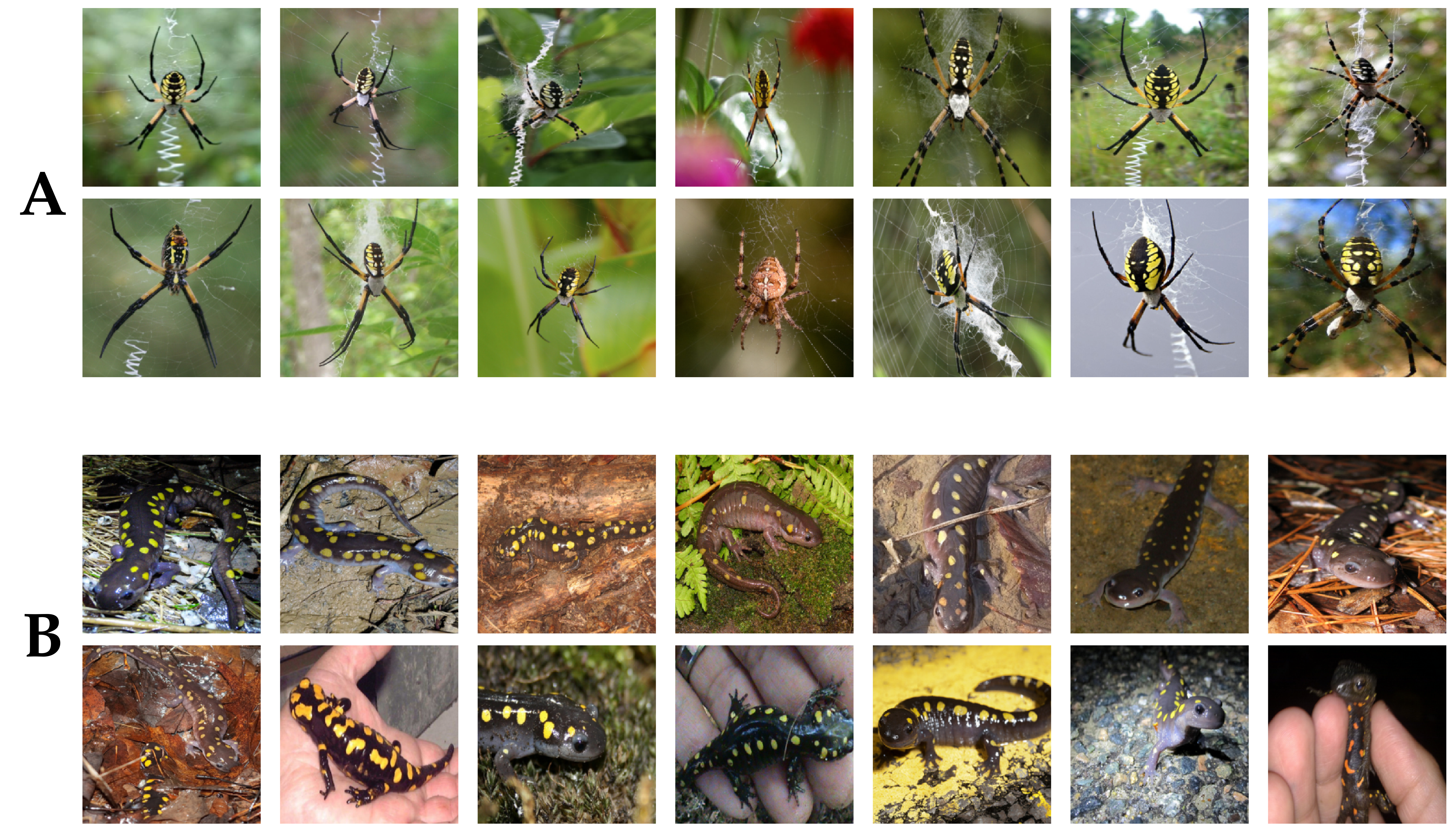}
	\end{subfigure}
	\caption{
	\textbf{Comparing \textsc{Living17} models trained with and without standard data augmentation.}
	\textbf{(Left)} Each green (resp., red) point is a {\em training direction}
	(i.e., a vector $v \in \mathbb{R}^{|S|}$
	representing a weighted combination of training examples)
	that distinguishes $\mathcal{A}_1$ from $\mathcal{A}_2$
	(resp., $\mathcal{A}_2$ from $\mathcal{A}_1$)
	as identified by \textsc{ModelDiff} (i.e., each point is a solution to \eqref{eq:pca}).
	The $x$ and $y$ coordinates of each point
	represent the ``importance'' (as given by \Cref{thm:informal}) of the training direction to models
	trained with $\mathcal{A}_1$ and $\mathcal{A}_2$ respectively.
	\textbf{(Right)}
	The surfaced distinguishing subpopulations, corresponding to the distinguishing training
	directions annotated $\text{\bf{A}}$ and $\text{\bf{B}}$.
	Direction $\text{\bf{A}}$ surfaces spiders on spider webs,
	and direction $\text{\bf{B}}$ surfaces salamanders with yellow polka dots.
	}
	\label{fig:ev_pca-living17}
\end{figure}

Data augmentation is a key component of the standard computer vision
training pipeline.
Still, while it often improves overall performance,
the effect of data augmentation
on models' learned {\em features} remains elusive.
In this case study,
we study the effect of data augmentation on classifiers trained 
for
the ImageNet-derived \textsc{Living17} task \citep{santurkar2021breeds}.
Specifically, we compare two classes of ResNet-18 models trained on this dataset
with the exact same settings 
modulo the use of data augmentation, i.e.,
\begin{itemize}
	\item {\bf Algorithm $\mathcal{A}_1$}: Training with standard augmentation,
	(namely, horizontal flip and resized random crop with \texttt{torchvision}
	default parameters).
	Resulting models attain 89\% average test accuracy
	(where the average is taken over randomness in training).
	\item {\bf Algorithm $\mathcal{A}_2$}: Training without data augmentation.
	Models attain 81\% average accuracy.
\end{itemize}
\Cref{app:datasets} provides further experimental details for this case studies and the others.

\begin{figure}[bp!]
	\centering
	\begin{subfigure}{0.45\textwidth}
		\includegraphics[width=\textwidth]{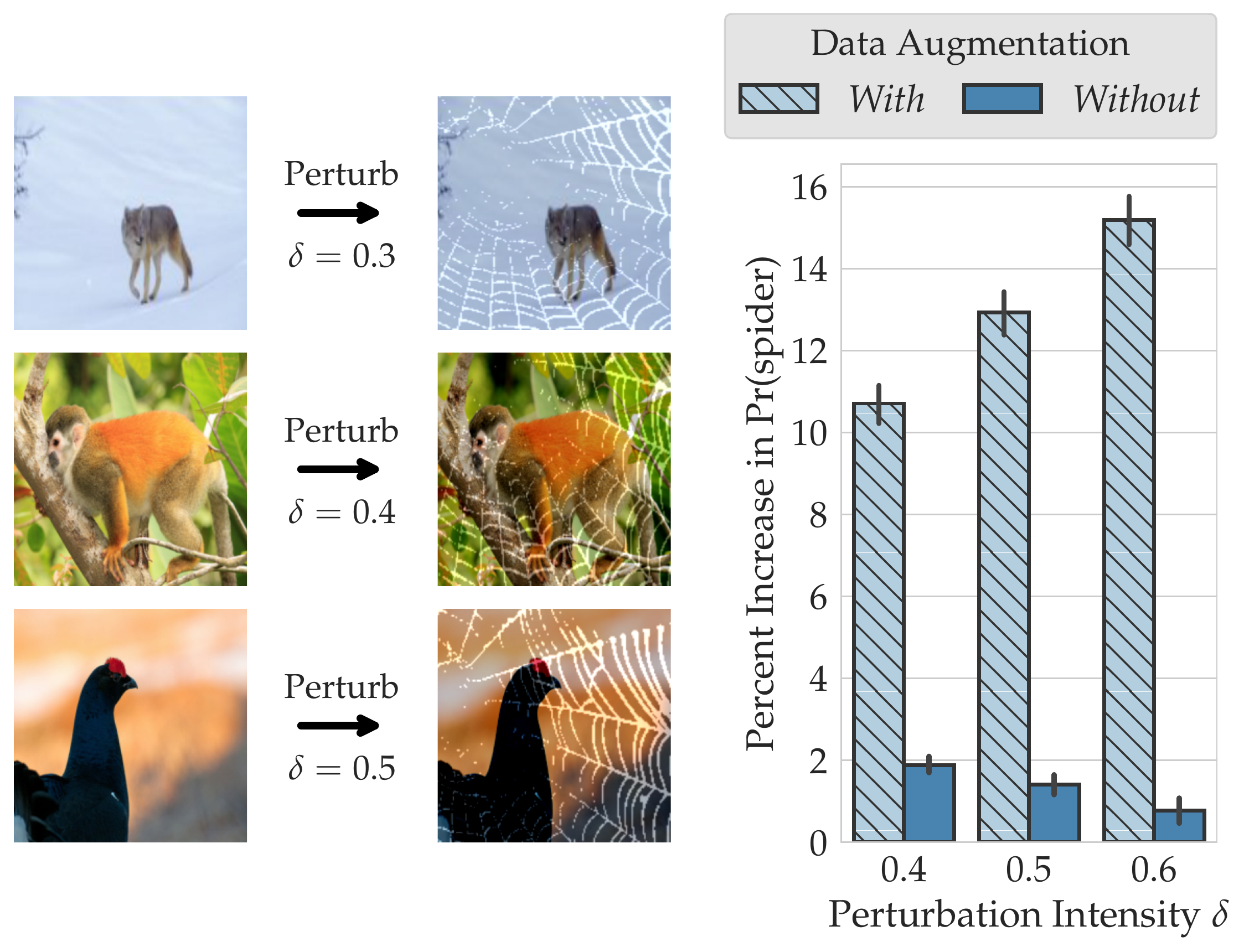}
		\caption{``Spider web'' feature}
		\label{fig:ft_l17_spider}
	\end{subfigure}
	\hspace{2em}
	\begin{subfigure}{0.45\textwidth}
		\includegraphics[width=\textwidth]{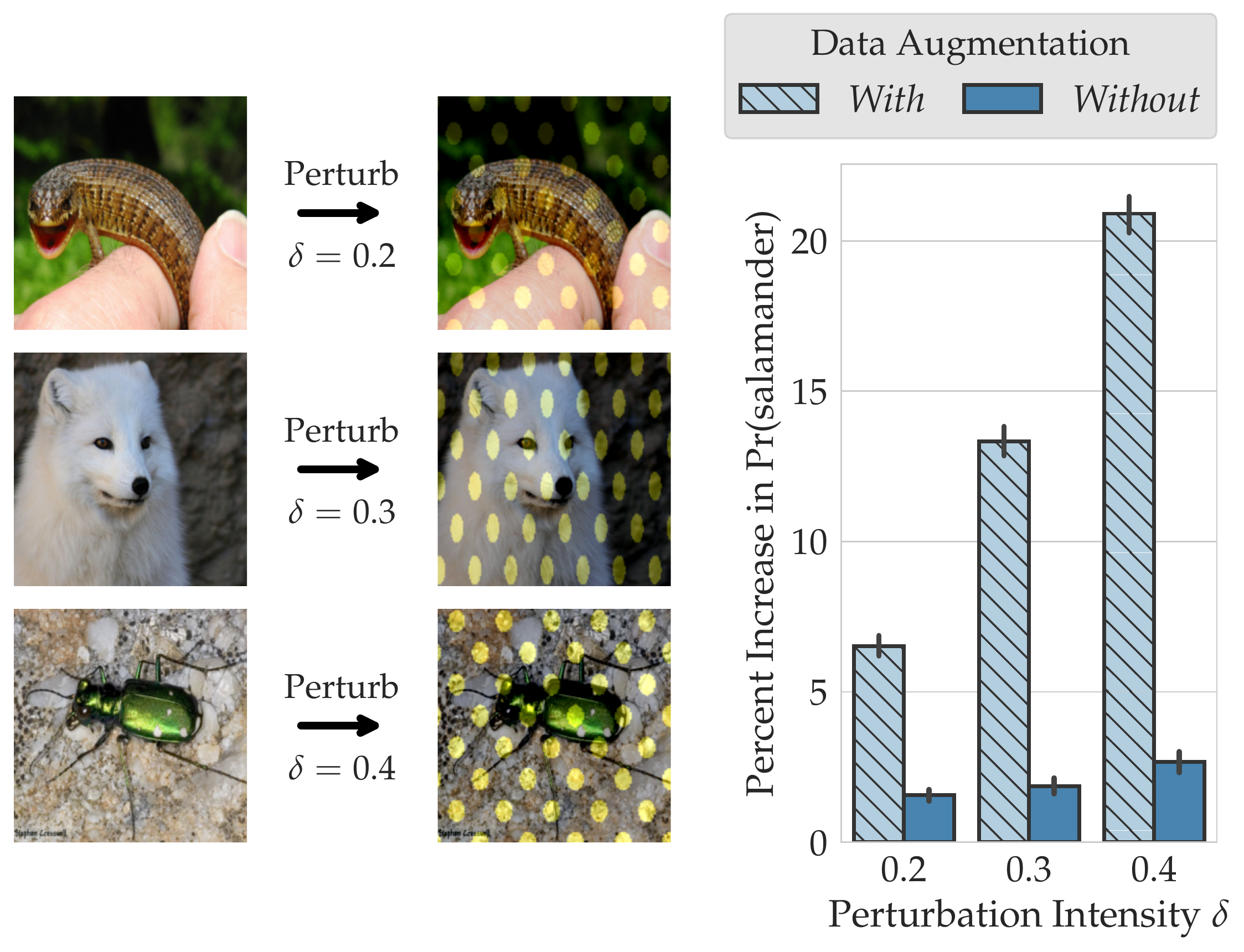}
		\caption{``Polka dots'' feature}
		\label{fig:ft_l17_salamander}
	\end{subfigure}
	\caption{
		\textbf{Effect of data augmentation on \textsc{Living17} models}.
		Standard data augmentation amplifies specific instances of co-occurrence
		(panel (a)) and texture (panel (b)) biases.
		The left side of each panel illustrates the distinguishing feature transformation
		at three different intensities $\delta$.
		On the right, we plot the average effect of the transformation on the confidence 
		of models trained with and without data augmentation, for varying $\delta$.
		At moderate intensity $\delta$, adding a spider web pattern to images makes models trained with
		(without) data augmentation, on average, 13\% (1\%) more confident in
		predicting the class ``spider,''
		while
		overlaying a polka dot pattern makes them
		14\% (2\%) more confident in the class ``salamander.''
		In both cases, increasing the intensity $\delta$ widens the gap between the two model classes. 
	}
	\label{fig:ft_living17}
\end{figure}

\paragraph{Applying \textsc{ModelDiff}.}
Applying our \textsc{ModelDiff} method to algorithms $\mathcal{A}_1$ and $\mathcal{A}_2$ 
gives us a set of distinguishing subpopulations, visualized in Figure \ref{fig:ev_pca-living17} (right).

On the left side of the same figure---inspired by \Cref{thm:informal}---we 
plot the variance explained by each distinguishing training direction
(i.e., the solutions to \eqref{eq:pca}) in the datamodels for both  
$\mathcal{A}_1$ ($x$ axis) and $\mathcal{A}_2$ ($y$ axis).
The training directions distinguishing $\mathcal{A}_1$ from $\mathcal{A}_2$
(in green)
indeed explain a significant amount of variance in the datamodels of $\mathcal{A}_1$ but not
in those of $\mathcal{A}_2$. 
For the directions distinguishing $\mathcal{A}_2$ from $\mathcal{A}_1$ (in red)
the situation is reversed, just as we would expect.

\paragraph{From subpopulations to IDTs.}
The subpopulations 
in Figure \ref{fig:ev_pca-living17} suggest the following features:
\begin{itemize}
	\item {\bf Spider web:} Direction $\textbf{A}$ surfaces a subpopulation of test images that contain spiders.
		In contrast to random spider images in $\textsc{Living17}$, all the
		images in the surfaced subpopulation contain a \emph{spider web} in the
		background (\Cref{app:hooman}).
		We thus hypothesize that
		models trained
		with standard data augmentation---moreso than those trained without
		it---use spider webs to predict the class ``spider''.
		We test this hypothesis using a feature transformation that overlays a
		spider web pattern onto an entire image
		(see Figure \ref{fig:ft_l17_spider}).
	\item {\bf Polka dots:} Direction $\textbf{B}$ surfaces a subpopulation of test images that contain salamanders.
		Again, comparing these images to random salamander images in
		$\textsc{Living17}$ test data reveals 
		yellow-black polka dots as a common feature (see also \Cref{app:hooman}).
		This suggests that 
		models trained with data augmentation might rely on the
		polka dot texture to predict the class ``salamander.''
		To test this hypothesis, we design a feature transformation that adds
		small yellow polka dots to the entire image (see Figure \ref{fig:ft_l17_salamander}).
\end{itemize}
We provide additional analysis in support of these hypotheses in Appendix \ref{app:hooman}.

Figure \ref{fig:ft_living17} shows the effect of the above
transformations on models trained with and without data augmentation.
The results confirm our hypotheses.
For example,
overlaying a spider web pattern with 
30\% opacity
increases $P(\text{``spider''})$
(i.e., models' average softmax confidence in the spider label)
for models trained with (without) data augmentation
by 11\% (1\%).
Similarly, overlaying the yellow polka dot texture with 
30\% opacity increased
$P(\text{``salamander''})$ 
by 14\% (2\%).
Furthermore, increasing the transformations' intensity consistently widens the
gap between the predictions of the two model classes.
Overall, these differences verify that the feature transformations we constructed can indeed distinguish the two learning algorithms as per Definition \ref{def:informal_comparison}.

\paragraph{Connection to previous work.}
Our case study demonstrates how training with standard data augmentation on
\textsc{Living17} data can amplify {specific instances} of co-occurrence
bias (spider webs) and texture bias (polka dots).
These findings are consistent with prior works that show that
data augmentation can introduce biases (e.g., \citep{hermann2020origins}).
Also, the distinguishing features that we find are specific to certain classes---this may explain why data augmentation can have disparate impacts on performance across different classes~\cite{balestriero2022effects}.
Finally, our findings, which demonstrate how data augmentation alters relative importance of specific features, corroborate the view of data augmentation as {\em feature manipulation}~\cite{shen2022data}.

%% file: sections/comparisons/pretraining.tex
Pre-training on large datasets is a standard approach to improving performance on downstream tasks where training data is scarce.
In this case study, we study the effect of ImageNet
pre-training~\cite{deng2009imagenet,kornblith2019better} on classifiers
trained for the \textsc{Waterbirds} task~\citep{sagawa2020distributionally}. 
To study this, we compare two classes of ResNet-50 models trained with the exact
same settings (see Appendix \ref{app:exp_setup}) modulo the use of ImageNet pre-training, i.e.,
\begin{itemize}
	\item {\bf Algorithm $\mathcal{A}_1$}: Pre-training on ImageNet, followed by
	full-network finetuning on \textsc{Waterbirds}. Models 
	attain $89.1\%$ average accuracy.
	\item {\bf Algorithm $\mathcal{A}_2$}: Training directly on \textsc{Waterbirds}. Models attain $63.9\%$ average accuracy.
\end{itemize}

\paragraph{Identifying IDTs with \textsc{ModelDiff}.} 
We apply \textsc{ModelDiff} just as in \Cref{ssec:data_aug},
with results shown in Figure \ref{fig:ev_pca-waterbirds}.
Verifying our results, we identify two candidate distinguishing features:
\begin{itemize}
	\item {\bf Yellow color:} Subpopulation $\textbf{A}$ comprises
	test images that contain yellow birds belonging to class ``landbird.''
	This leads us to hypothesize that models trained from scratch (i.e.,
	without ImageNet pre-training) spuriously rely on the color yellow to
	predict the class ``landbird,'' whereas ImageNet-pretrained models do not.
	To test this hypothesis, we design a feature transformation that adds a
	yellow square patch to images (see~\Cref{fig:ft_wb_yellow}).
	\item {\bf Human face:} Subpopulation $\textbf{B}$ comprises ``landbird'' images with \emph{faces} in the background.
	This suggests that ImageNet pre-training introduces a spurious dependence on human faces to predict the label ``landbird.''
	To test this hypothesis, we design a feature transformation that inserts patches of human faces to \textsc{Waterbirds} image backgrounds (see~\Cref{fig:ft_wb_face}).
\end{itemize}

\begin{figure}
	\centering
	\begin{subfigure}{0.47\textwidth}
		\includegraphics[width=\textwidth]{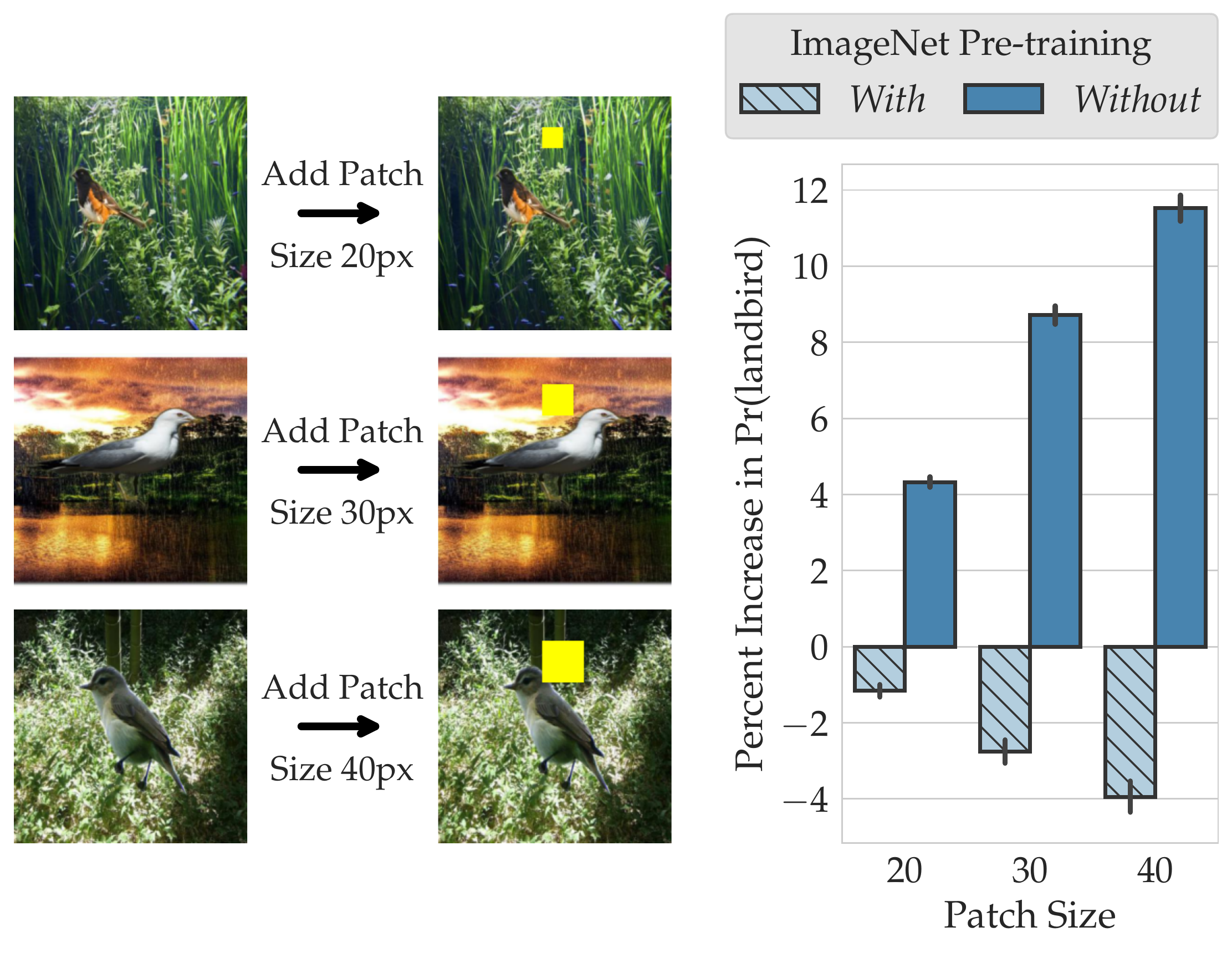}
		\caption{``Yellow'' feature}
		\label{fig:ft_wb_yellow}
	\end{subfigure}
	\hspace{2em}
	\begin{subfigure}{0.47\textwidth}
		\includegraphics[width=\textwidth]{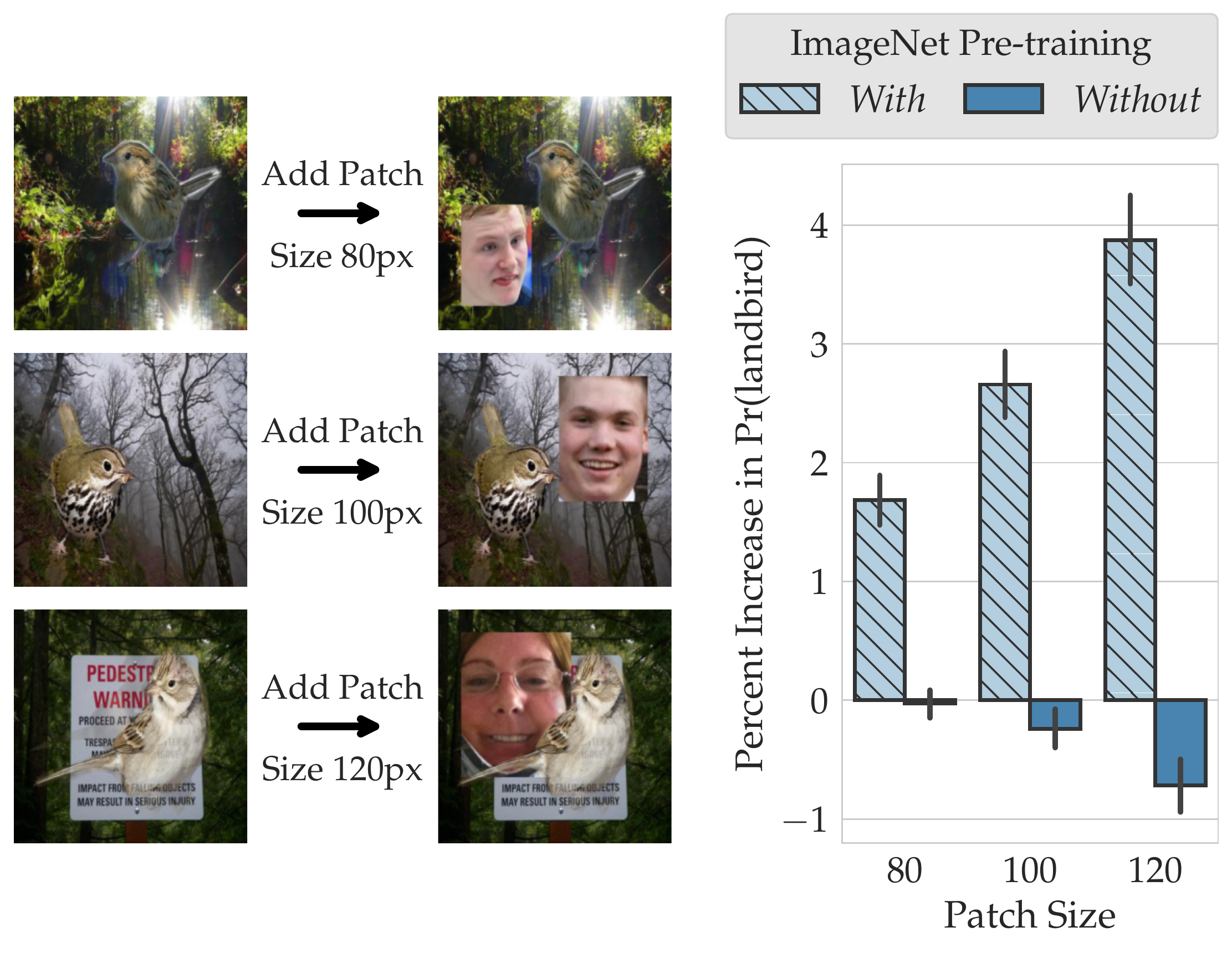}
		\caption{``Human face'' feature}
		\label{fig:ft_wb_face}
	\end{subfigure}
	\caption{
		\textbf{Effect of ImageNet pre-training on \textsc{Waterbirds} classification}.
		Analogously to Figure~\ref{fig:ft_living17}, we use our
		framework to identify
		spurious correlations that
		are either suppressed or amplified by ImageNet pre-training.
		\textbf{(Left)}
		Adding a yellow patch to images makes models
		trained without (with) pre-training, on average, 9\% more (2\% less)
		confident in predicting the label ``landbird.''
		\textbf{(Right)}
		Adding human faces to image
		backgrounds makes models trained with (without) pre-training, on average,
		3\% (0\%) more confident in predicting the label ``landbird.''
		Again, in both cases, increasing the intensity of feature transformations
		widens the gap in treatment effect between the two model classes. 
	}
	\label{fig:ft_waterbirds}
\end{figure}

\noindent Once again, additional analysis (Appendix \ref{app:hooman}) supports these hypotheses.

In Figure \ref{fig:ft_waterbirds}, we compare the effect of the above
feature transformations on models trained with and without ImageNet pre-training.
The results confirm both of our hypotheses.
Adding a small (40px) yellow square patch 
to test images increased
$P(\text{``landbird''})$ by 12\% for models trained from scratch
but {\em decreased} $P(\text{``landbird''})$ for models pre-trained on ImageNet.
Similarly, adding a human face patch\footnote{We collect a bank of human
face patches using ImageNet validation examples and their corresponding face
annotation~\cite{yang2022study}; see~\Cref{app:feature_trans} for more information.} to
image backgrounds increased
$P(\text{``landbird''})$
by up to 4\% for pre-trained models, but did not significantly affect models trained from scratch.
Once again, increasing the intensity (i.e., patch size) of these feature transformations
further widens the gap in sensitivity between the two model classes.

\paragraph{Streamlining the verification step with CLIP.} 
Since one can apply \textsc{ModelDiff} 
across a variety of domains
(i.e., anywhere where one can compute datamodels),
we have left the verification step intentionally vague so far.
However, often the structure of the data being studied allows 
us to find IDTs in a more streamlined manner.
For example, in Appendix \ref{app:clip} we apply 
\textsc{Clip} \citep{radford2021learning}
to generate automatically hypotheses for this case study.
Given the images in subpopulations $\textbf{A}$ and $\textbf{B}$,
\textsc{Clip} generates candidate captions 
\texttt{\{``yellow'',``canary'',``lemon''\}} 
and
\texttt{\{``florists'', ``florist'', ``faces''\}} respectively
(see \Cref{app:clip} for more details).

\paragraph{Connections to prior work.}
Our results show that
ImageNet pre-training reduces dependence on some spurious correlations (e.g.,
yellow color $\rightarrow$ landbird) but also \emph{introduces} new ones (e.g.,
human face $\rightarrow$ landbird).
These findings connect to the (seemingly contradictory) phenomenon
where pre-training improves robustness to spurious
features~\cite{ghosal2022vision,tu2020empirical} 
while also transferring over spurious correlations~\cite{salman2022when,neyshabur2020being} from
the pre-training dataset.

%% file: sections/comparisons/lr.tex
\begin{figure}[t]
	\centering
	\begin{subfigure}[t]{0.35\textwidth}
		\centering
		\includegraphics[width=\textwidth]{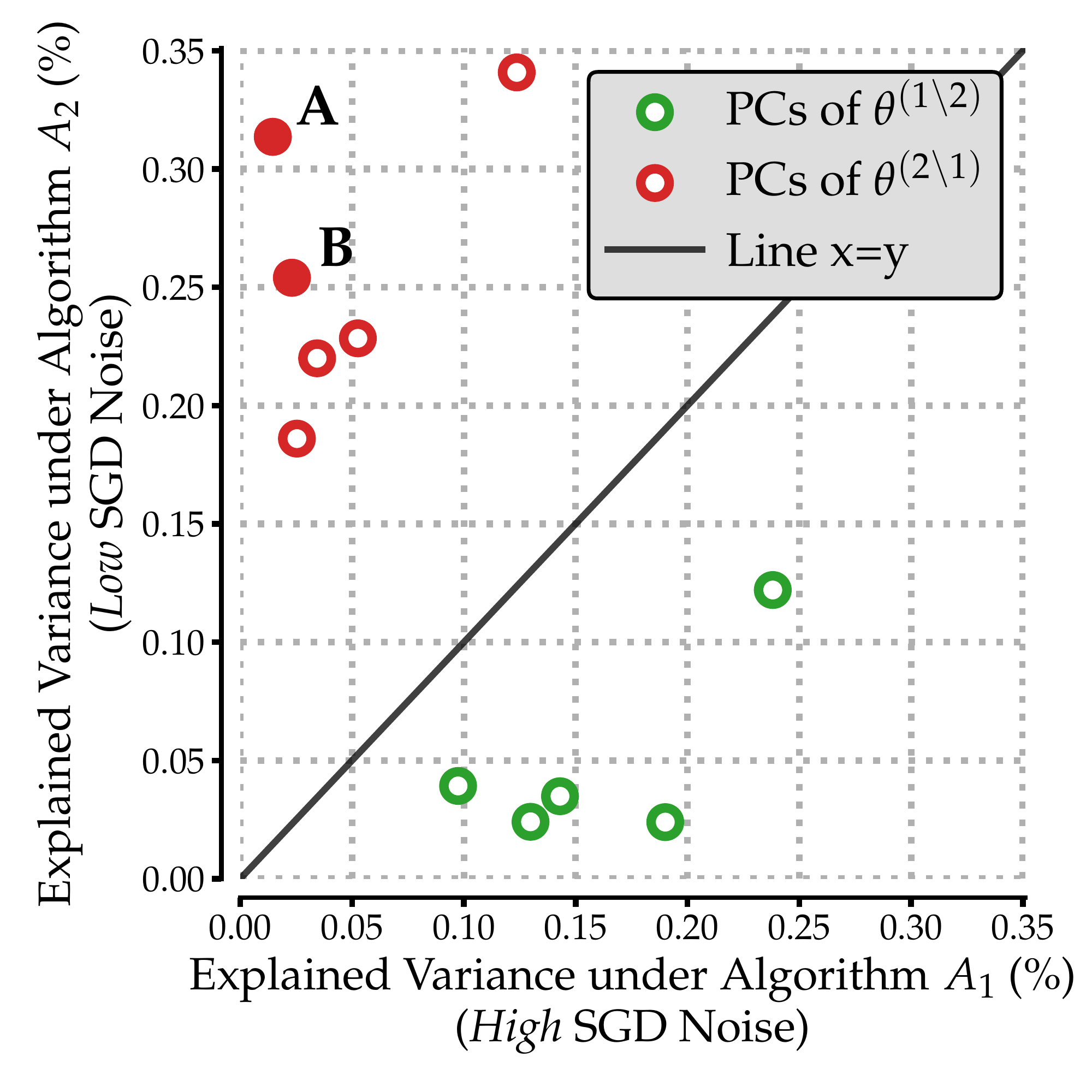}
	\end{subfigure}
	\hspace{1em}
	\begin{subfigure}[t]{0.6\textwidth}
		\includegraphics[width=\textwidth]{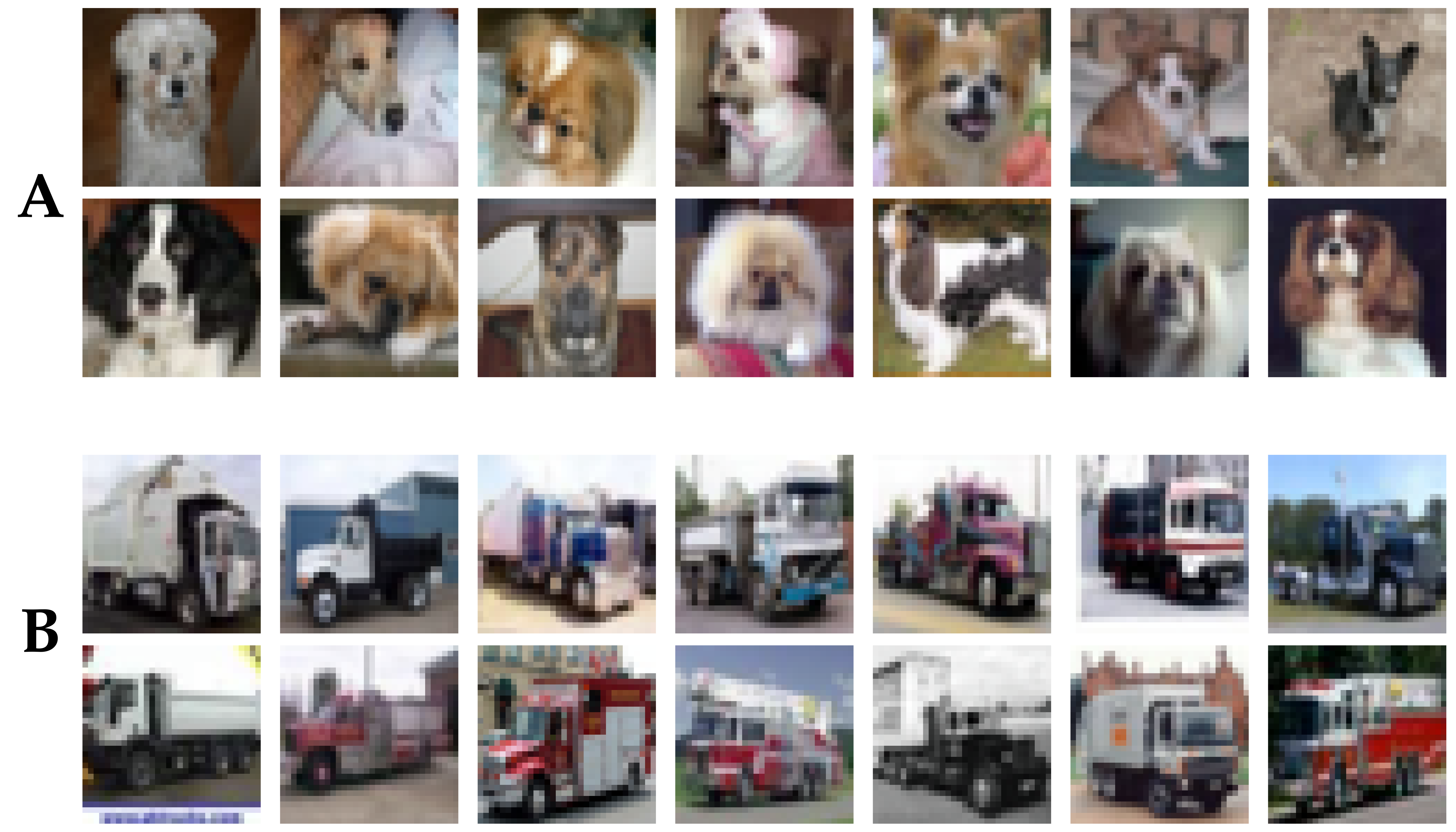}
	\end{subfigure}
	\caption{
	\textbf{Comparing \text{CIFAR-10} models trained with high and low SGD noise.}
	An analog to Figure \ref{fig:ev_pca-living17} for our third case study. Here, subpopulation
	{\bf A} seems to correspond to dogs with a particular texture, and subpopulation {\bf B} to front-facing trucks with a prominent rectangular shape.
	}
	\label{fig:ev_pca-cifar}
\end{figure}

The choices of optimizer and corresponding hyperparameters can affect both the trainability and the generalization of resulting models~\cite{hoffer2017train,keskar2017large}.
In this case study, we study the effect of two 
hyperparameters---learning rate and batch size---that control the effective
scale of the noise in stochastic gradient descent (SGD).
We study the effect of these hyperparameters in the context of \textsc{Cifar-10}~\cite{krizhevsky2009learning} classifiers by comparing the following two learning algorithms:

\begin{itemize}
	\item {\bf Algorithm $\mathcal{A}_1$}: Training with high SGD noise, i.e., large learning rate (0.1) and small batch size (256). Models attain 93\% average test accuracy.
	\item {\bf Algorithm $\mathcal{A}_2$}: Training with low SGD noise, i.e., small large rate (0.02) and large batch size (1024). Models attain 89\% average test accuracy.
\end{itemize}

\paragraph{Identifying IDTs with \textsc{ModelDiff}.}
We apply \textsc{ModelDiff} just as in \Cref{ssec:data_aug},
with results shown in Figure \ref{fig:ev_pca-cifar}.
\begin{itemize}
	\item {\bf Black-and-white texture:} Subpopulation $\textbf{A}$ contains two-colored dogs (Figure~\ref{fig:ev_pca-cifar}, top right).
		The low-resolution nature of \textsc{CIFAR-10} makes it difficult to identify a 
		single distinguishing feature from this subpopulation.
		Additional analysis using datamodels (\Cref{app:hooman}), however, reveals a set of
		black-and-white training images from other classes that influence
		predictions on subpopulation $\textbf{A}$ only when models are trained with low
		SGD noise (algorithm $\mathcal{A}_2$). 
		We thus hypothesize that models trained with low SGD noise rely more on black-and-white textural features to predict the class ``dog.''
		To test this hypothesis, we design a transformation that
		adds a small black-and-white patch to a given image
		(see~\Cref{fig:ft_c10_dog}).
	\item {\bf Rectangular shape:} Direction $\textbf{B}$ surfaces a
	subpopulation of front-facing trucks (Figure~\ref{fig:ev_pca-cifar},
	bottom right) that all have a similar rectangularly-shaped cabin and
	cargo area.
	The same style of additional analysis (see \Cref{app:hooman}) supports this observation, and
	suggests that models trained with low SGD noise (i.e., with $\mathcal{A}_2$) rely on rectangular-shaped patterns to predict the class ``truck''.
	To test this hypothesis, we design a feature transformation that modifies a given image
	with a patch of two high-contrast rectangles, loosely resembling the
	cabin/cargo shape of trucks in the subpopulation $\textbf{B}$
	(see~\Cref{fig:ft_c10_truck}).
\end{itemize}

\begin{figure}
	\centering
	\begin{subfigure}{0.47\textwidth}
		\includegraphics[width=\textwidth]{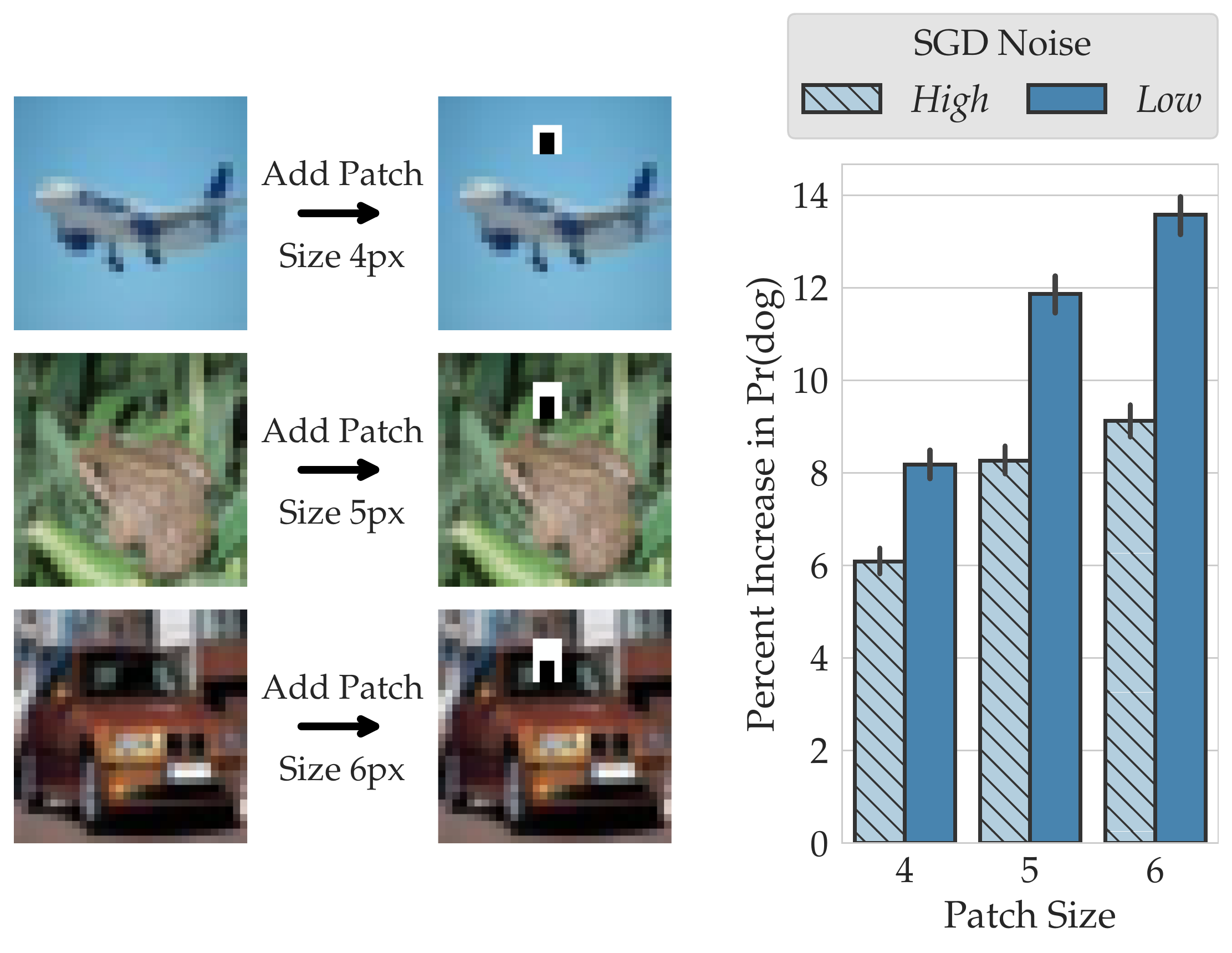}
		\caption{``Black-and-white texture'' feature}
		\label{fig:ft_c10_dog}
	\end{subfigure}
	\hspace{2em}
	\begin{subfigure}{0.47\textwidth}
		\includegraphics[width=\textwidth]{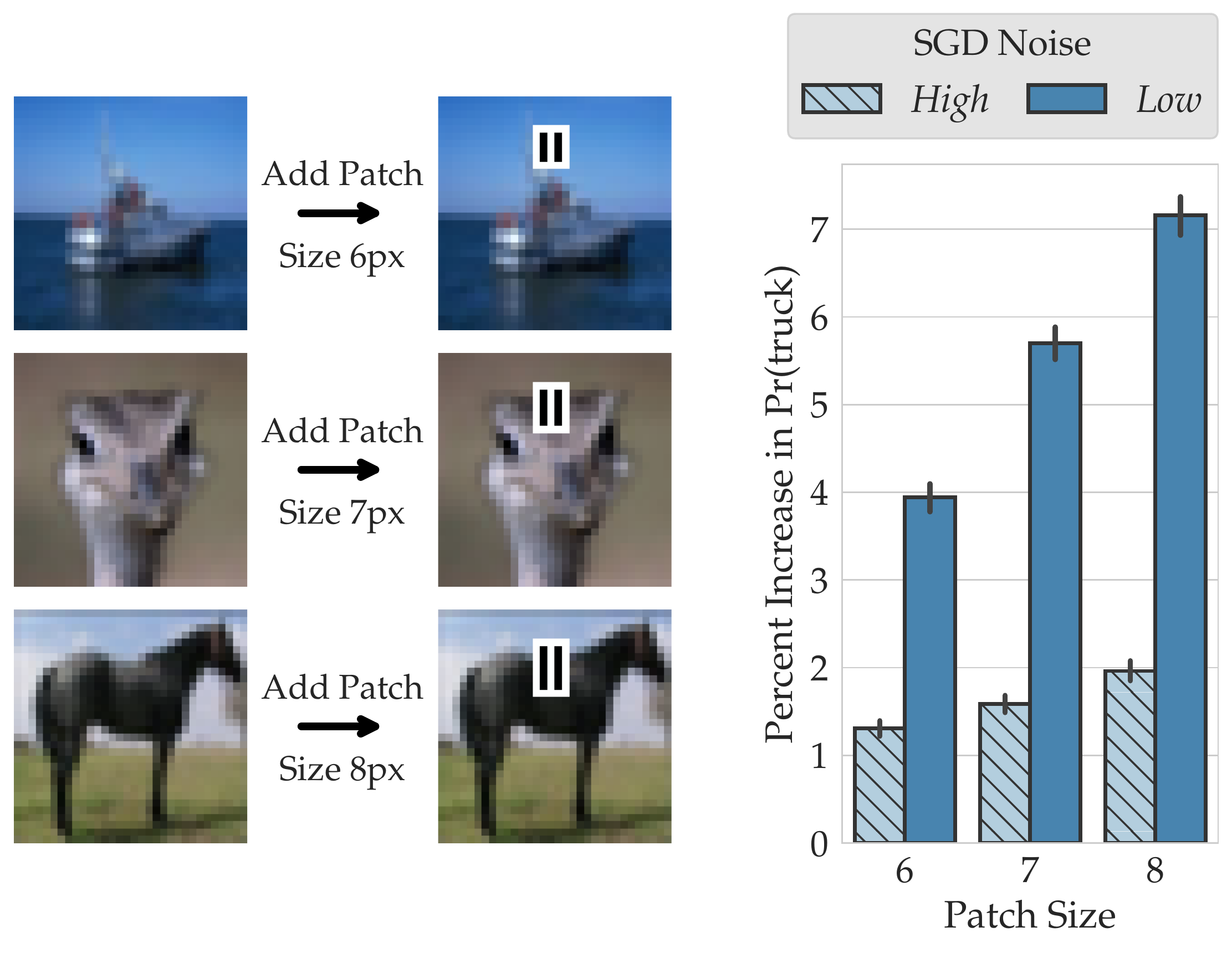}
		\caption{``Rectangular shape'' feature}
		\label{fig:ft_c10_truck}
	\end{subfigure}
	\caption{
		\textbf{Effect of SGD hyperparameters on CIFAR-10 models}.
		Analogously to Figures \ref{fig:ft_living17} and \ref{fig:ft_waterbirds}, we use our framework to identify features that distinguish models trained with lower SGD noise from models trained with higher SGD noise.
		\textbf{(Left)} %
		Adding a black-and-white patch to images makes models trained with low (high) SGD noise, on average, 11\% (8\%) more confident in predicting the label ``dog.''
		\textbf{(Right)} %
		Adding high-contrast rectangles to images makes models trained with low (high) SGD noise, on average, 5.5\% (1.5\%) more confident in predicting the label ``truck.''
		In both cases, increasing the intensity of feature transformations widens the gap in treatment effect between the two model classes.
	}
	\label{fig:ft_cifar}
\end{figure}

In Figure \ref{fig:ft_cifar}, we compare the effect of the above feature
transformations on models trained with high and low SGD noise. The results again
confirm both hypotheses.
Adding a small (6px) black-and-white patch 
to test images increased $P(\text{``dog''})$ by 14\% (9\%) for models trained with low (high) SGD noise.
Similarly, applying a small (8px) rectangular-shape patch 
increased $P(\text{``truck''})$ by 7\% (2\%) for models trained with low (high) SGD noise.
Increasing the intensity (i.e., patch size) of the transformations again widens the gap in sensitivity between the two model classes.

\paragraph{Connections to prior work.} 
This case study shows how reducing the scale of SGD noise can increase reliance on certain low-level features (e.g., rectangular shape $\rightarrow$ trucks). %
While prior works show that lower SGD noise worsens aggregate model performance~\cite{keskar2017large,wen2019empirical}, our methodology identifies \emph{specific} features that are amplified due to low SGD noise. 
Furthermore, the simplistic nature of the identified features corroborate the theoretical explanation put forth in~\citet{li2019towards}: learning rate scale determines the extent to which models memorize patterns that are easy-to-fit but hard-to-generalize. %
More broadly, our framework motivates a closer look at how features amplified via low SGD noise alter aggregate model performance.

%% file: sections/more_discussion.tex
In this section, we discuss some ways of extending and evaluating \textsc{ModelDiff}.

\paragraph{Adapting our framework for aggregate comparison.}
Our goal in this paper %
has been fine-grained feature-based comparisons between learning algorithms. %
This goal is different from the typical one studied in the
{\em model comparison} literature~\cite{morcos2018importance,raghu2017svcca,bansal2021revisiting,ding2021grounding,morcos2018insights},
where the objective is to quantify the similarity between two (fixed) models in aggregate.
In Appendix \ref{app:cosinesim}, we show how to use datamodels for this alternate objective.\footnote{Note, however, that by using datamodels we are comparing the two model classes rather than two fixed models.}
In short, we can use the cosine similarity between the two datamodels of an example $x$ to measure how similarly training data impacts the two model classes' predictions at $x$.
Aggregating over all $x$ then gives a distribution (which we can also summarize with a score) that quantifies the overall similarity of the two learning algorithms. We find that such a metric yields qualitatively accurate comparisons.

\paragraph{Alternative methods for inferring features from distinguishing subpopulations.}
Our case studies in~\Cref{sec:examples} show that manual inspection of distinguishing subpopulations suffices to infer informative distinguishing transformations.
More generally, we can supplement manual inspection with dataset-specific annotations~\cite{singla2021salient} or domain-specific tools such as vision-language models.
In~\Cref{app:clip}, we illustrate how using \textsc{Clip}~\cite{radford2021learning} can streamline feature extraction. %
Specifically, we leverage \textsc{Clip}'s shared vision-language embedding space to identify candidate distinguishing features by generating captions (or, descriptions) that best describe distinguishing subpopulations.
For problems in other domains, one could employ other domain-specific tools to assist with feature inference.

\paragraph{On prediction-level comparisons between two learning algorithms.}
\textsc{ModelDiff} extracts distinguishing subpopulations by analyzing how training data influences test predictions.
One might ask whether simply analyzing differences in predictions suffice to extract similar subpopulations.
To investigate this, in~\Cref{app:disagree} we look at whether restricting our analysis to examples on which the predictions of the two learning algorithms agree affects the distinguishing directions that we identified in~\Cref{sec:examples}. Note that an analysis only based on disagreement in predictions trivially cannot identify any differences in this case.
Our results indicate that \textsc{ModelDiff} outputs similar distinguishing directions even in this case, suggesting that analyzing the role of training data allows \textsc{ModelDiff} to extract finer-grained differences beyond just predictions.

\paragraph{Computational cost of datamodel estimation.}
The computational cost of \textsc{ModelDiff} scales with the number of models trained to compute datamodels in the first step of our method.
In~\Cref{app:samp_size}, we find that %
we can recover similar datamodels (and as a result, similar distinguishing subpopulations) even with $10\times$ fewer samples (i.e., number of models trained to estimate datamodels).
More generally, the ``black-box'' usage of datamodels in our framework implies that any improvements to the sample efficiency of datamodel estimation would directly make \textsc{ModelDiff} faster as well.

%% file: sections/discussion.tex
Our case studies in~\Cref{sec:examples} shed light on the implicit task-specific biases of various learning algorithms.  
We refer the reader to the ``Connections to prior work'' headings within each
respective case study for a discussion of these biases. 
Here, we instead focus on placing our work into context with existing approaches
to model and algorithm comparison in machine learning.

In particular, we compare and contrast our approach to algorithm comparison 
with approaches to the related problem of {\em model comparison}, 
where one tries to characterize the difference between two (usually fixed)
machine learning models. 
For simple models (e.g., sparse linear),
we can directly compare models in parameter space.
However, distances between more complex models such as deep neural networks are
much less meaningful.
To this end, a long line of work has sought to design methods for characterizing 
the differences between models:

\medskip
\noindent\textbf{Representation-based comparison.\ \ }
A popular approach (particularly in the context of deep learning) is to compare
models using their internal {\em representations}.
Since the coordinates of these representations do not have a consistent interpretation,
representation-based model comparison typically studies the degree to which 
different models' representations can be aligned.
Methods based this approach include canonical correlation analysis (CCA) and
variants \cite{raghu2017svcca,morcos2018insights,cui2022deconfounded}, centered
kernel alignment (CKA) \cite{kornblith2019similarity}, graph-based
methods \cite{li2015convergent,chen2021revisit}, and model stitching
\cite{csiszarik2021similarity, bansal2021revisiting}.
Prior works have used these methods to compare
wide and deep neural networks \cite{nguyen2021wide};
vision transformers and convolutional networks \cite{raghu2021vision}; 
pre-trained and trained-from-scratch models \cite{neyshabur2020being}; 
and different language models \cite{wu2020similarity}. 
Though they are often useful, prior work shows that representation-based
similarity measures are not always reliable for testing functional
differences in models~\cite{ding2021grounding}.
Our approach to algorithm comparison differs from these methods in both
objective and implementation:
\begin{itemize}
\item {\em Learning algorithms rather than fixed models}: Rather than focusing on a single fixed model, 
our goal in this paper is to compare the classes of models that result from
a given learning algorithm. 
In particular, we aim to find only differences that arise from algorithmic
design choices, and not those that arise from the (sometimes significant)
variability in training across random seeds \cite{zhong2021larger}.  
\item {\em Feature-based rather than similarity-based}: 
Methods such as CCA and CKA focus on outputting a single score
that reflects the overall similarity between two models. 
On the other hand, the goal of our framework is to find
fine-grained differences in model behavior. 
Still, in \Cref{app:cosinesim} we show that we can also use our
method for more global comparisons, for instance by computing the average cosine
similarity of the datamodel vectors.
\item {\em Model-agnostic}: Our framework is agnostic to type of model used
and thus allows one to easily compare models across learning algorithms---our method 
extends even to learning algorithms that do not have explicit representations 
(e.g., decision trees and kernel methods).
\end{itemize}

\paragraph{Example-level comparisons.}
An alternative method for comparing models is to compare their predictions
directly.
For example, \citet{zhong2021larger} compare predictions of small and large
language models (on a per-example level) to find that larger models are not
uniformly better across example.
Similarly, \citet{mania2019model} study the {\em agreement} between models,
i.e., how often they output the same prediction on a per-example level.
In another vein, 
\citet{meding2022trivial} show that after removing impossible or trivial
examples from test sets, different models exhibit more variations in their
predictions.
Our framework also studies instance-level predictions, but ultimately connects 
the results back to human-interpretable distinguishing features.

\paragraph{Comparing feature attributions.}
Finally, another line of work compares models in terms of how they use 
features at test time.
In the presence of a known set of features one can compute feature importances
(e.g., using SHAP \citep{lundberg2017unified}) and compare them across models
\cite{wang2022learning}. 
In cases where we do not have access to high-level
features, we can use instance-level explanation methods such as saliency maps to
highlight different parts of the input, but these methods generally do not help
at distinguishing models \cite{denain2022auditing}.
Furthermore, multiple evaluation metrics~\cite{adebayo2018sanity,hooker2018benchmark,shah2021input} indicate that common instance-specific feature attribution methods can fail at accurately highlighting features learned by the model.

%% file: sections/conclusion.tex
We introduce a general framework for fine-grained comparison of any two learning algorithms.
Specifically, our framework compares models trained using two different algorithms in terms of how the models \emph{rely on training data} to make predictions.
Through three case studies, we showcase the utility of our framework in pinpointing how three aspects of the standard training pipeline---data augmentation, pre-training, optimization hyperparameters---can shape model behavior. 

%% file: sections/ack.tex
Work supported in part by the NSF grants CCF-1553428 and CNS-1815221, and Open Philanthropy. This material is based upon work supported by the Defense Advanced Research Projects Agency (DARPA) under Contract No. HR001120C0015.

Research was sponsored by the United States Air Force Research Laboratory and the United States Air Force Artificial Intelligence Accelerator and was accomplished under Cooperative Agreement Number FA8750-19-2-1000. The views and conclusions contained in this document are those of the authors and should not be interpreted as representing the official policies, either expressed or implied, of the United States Air Force or the U.S. Government. The U.S. Government is authorized to reproduce and distribute reprints for Government purposes notwithstanding any copyright notation herein.

%% file: appendices/app_explanation.tex
\section{Algorithm analysis}
\label{app:analysis}

Here, we give a more formal version of \Cref{thm:informal}.
First, we need to define the sense in which a given learning algorithm is sensitive to up/down-weighting the training set according to a training direction $u \in \mathbb{R}^{|S|}$.

\paragraph{Sensitivity along a training direction.}
We consider the following probabilistic notion of sensitivity: we measure how $f(x,S')$ varies in expectation, if relative to some base distribution $\mathcal{D}_0$ over $S'$, we up/down weight the probability of each training example by its weight $u_i$.

More specifically, for $\mathcal{D}_0$, we consider sampling $S'$ uniformly, i.e., by choosing each element of the full training set $S$ with probability 1/2. Then, we consider the perturbed distribution $\mathcal{D}_u$ where $x_i$ is sampled with probability $(1+u_i)/2$.
Finally, we define the {\em sensitivity} of $f(x,\cdot)$ along $u$ (or $u$-{\em sensitivity}) as:
\[ f(x,\cdot) |_u = (\mathbb{E}_{S' \sim \mathcal{D}_u} f(x,S') - \mathbb{E}_{S' \sim \mathcal{D}_0} f(x,S'))^2 \]
where we take the square as we are interested in the magnitude.

We now give the formal result connecting the explained variance gap in datamodels (which we empirically observed in our case studies in Section~\ref{sec:examples}) to the above definition of senstivity.

\begin{theorem}[Formal version of \Cref{thm:informal}] Consider two learning algorithms $\mathcal{A}_1$ and $\mathcal{A}_2$ applied to a training set $S$ and evaluated on a test set $T$. Assume that their datamodels $\theta^{(1)}_x$ and $\theta^{(2)}_x$ perfectly approximate the respective model outputs $f^{(1)}(x,\cdot)$ and $f^{(2)}(x,\cdot)$.
Then, for a given training direction $u \in \mathbb{R}^{|S|}$, the explained variance gap
    \[
        \Delta(\bm{u}) = \|\Theta^{(1)} \bm{u}\|^2 - \|\Theta^{(2)} \bm{u}\|^2
    \]
is equal to the following quantity:
\[
    \sum\limits_{x \in T}  f^{(1)}(x,\cdot) |_u - \sum\limits_{x \in T} f^{(2)}(x,\cdot) |_u
\]
which is the difference in $u$-sensitivity of model outputs of algorithms $\mathcal{A}_1$ and $\mathcal{A}_2$.
\end{theorem}
\begin{proof}
    With our linear assumption and definition of sensitivity, the proof is almost immediate.
    First, under the linearity assumption, the sensitivity $f(x_i,\cdot)|_u$ just reduces
    to $(\theta_i \cdot u)^2$.
    It follows that the total sensitivity, $\sum_i f(x_i,\cdot)|_u$ is just $\sum_i (\theta_i \cdot u)^2$, from which the claim follows.

\end{proof}

%% file: appendices/setup.tex
\section{Experimental Setup}
\label{app:exp_setup}
In this section, we outline the experimental setup---datasets, models, training algorithms, hyperparameters, and datmodels---used for our case studies in~\Cref{sec:examples}.
 
\subsection{Datasets}
\label{app:datasets}

\paragraph{Living17.} 
The Living17 dataset~\cite{santurkar2021breeds} is an ImageNet-derived dataset, where the task is to classify images belonging to 17 types of living organisms (e.g., salamander, bear, fox). 
Each Living17 class corresponds to an ImageNet superclass (i.e., a set of ImageNet classes aggregated using WordNet~\cite{miller1995wordnet}). 
\citet{santurkar2021breeds} introduce Living17 as one of four benchmark to evaluate model robustness to realistic subpopulation shifts.  
In our case study, we study the effect of data augmentation using  a variant of this dataset, wherein the training and test images belong to the same set of subpopulations (i.e., no subpopulation shift).

\paragraph{Waterbirds.}
The Waterbirds dataset~\cite{sagawa2020distributionally} consists of bird images taken from the CUB dataset \cite{wah2011caltech} and pasted on backgrounds from the Places dataset~\cite{zhou2017places}. 
The task here is to classify ``waterbirds'' and ``landbirds'' in the presence of spurious correlated ``land'' and ``water'' backgrounds in the training data.
\citet{sagawa2020distributionally} introduce Waterbirds as a benchmark to evaluate models under subpopulation shifts induced by spurious correlations. 
In our case study, we compare how models trained from scratch on Waterbirds data differ from ImageNet-pretrained models that are fine-tuned on Waterbirds data.

\paragraph{CIFAR-10.} We consider the standard CIFAR-10~\citep{krizhevsky2009learning} image classification dataset in order to study the effect of two SGD hyperparameters: learning rate and batch size.

Summary statistics of the datasets described above are outlined in~\Cref{table:dataset}.

\begin{table}[h]
\centering
\caption{Summary statistics of datasets}
\label{table:dataset}
\begin{tabular}{lrrrrrr}
\toprule
\textbf{Dataset} & Classes & Size (Train/Test) & Input Dimensions \\
\midrule
Living17 & 17 & 88,400/3,400 & $3 \times 224 \times 224$ \\
Waterbirds & 2 & 4,795/5,794 & $3 \times 224 \times 224$ \\
CIFAR-10 &  10 & 50,000/10,000 & $3 \times 32\times 32$ \\
\bottomrule
\end{tabular}
\end{table}

\subsection{Models, learning algorithms, and hyperparameters}
\label{app:hyperparams}

\paragraph{Living17.} 
We use the standard ResNet18 architecture~\cite{he2015deep} from the \texttt{torchvision} library. 
We train models for $25$ epochs using SGD with the following configuration: initial learning rate $0.6$, batch size $1024$, cyclic learning rate schedule (with peak at epoch $12$), momentum $0.9$, weight decay $0.0005$, and label smoothing (with smoothing hyperparameter $0.1$). 
To study the effect of data augmentation, we train models with the following algorithms:
\begin{itemize}
	\item \textbf{Algorithm $\mathcal{A}_1$ (with data augmentation)}: Models are trained with standard data augmentation: random resized cropping (with default \texttt{torchvision} hyperparamters) and random horizontal flips. On average, models attain $89.2\%$ average test accuracy. 
	\item \textbf{Algorithm $\mathcal{A}_2$ (without data augmentation)}: Models are trained without data augmentation.  On average, models attain $81.9\%$ average test accuracy. 
\end{itemize}

\paragraph{Waterbirds.}
We use the standard ResNet50 architecture from the \texttt{torchvision} library. 
We train models using SGD with momentum $0.9$ and weight decay $0.0001$ for a maximum of $50$ epochs (and stop early if the training loss drops below $0.01$).
For model selection, we choose the model checkpoint that has the maximum average accuracy on the validation dataset.
As in~\citet{sagawa2020distributionally}, we do not use data augmentation. 
In our case study on pre-training, we consider ImageNet pre-trained models from \texttt{torchvision}. 
We consider models trained using the following algorithms:
\begin{itemize}
	\item \textbf{Algorithm $\mathcal{A}_1$ (ImageNet pre-training)}: Models pre-trained on ImageNet are fully fine-tuned on Waterbirds data with a fixed SGD learning rate $0.005$ and batch size $64$. On average, models attain $89.1\%$ (non-adjusted) average test accuracy and $63.9\%$ worst-group test accuracy.
	\item \textbf{Algorithm $\mathcal{A}_2$ (Training from scratch)}: Models are trained from scratch (i.e., random initialization) on Waterbirds data with SGD: initial learning rate $0.01$, batch size $64$, and a linear learning rate schedule ($0.2\times$ every $15$ epochs). On average, models attain $63.6\%$ average test accuracy and $5.7\%$ worst-group test accuracy. 
\end{itemize}

\paragraph{CIFAR-10.} 
We use the ResNet9 architecture from Kakao Brain\footnote{\url{https://github.com/wbaek/torchskeleton/blob/master/bin/dawnbench/cifar10.py}}, which is optimized for fast training.
We train models using SGD with momentum $0.9$ and weight decay $0.0005$ for a maximum of $100$ epochs (and stop early if the training loss drops below $0.01$).
We augment training data with a standard data augmentation scheme: random resized cropping with 4px padding and random horizontal flips.
To study the effect of SGD noise in our case study, we vary learning rate and batch size. 
Specifically, we compare models trained with the  following algorithms: 
\begin{itemize}
	\item \textbf{Algorithm $\mathcal{A}_1$ (high SGD noise)}: Models are trained with SGD using a large initial learning rate ($0.1$), small batch size ($256$), and a linear learning rate schedule ($0.5\times$ every $20$ epochs). On average, models attain $93.3\%$ test accuracy.
	\item \textbf{Algorithm $\mathcal{A}_2$ (low SGD noise)}: Models are trained with SGD using a small fixed learning rate ($0.02$) and large batch size ($1024$). On average, models attain $89.5\%$ test accuracy.
\end{itemize}

\subsection{Datamodels}
\label{app:datamodels}

Now, we provide additional details on datamodels which, we recall, are computed in the first step of our algorithm comparison framework.

\paragraph{Estimating linear datamodels.} 
Recall that the datamodel vector for example $x_j$, $\theta_j^{(i)} \in \mathbb{R}^{|S|}$, encodes the importance of individual training examples $S$ to model's loss at example $x_j$ when trained with algorithm $\mathcal{A}_i$.
Concretely, given test example $x_j$ and training set $S = \{x_1,\ldots,x_d\}$, the datamodel $\theta_j$ is a sparse linear model (or surrogate function) trained on the following regression task: For a training subset $S' \subset S$, can we predict the correct-class margin $f_{\mathcal{A}}(x_j; S')$ of a model trained on $S'$ with algorithm $\mathcal{A}$?  
This task can be naturally formulated as the following supervised learning problem: Given a  training set $\{(S_i, f_{\mathcal{A}}(x; S_i))\}^m_{i=1}$ of size $m$, the datamodel $\theta_j$ (for example $x_j$) is the solution to the following problem:
\begin{align}
    \label{eq:testset_datamodels}
    \theta_j = \min_{w \in \mathbb{R}^{|S|}}
    \frac{1}{m} \sum_{i=1}^m \lr{w^\top \mask{S_i} - \modeleval{x_j}{S_i}}^2 + \lambda \|w\|_1,
\end{align}
where $\mask{S_i}$ is a boolean vector that indicates whether examples in the training dataset $x \in S$ belong to the training subset $S_i$. 
Note that each datamodel training point $(S_i, f_{\mathcal{A}}(x_j, S_i))$ is obtained by (a) training a model $f$ (e.g., ResNet9) on a subset of data $S_i$ (e.g., randomly subsampled CIFAR data) and (b) evaluating the trained model's output on example $x_j$. 
\Cref{alg:datamodel_pseudo} provides pseudocode for estimating datamodels:

\begin{algorithm}
    \caption{An outline of the datamodeling framework introduced in~\cite{ilyas2022datamodels}.}
    \label{alg:datamodel_pseudo}
    \begin{algorithmic}[1]
        \Procedure{EstimateDatamodel}{target example $x$, train set $S$ of size $d$, subsampling frac. $\alpha \in (0,1)$}
        \State $T \gets []$ \Comment{Initialize {\em datamodel training set}}
        \For{$i \in \{1,\ldots,m\}$}
        \State Sample a subset $S_i \subset S$ from $\mathcal{D}_S$ where $|S_i| = \alpha \cdot d$
        \State $y_i \gets \modeleval{x}{S_i}$
        \Comment{Train a model on $S_i$ using $\mathcal{A}$, evaluate on $x$}
        \State Define $\mask{S_i} \in \{0, 1\}^d$ as
        $(\mask{S_i})_j = 1$ if $x_j \in S_i$ else 0
        \State $T \gets T + [(\mask{S_i}, y_i)]$ \Comment{Update datamodel
        training set}
        \EndFor
        \State $\theta \gets$ \Call{RunRegression}{T} \Comment{Predict the
        $y_i$ from the $\mask{S_i}$ vectors}
        \State\Return $\theta$ \Comment{Result: a weight vector $\theta \in
        \mathbb{R}^d$}
        \EndProcedure
    \end{algorithmic}
\end{algorithm}

\paragraph{Datamodel estimation hyperparameters.}
Recall that our algorithm comparison framework in~\Cref{sec:approach} involves estimating two sets of datamodels $\{\theta^{(1)}\}$ and $\{\theta^{(2)}\}$ for learning algorithms $\mathcal{A}_1$ and $\mathcal{A}_2$ respectively. 
In our case studies, we estimate two datamodels, $\theta^{(1)}_i$ and $\theta^{(2)}_i$  for every example $x_i$ in the test dataset.
Estimating these datamodels entail three design choices: 
\begin{itemize}
	\item {\bf Sampling scheme for train subsets:} Like in~\citet{ilyas2022datamodels}, we use $\alpha$-random subsets of the training data, where $\alpha$ denotes the subsampling fraction; we set $\alpha=50\%$ as it maximizes sample efficiency (or model reuse) for empirical influence estimation~\cite{feldman2020neural}, which is equivalent to a variant of linear datamodels~\cite{ilyas2022datamodels}.
	\item {\bf Sample size for datamodel estimation:} Recall that a datamodel training set of size $m$ corresponds to training $m$ models (e.g., $m$ ResNet18 models on \textsc{Cifar-10}) on independently sampled train subsets (or masks). We estimate datamodels on \textsc{Living17}, \textsc{Waterbirds}, and \textsc{Cifar-10} using $120k$, $50k$, and $50k$ samples (or models) per learning algorithm respectively; we make a $90-10\%$ train-validation split.
	\item {\bf $\ell_1$ sparsity regularization:} We use cross-validation to select the sparsity regularization parameter $\lambda$. Specifically, for each datamodel, we evaluate the MSE on a validation split to search over $k=50$ logarithmically spaced values for $\lambda$ along the regularization path. As in~\cite{ilyas2022datamodels}, we then re-compute the datamodel on the entire dataset with the optimal $\lambda$ value and all $m$ training examples.
\end{itemize}

\subsection{Feature transformations}
\label{app:feature_trans}

We counterfactually verify distinguishing feature candidates
(inferred from distinguishing subpopulations) by evaluating whether the corresponding transformations change model behavior as hypothesized.
Here, we describe the feature transformations used in~\Cref{sec:examples} in more detail\footnote{The code for these feature transformations is available at \url{https://github.com/MadryLab/modeldiff}.}.

\paragraph{Designing feature transformations.} 
We design feature transformations that modify examples by adding a specific patch or perturbation. 
We vary the intensity of patch-based and perturbation-based transformations via patch size $k$ and perturbation intensity $\delta$ respectively.
Additional details specific to each case study:
\begin{itemize}
	\item \textbf{Pre-training.} We use patch-based transformations in this case. For the yellow color feature, we add a $k \times k$ square yellow patch to the input. For the human face feature, we add a $k \times k$ image of a human face to the input. To avoid occlusion with objects in the image foreground, we add the human face patch to the background. We make a bank of roughly $300$ human faces using ImageNet face annotations~\cite{yang2022study} by (a) cropping out human faces from ImageNet validation examples and (b) manually removing mislabeled, low-resolution, and unclear human face images. 
	\item \textbf{Data augmentation}. We design perturbation-based transformations to verify the identified distinguishing features: spider web and polka dots. In both cases, we $\delta$-perturb each input with a random crop of a fixed grayscale spider web or yellow polka dot pattern.   
	\item \textbf{SGD hyperparameters.} We use patch-based transformations in this case study. For the black-white texture feature, we add a $k$-sized patch that loosely resembles a black-white dog nose. Similarly, for the rectangular shape feature, we add a $k$-sized patch that loosely resembles windows.
\end{itemize}

\paragraph{Evaluating feature transformations.} As shown in~\Cref{sec:examples}, given two learning algorithms $\mathcal{A}_1$ and $\mathcal{A}_2$, we evaluate whether a feature transformation $F$ changes predictions of models trained with $\mathcal{A}_1$ and $\mathcal{A}_2$ as hypothesized. To evaluate the counterfactual effect of transformation $F$ on model $M$, we evaluate the extent to which applying $F$ to input examples $x$ increases the confidence of models in a particular class $y$. In our experiments, we estimate this counterfactual effect by averaging over all test examples and over $500$ models trained with each learning algorithm.

\subsection{Training infrastructure}
\label{app:infra}

\paragraph{Data loading.}
We use \texttt{FFCV}\footnote{Webpage: \url{http://ffcv.io}}~\citep{leclerc2022ffcv}, which removes the data loading bottleneck for smaller models, gives a $3\text{-}4 \times$ improvement in throughput (i.e., number of models a day per GPU).

\paragraph{Datamodels regression.} In addition to \texttt{FFCV}, we use the \texttt{fast-l1} package\footnote{Github repository: \url{https://github.com/MadryLab/fast_l1}}---a SAGA-based GPU solver for $\ell_1$-regularized regression---to parallelise datamodel estimation.

\paragraph{Computing resources.}
\label{app:resources}

We train our models on a cluster of machines, each with 9 NVIDIA A100 or V100 GPUs and 96 CPU cores.
We also use half-precision to increase training speed.

%% file: appendices/more_related_work.tex
\section{Additional related work}

In the main paper, we focused our discussion of related work to those closely related to the problem of algorithm comparison. Here, we give discuss additional related work on interpretability and debugging of model biases.

\paragraph{Interpretability, explainability, and debugging.}
Our method hinges on the interpretability of the extracted subpopulation.
A long line of prior work propose different interpretability and explainability methods for models.

Local explanation methods include saliency maps \citep{simonyan2013deep,dabkowski2017real,adebayo2018sanity}, surrogate models such as LIME \citep{ribeiro2016}, and Shapley values \citep{lundberg2017unified}.
Our method is similar to per-example based interpretability methods such as influence functions \citep{koh2017understanding} in that our interpretation is based on data; however, our analysis differs from these priors methods in that it looks at entire subpopulations of inputs.

Global interpretability and debugging methods often leverage the rich latent space of neural networks in order to identify meaningful subpopulations or biases more automatically.
Concept activation vectors and its variants \citep{kim2018interpretability,abid2022meaningfully,ghorbani2019towards} help decompose model predictions into a set of concepts. Other recent works \citep{eyuboglu2022domino,jain2022distilling} leverage the recent cross-model representations along with simple models---mixture models and SVMs, respectively---to identify coherent subpopulations or slices.
Other methods \citep{wong2021leveraging,singla2021salient} analyze the neurons of the penultimate layer of (adversarially robust) models to identify spurious features.
Our framework can be viewed as leveraging a different embedding space, that of datamodel representations, to analyze model predictions.

\paragraph{Robustness to specific biases.}
In applying our framework across the three case studies, we identify a number of both known and unknown biases.
A large body of previous work aims at finding and debugging these biases:
Priors works investigate specific biases such as the role of texture \citep{geirhos2019imagenet} or backgrounds \citep{xiao2020noise} by constructing new datasets. \citet{leclerc20213db} automate many of these studies in the context of vision models with a render-based framework.
Crucially, these works rely on having control over data generation and having candidate biases ahead of time.
See the ``Connections to prior work'' within each case study in~\Cref{sec:examples} for additional details.

%% file: appendices/analysis.tex
\section{Additional analysis of distinguishing subpopulations}
\label{app:hooman}

As outlined in~\Cref{sec:approach}, we analyze distinguishing subpopulations to infer (and test) distinguishing feature transformations.
In this section, we present additional analysis in order to substantiate the distinguishing features inferred in each case study.

First, we describe two additional tools that we use to analyze subpopulations surfaced by principal components (PCs) of residual datamodels.
\begin{itemize}
	\item {\bf Class-specific visual inspection.} As shown in~\Cref{sec:examples}, the subpopulation of test examples whose datamodels have maximum projection onto PCs of residual datamodels largely belong to same class; these subpopulation mostly surface images from the same class. So, a simple-yet-effective way to identify \emph{subpopulation-specific} distinguishing feature(s) is to just visually contrast the surfaced subpopulation from a set of randomly sampled examples that belong to the same class.
	\item {\bf Relative influence of training examples.} Given a subset of test examples $S' \subset S$, can we identify a set of training examples $T' \subset T$ that strongly influence predictions on $S'$ when models are trained with algorithm $\mathcal{A}_1$ but not when trained with $\mathcal{A}_2$? Given datamodel representations $\{\smash{\dma_i}\}$ for $\mathcal{A}_1$ and $\{\smash{\dma_i}\}$ for $\mathcal{A}_2$, we apply a two-step (heuristic) approach identify training examples with high influence on $\mathcal{A}_1$ relative to $\mathcal{A}_2$:
		\begin{itemize}
			\item First, given learning algorithm $\mathcal{A}_i$ and test subset $S'$, we estimate the aggregate (positive or negative) influence of training example $x_k$ on subset $S$ by taking the absolute sum over the corresponding datamodel weights:  $\smash{\sum_{j \in S'} |\theta^{(i)}_{jk}|}$.
			\item Then, we take the absolute difference between the aggregate influence estimates of training example $x_k$ using $\smash{\dma}$ and $\smash{\dmb}$. This difference measures the \emph{relative influence} of training example $x_k$ on predictions of test subset $S$ when models are trained with algorithm $\mathcal{A}_1$ instead of algorithm $\mathcal{A}_2$.
		\end{itemize}
		In our analysis, we (a) identify training examples that have top-most relative influence estimates and then (b) visually contrast the subsets of test examples (one for each learning algorithm) that are most influenced by these training examples.
\end{itemize}

\clearpage
\subsection{Case study: Standard data augmentation}
\label{subsec:app_data_aug}
Our case study on \textsc{Living17} data in~\Cref{ssec:data_aug} shows that standard data augmentation can amplify co-occurrence bias (spider web) and texture bias (polka dots).
We further substantiate these findings with relative influence analysis (\Cref{fig:ti_dataaug}) and class-specific visual inspection (\Cref{fig:topkrandom_data_aug}).
\begin{figure}[h]
	\centering
	\begin{subfigure}[t]{0.81\textwidth}
		\centering
		\includegraphics[width=\textwidth]{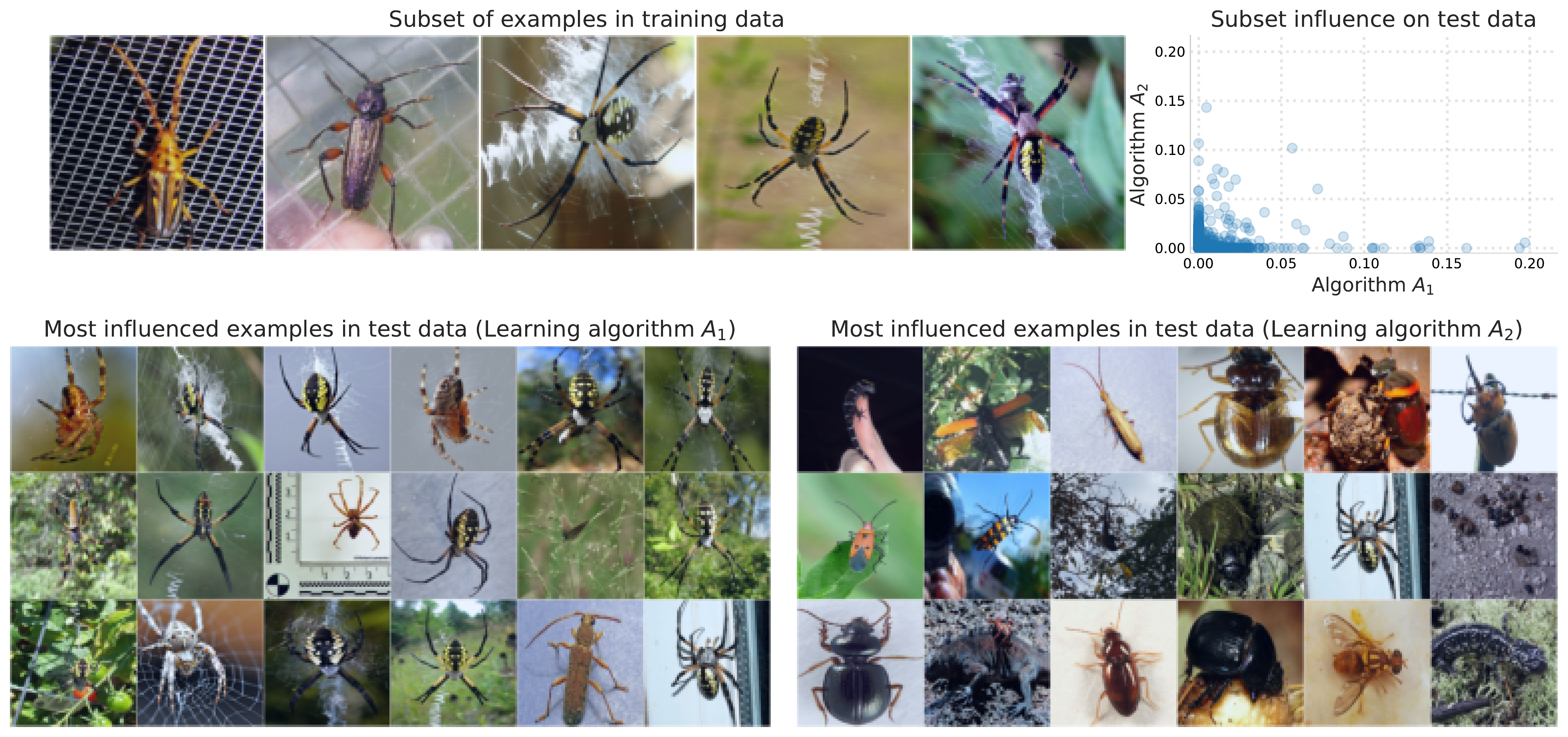}
		\caption{``Spider web'' feature}
	\end{subfigure}
	\begin{subfigure}[t]{0.81\textwidth}
		\includegraphics[width=\textwidth]{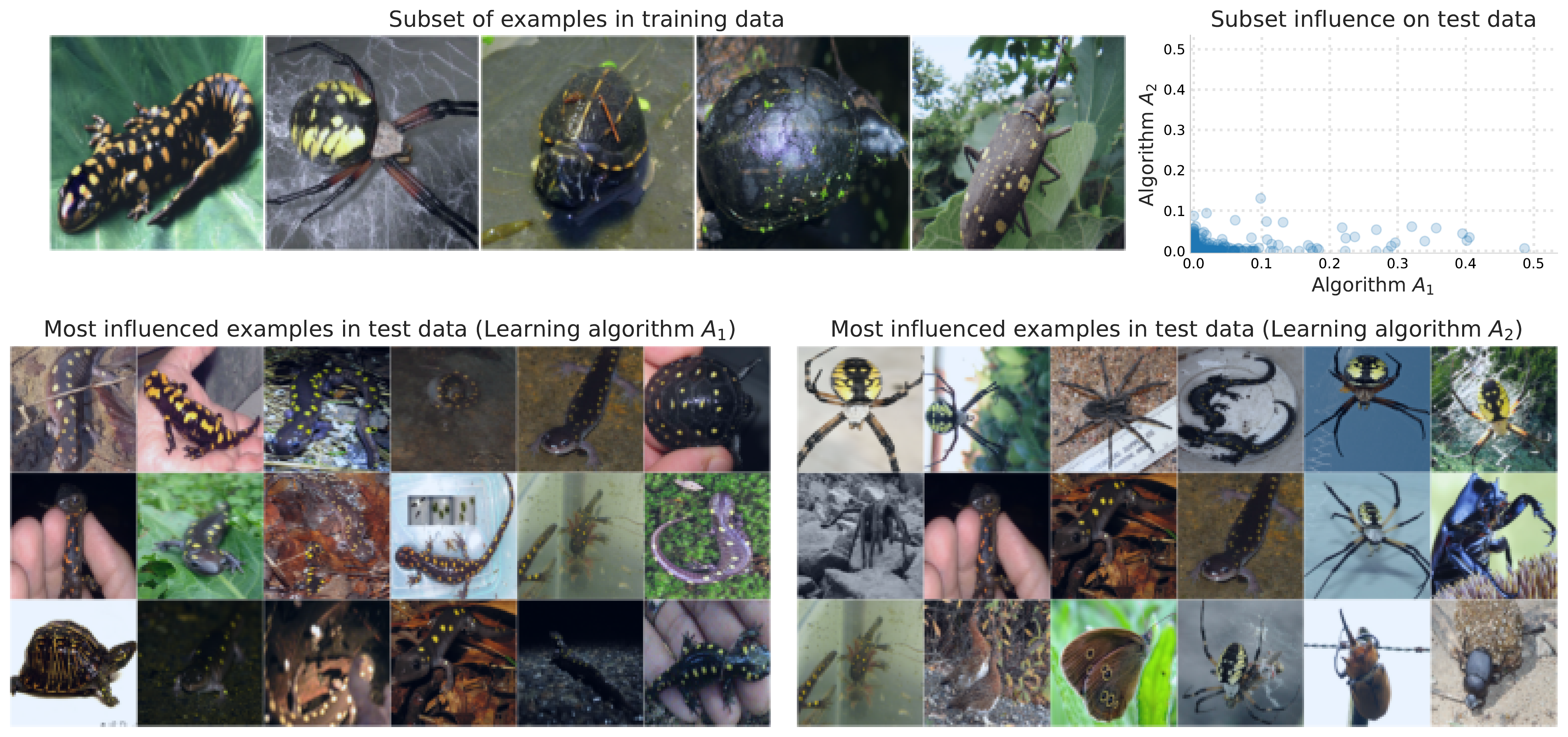}
		\caption{``Polka dots'' feature}
	\end{subfigure}
	\vspace{-6px}
	\caption{{\bf Relative influence of training data on \textsc{Living17} subpopulations}. {\bf Panel (a)}: Training images that contain web-like patterns have high relative influence on the ``spider web'' test subpopulation (see~\Cref{fig:ev_pca-living17}). These images strongly influence model predictions on test images that contain spider webs (in bottom row) only when models are trained with augmentation (algorithm $\mathcal{A}_1$). {\bf Panel (b)}: Training images that contain yellow-black texture have high relative influence on the ``polka dots'' test subpopulation  (see~\Cref{fig:ev_pca-living17}). These images strongly influence model predictions on test images of salamanders with yellow polka dots (in bottom row) {only when} models are trained with augmentation (algorithm $\mathcal{A}_1$).}
	\label{fig:ti_dataaug}
\end{figure}

\clearpage
\subsection{Case study: ImageNet pre-training}
Our case study on \textsc{Waterbirds} data shows that ImageNet pre-training  reduces dependence on the ``yellow color'' feature, but introduces dependence the ``human face'' feature.
We support these findings with relative influence analysis in~\Cref{fig:ti_pretrain} and additional inspection in~\Cref{fig:topkrandom_pretrain}.
\begin{figure}[h]
	\centering
	\begin{subfigure}[t]{0.8\textwidth}
		\centering
		\includegraphics[width=\textwidth]{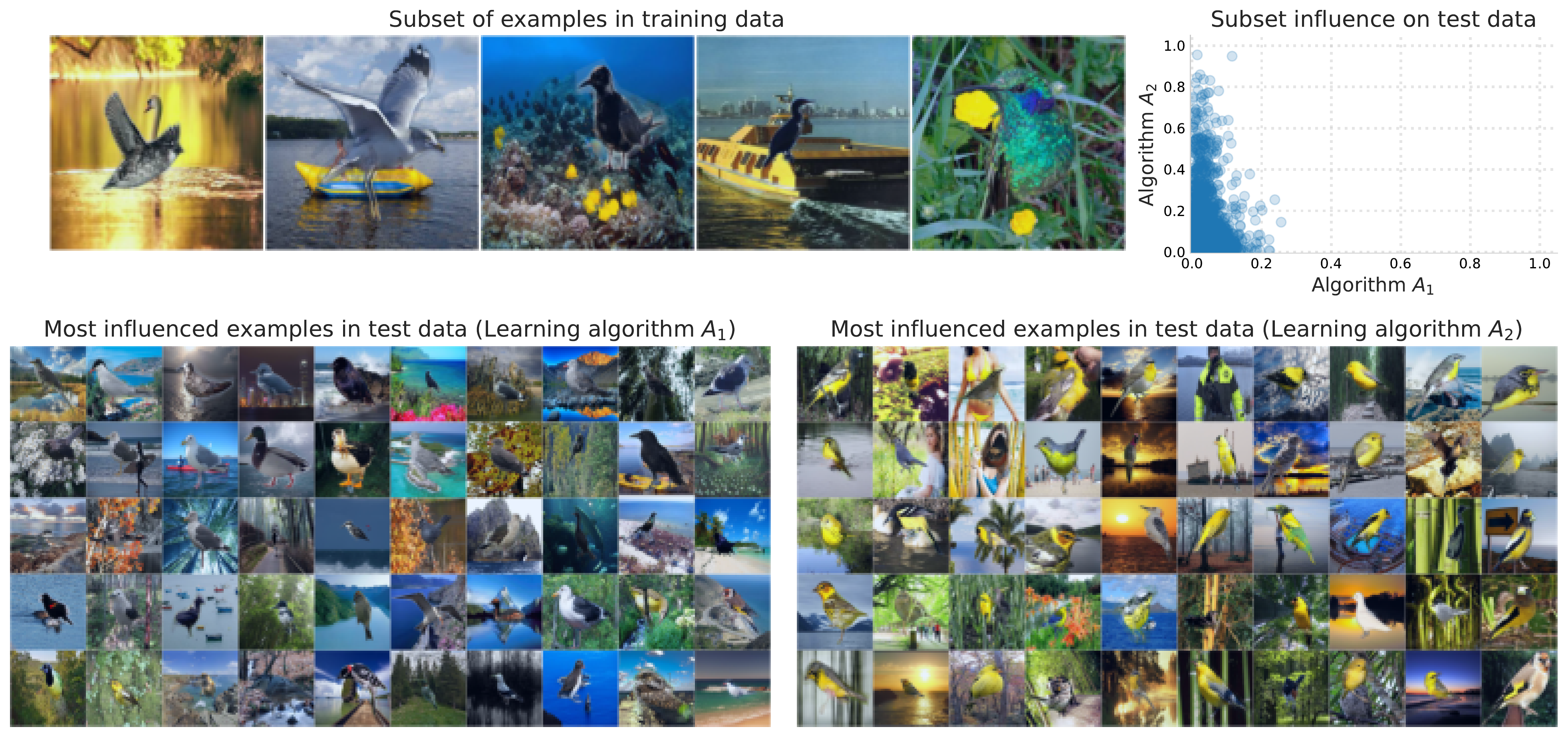}
		\caption{``Yellow color'' feature}
	\end{subfigure}
	\begin{subfigure}[t]{0.8\textwidth}
		\includegraphics[width=\textwidth]{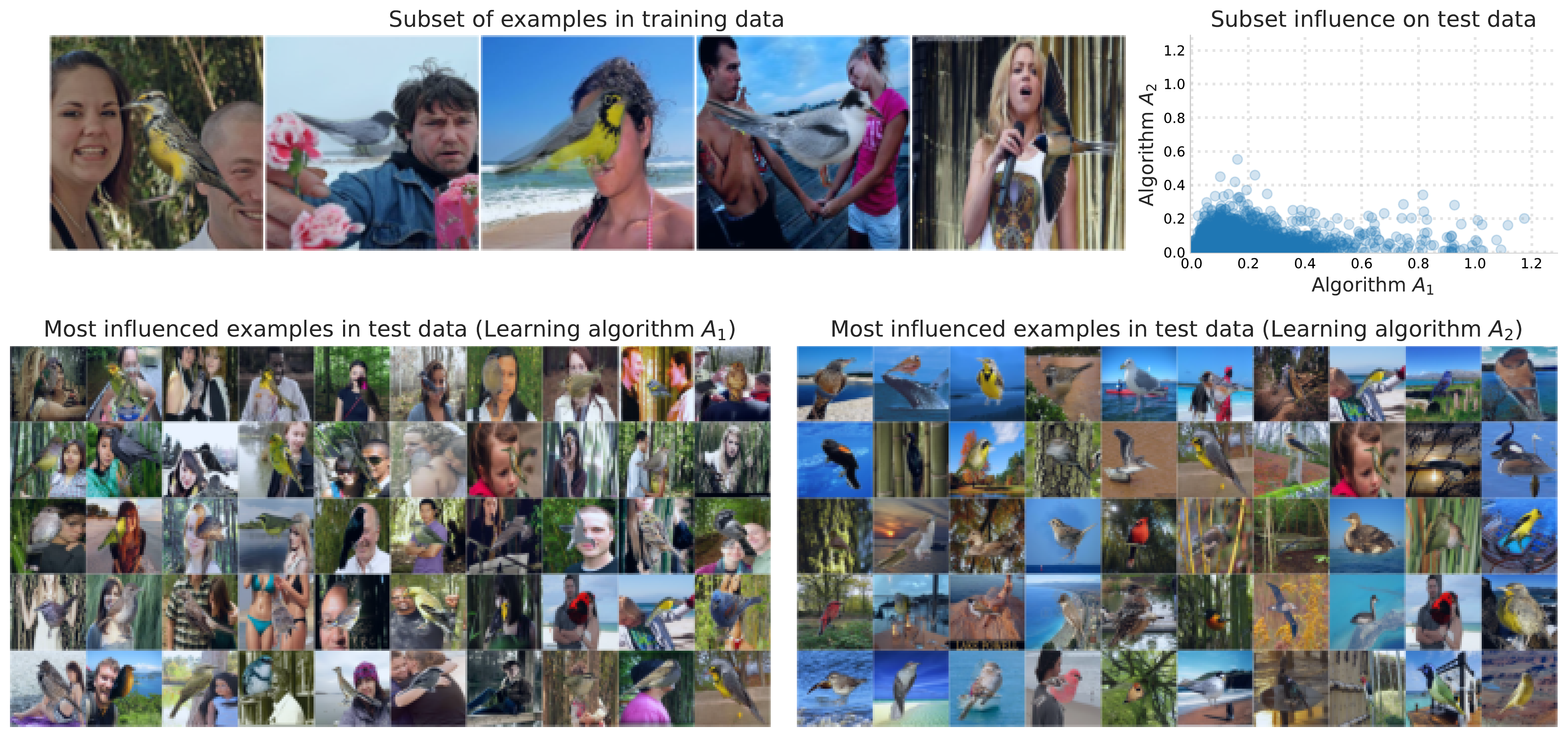}
		\caption{``Human face'' feature}
	\end{subfigure}
	\vspace{-7px}
	\caption{{\bf Relative influence of training data on \textsc{Waterbirds} subpopulations}. {\bf Panel (a)}: Training images with yellow objects in the background have high relative influence on the ``yellow color'' test subpopulation (see~\Cref{fig:ev_pca-waterbirds}). These images strongly influence model predictions on test images that have yellow birds / objects (bottom row) only when models are trained from scratch (algorithm $\mathcal{A}_2$). {\bf Panel (b)}: Training images that contain human faces in the background have high relative influence on the ``human face'' test subpopulation (see~\Cref{fig:ev_pca-waterbirds}). These images strongly influence model predictions on test images (in bottom row) with human face(s) {only when} models are pre-trained on ImageNet (algorithm $\mathcal{A}_1$).}
	\label{fig:ti_pretrain}
\end{figure}

\clearpage
\subsection{Case study: SGD hyperparameters}
We analyze relative influence (\Cref{fig:ti_sgd}), and class-specific subpopulations (\Cref{fig:topkrandom_sgd}) to hone in on two instances of distinguishing features--black-and-white texture and rectangular shape--- in~\textsc{Cifar-10} data that are amplified by low SGD noise.

\begin{figure}[h]
	\centering
	\begin{subfigure}[t]{0.81\textwidth}
		\centering
		\includegraphics[width=\textwidth]{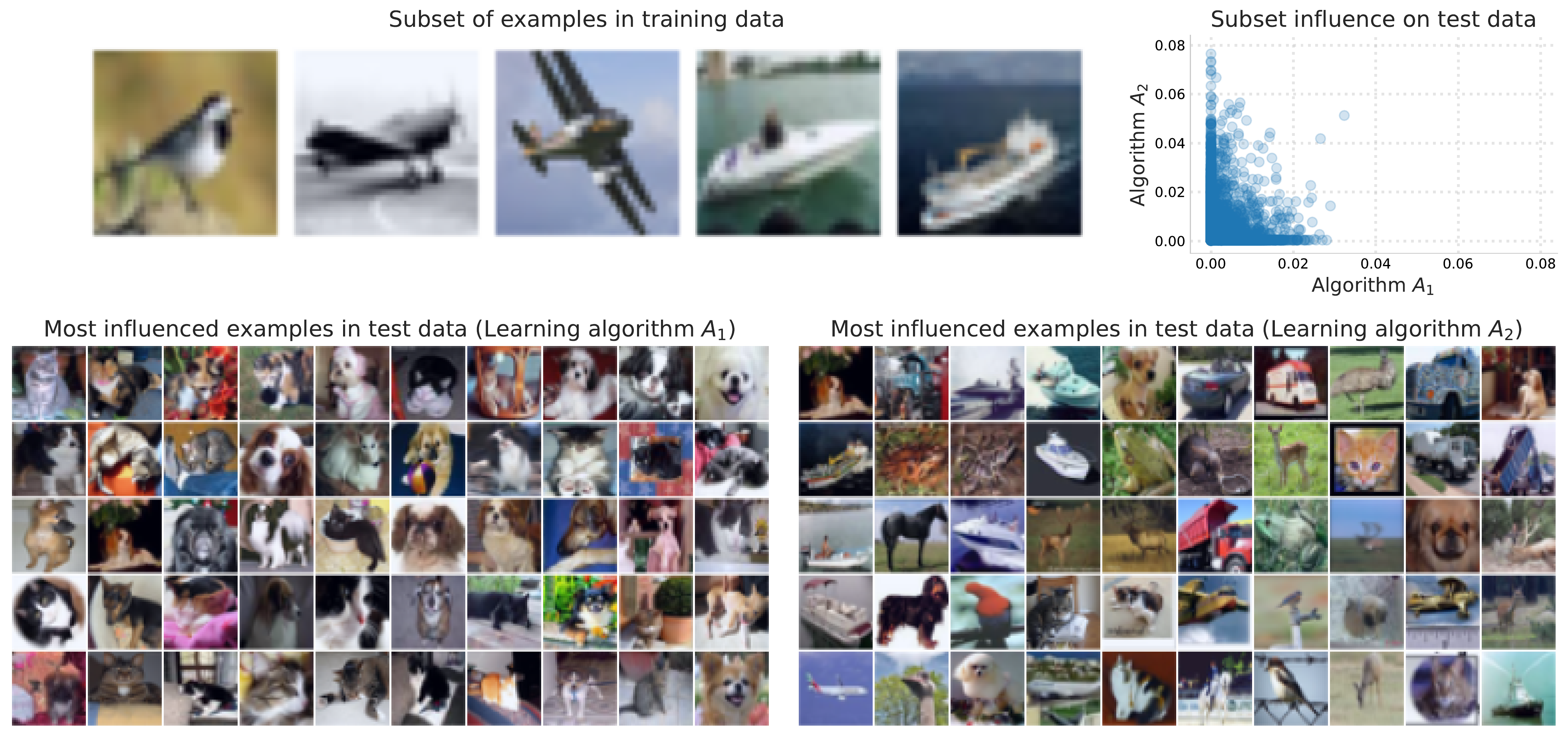}
		\caption{``Black-white texture`` bias}
	\end{subfigure}
	\begin{subfigure}[t]{0.81\textwidth}
		\includegraphics[width=\textwidth]{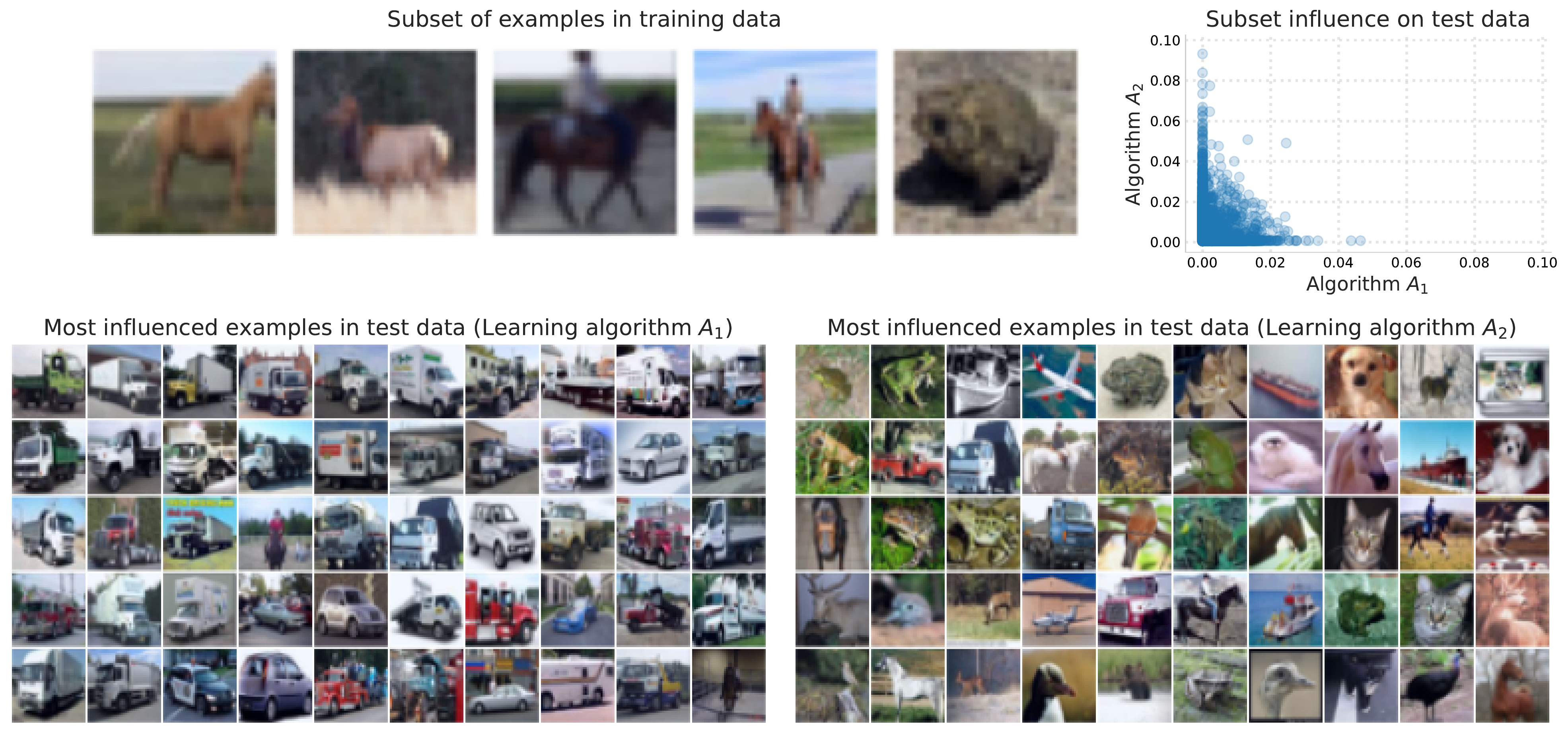}
		\caption{``Rectangular shape`` bias}
	\end{subfigure}
	\vspace{-7px}
	\caption{{\bf Relative influence of training data on \textsc{Cifar-10} subpopulations}. {\bf Panel (a)}: Training images with black-white objects have high relative influence on the ``black-white'' dog subpopulation (see~\Cref{fig:ev_pca-cifar}). These images influence model predictions on test images of black-white dogs (in bottom row) only when models are trained with low SGD noise (alg. $\mathcal{A}_2$). {\bf Panel (b)}: Training images with high-contrast rectangular components in the background have high relative influence on the ``rectangular shape'' truck subpopulation (see~\Cref{fig:ev_pca-cifar}). These images influence model predictions on test images of front-facing trucks with prominent rectangular components (in bottom row) {only when} models are trained with low SGD noise (alg. $\mathcal{A}_2$).}
	\label{fig:ti_sgd}
\end{figure}

\begin{figure}[t]
	\centering
    \includegraphics[width=0.9\textwidth]{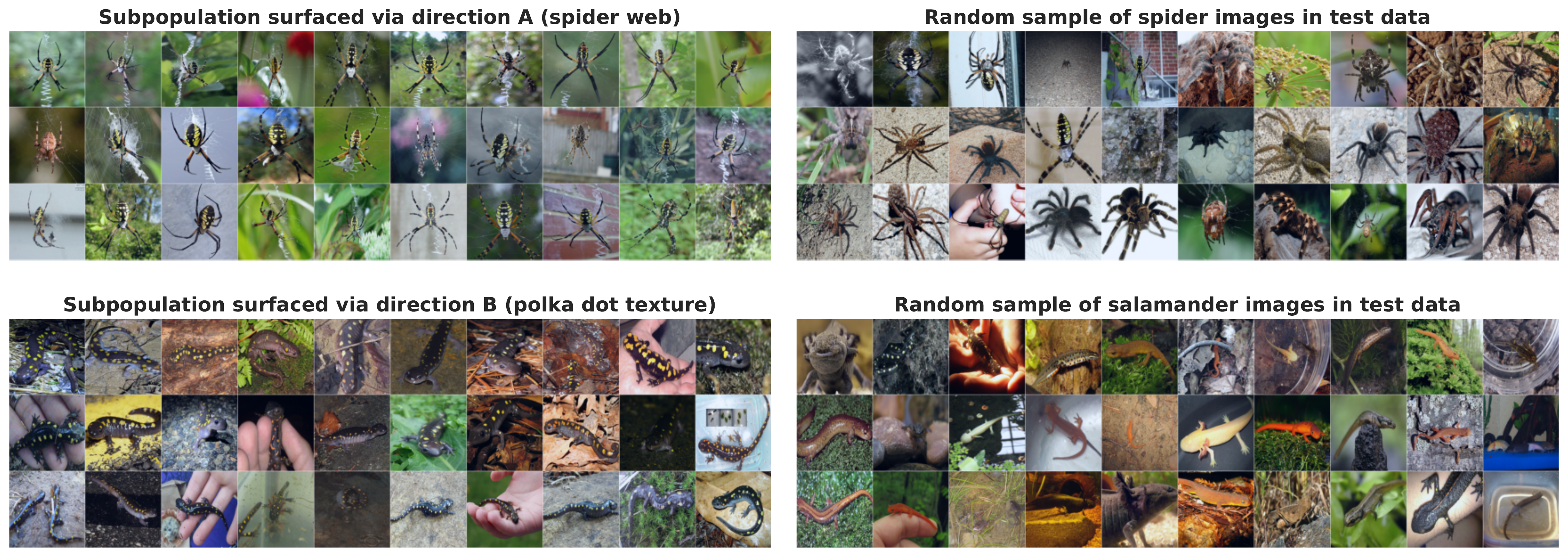}
    \caption{{Class-specific visual inspection of \textsc{Living17} subpopulations.} {\bf (Top)} In contrast to random \textsc{Living17} images of spiders, the ``spider web'' subpopulation surfaces spiders with a prominent spider web in the background. {\bf (Bottom)} Unlike random \textsc{Living17} images of salamanders, the ``polka dots'' subpopulation surfaces salamanders that have a yellow-black polka dot texture.}
    \label{fig:topkrandom_data_aug}
\end{figure}

\begin{figure}[t]
	\centering
    \includegraphics[width=0.9\textwidth]{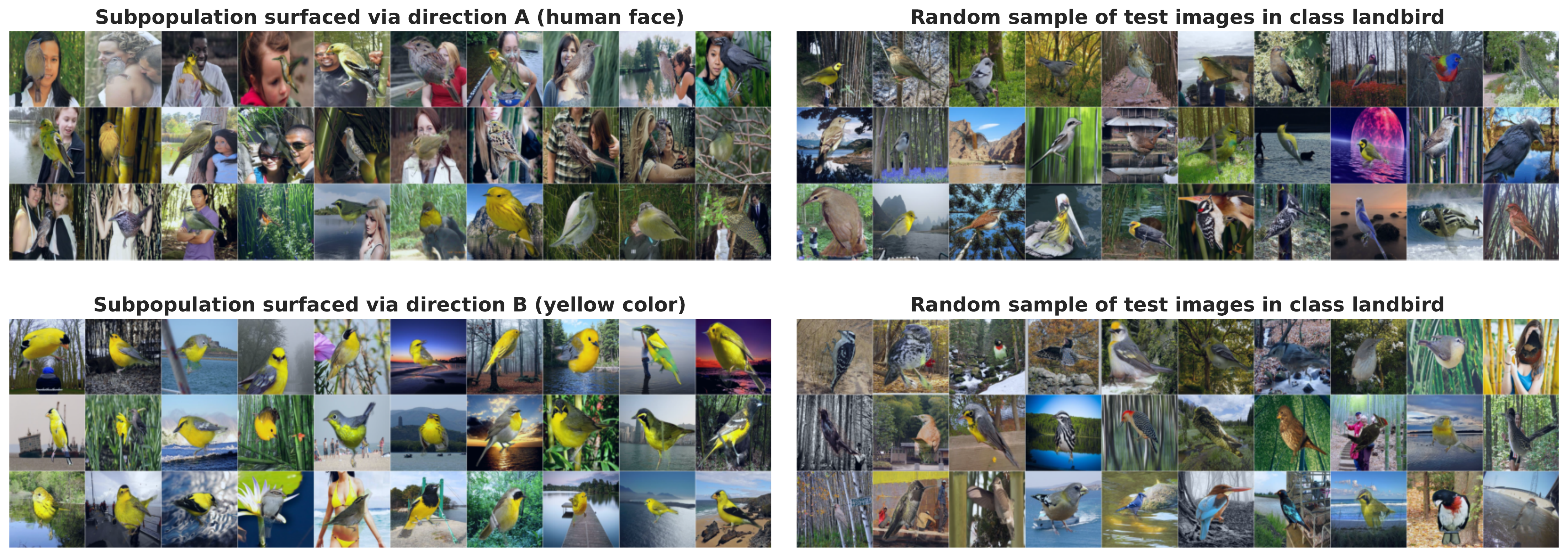}
    \caption{{Class-specific visual inspection of \textsc{Waterbirds} subpopulations.} {\bf (Top)} In contrast to random ``landbird'' images, the ``human face'' subpopulation surfaces landbirds with human face(s) in the background. {\bf (Bottom)} Unlike random ``landbird'' images, the ``yellow color'' subpopulation surfaces images with yellow birds \emph{or} yellow objects in the background.}
    \label{fig:topkrandom_pretrain}
\end{figure}

\begin{figure}[t]
	\centering
    \includegraphics[width=0.9\textwidth]{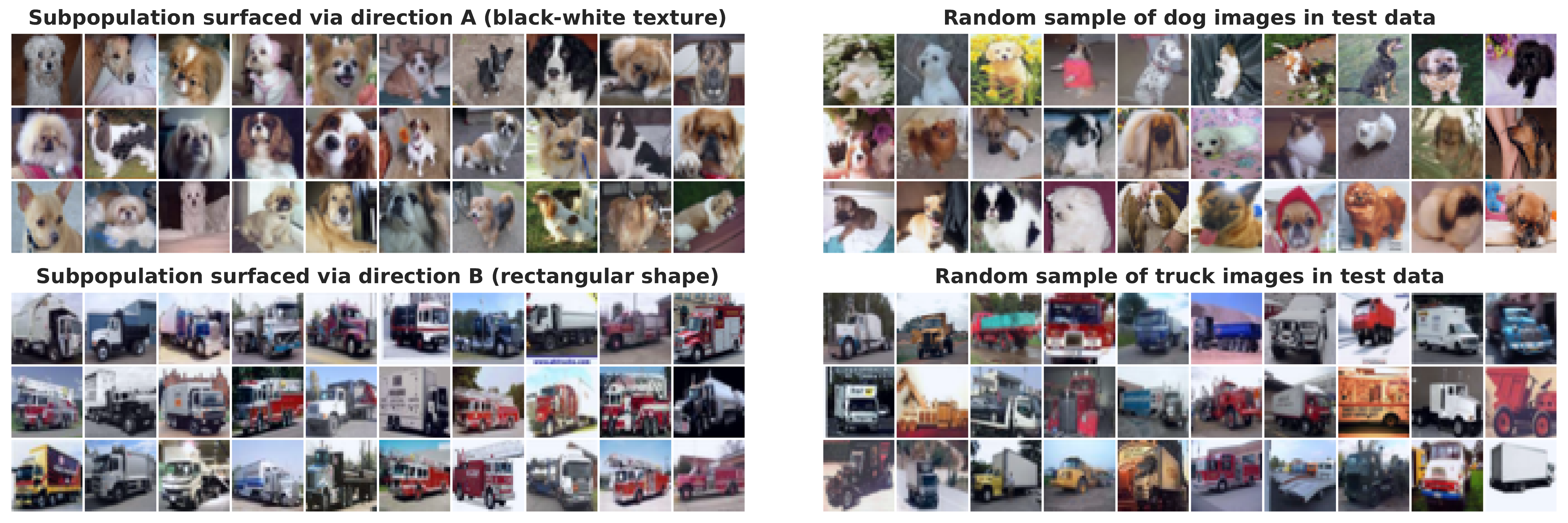}
    \caption{{Class-specific visual inspection of \textsc{Cifar-10} subpopulations.} In contrast to random images of dogs (top) and trucks (bottom), the ``black-white'' and ``rectangular shape'' subpopulations surface images of black-white dogs and front-facing trucks with multiple rectangular components respectively.}
    \label{fig:topkrandom_sgd}
\end{figure}

%% file: appendices/counterfactuals.tex
\section{Additional evaluation of feature transformations}

Distinguishing feature transformations, which we recall from~\Cref{sec:approach}, are functions that, when applied to data points, change the predictions of one model class---but not the other---in a consistent way.
In our case studies, we design distinguishing feature transformations that counterfactually verify features inferred from distinguishing subpopulations.
Our findings in~\Cref{sec:examples} use feature transformations to quantitatively measure the relative effect of the identified features on models trained with different learning algorithms.
In this section, we present additional findings on feature transformations for each case study:

\subsection{Case study: Standard data augmentation}
In~\Cref{ssec:data_aug}, we showed that standard data augmentation---horizontal flips and random crops---amplifies \textsc{Living17} models' reliance on ``spider web'' and ``polka dots'' to predict spiders and salamanders respectively.~\Cref{fig:more-ft:data_aug} verifies our findings over a larger range of perturbation intensity $\delta$ values. We also observe that decreasing the minimum allowable crop size in \texttt{RandomResizedCrop} from $1.0$ (i.e., no random cropping) to $0.08$ (default \texttt{torchvision} hyperparameter) increases models' sensitivity to both feature transformations.

\begin{figure}[h]
	\centering
	\begin{subfigure}[t]{0.9\textwidth}
		\centering
		\includegraphics[width=\textwidth]{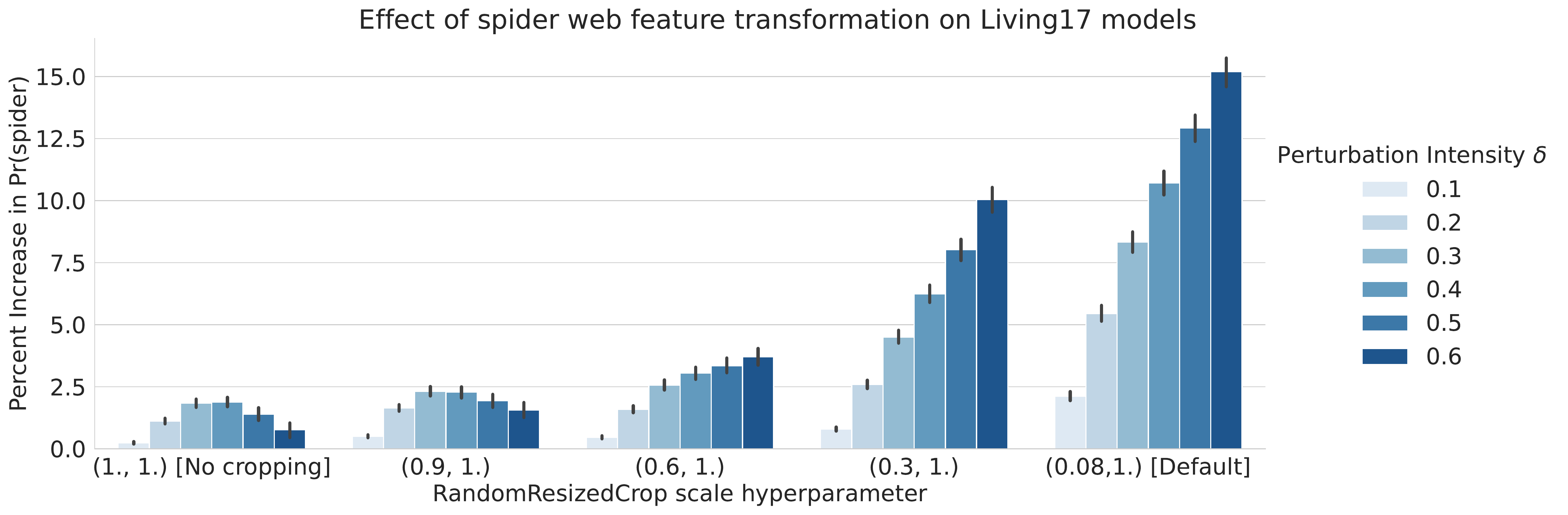}
		\caption{``Spider web'' feature}
	\end{subfigure}
	\begin{subfigure}[t]{0.9\textwidth}
		\includegraphics[width=\textwidth]{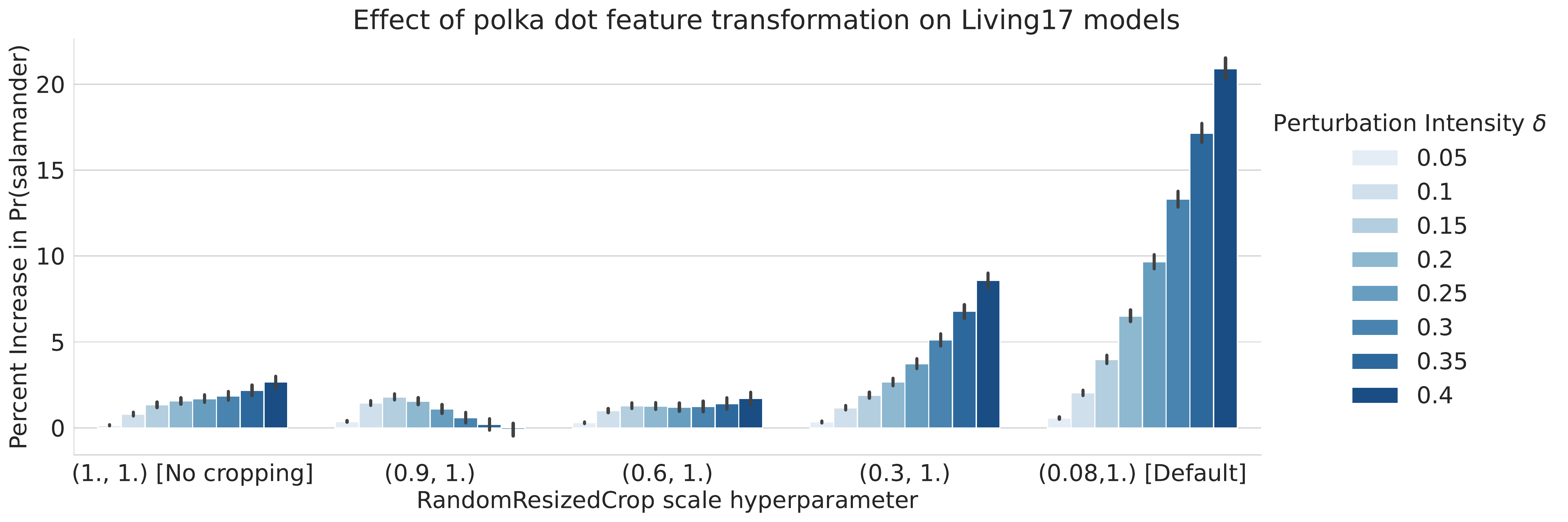}
		\caption{``Polka dots'' feature}
	\end{subfigure}
	\vspace{-7px}
	\caption{{\bf Additional evaluation of \textsc{Living17} feature transformations.} The top and bottom row evaluate the effect of ``spider web'' and ``polka dot'' feature transformations on models trained with different data augmentation schemes. Increasing the intensity of the transformations and the minimum crop size of \texttt{RandomResizedCrop} augmentation (via \texttt{scale} hyperparameter) increases the sensitivity of models to both feature transformations in a consistent manner.}
	\label{fig:more-ft:data_aug}
\end{figure}

\clearpage
\subsection{Case study: ImageNet pre-training}
In~\Cref{ssec:pretrain}, we showed that fine-tuning ImageNet-pretrained ResNet50 models on \textsc{Waterbirds} data instead of training from scratch alters the relative importance of two spurious features: ``yellow color'' and ``human face''. In~\Cref{fig:more-ft:pretrain}, we show that both feature transformations alter the predictions of ImageNet-pretrained ResNet18 and ImageNet-pretrained ResNet50 models in a similar way.

\begin{figure}[h]
	\centering
	\begin{subfigure}[t]{\textwidth}
		\centering
		\includegraphics[width=0.96\textwidth]{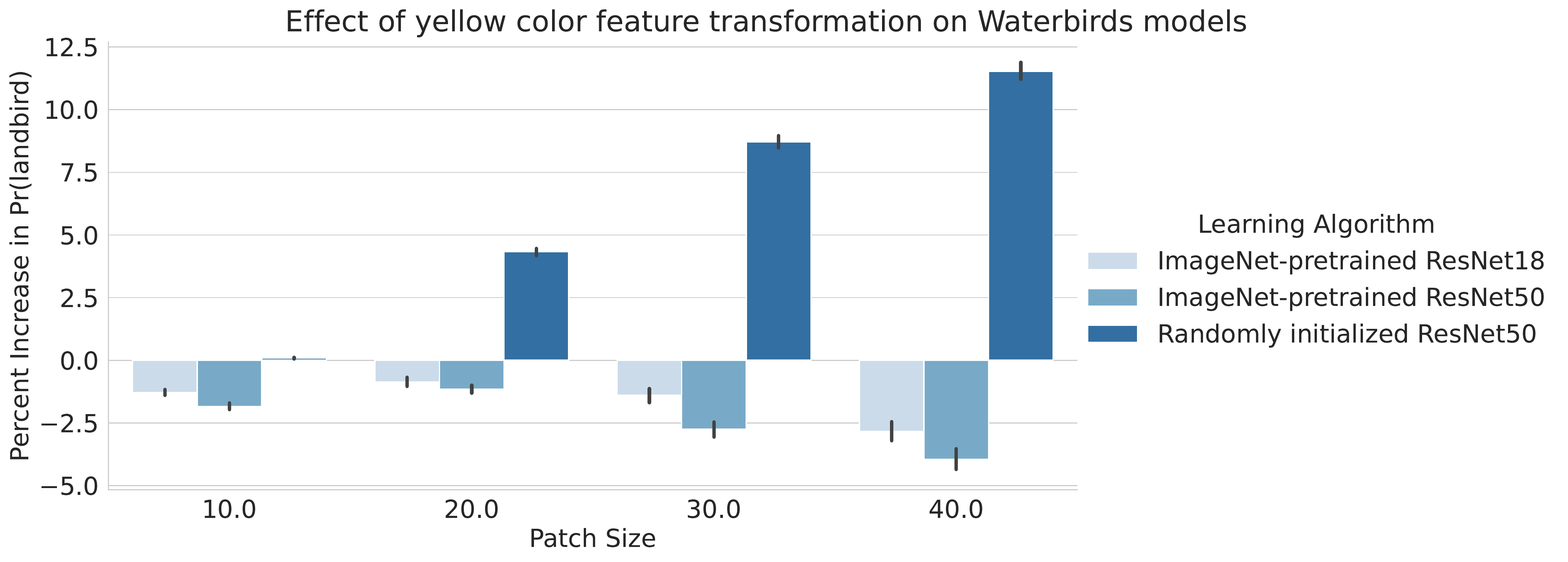}
		\caption{``Yellow color'' feature}
	\end{subfigure}
	\begin{subfigure}[t]{\textwidth}
		\includegraphics[width=\textwidth]{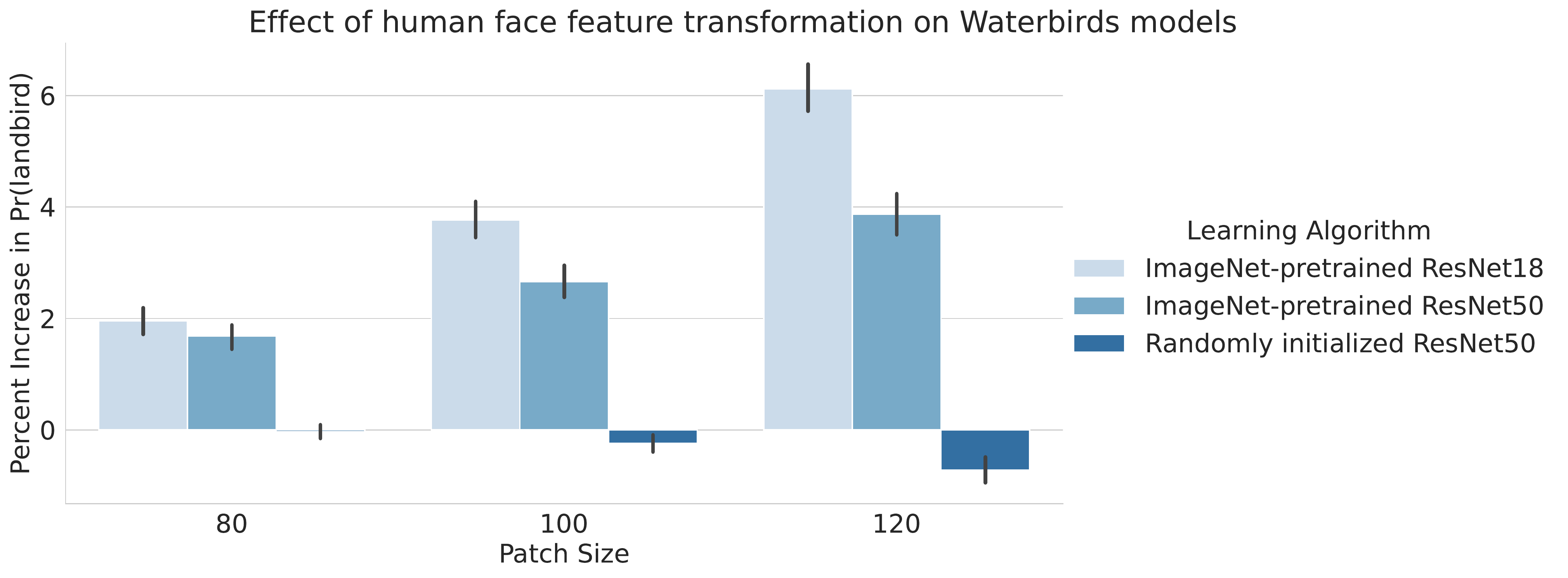}
		\caption{``Human face'' feature}
	\end{subfigure}
	\vspace{-5px}
	\caption{{\bf Additional evaluation of \textsc{Waterbirds} feature transformations.} The top and bottom row evaluate the effect of ``yellow color'' and ``human face'' feature transformations on models trained with and without ImageNet pre-training. In both cases, unlike ResNet50 models trained from scratch, ImageNet-pretrained ResNet18 and ResNet50 models are sensitive to the ``human face'' transformation but not to the ``yellow color'' transformation.}
	\label{fig:more-ft:pretrain}
\end{figure}

\clearpage
\subsection{Case study: SGD hyperparameters}
In~\Cref{ssec:lr}, we showed that reducing SGD noise results in \textsc{Cifar-10} models that are more sensitive to certain features, such as rectangular shape bias and black-white texture to predict trucks and dogs. In~\Cref{fig:more-ft:sgd}, we evaluate how feature transformations change class-wise predictions of models trained with different SGD learning rate and batch size hyperparameters.  

\begin{figure}[h]
	\centering
	\begin{subfigure}[t]{\textwidth}
		\centering
		\includegraphics[width=\textwidth]{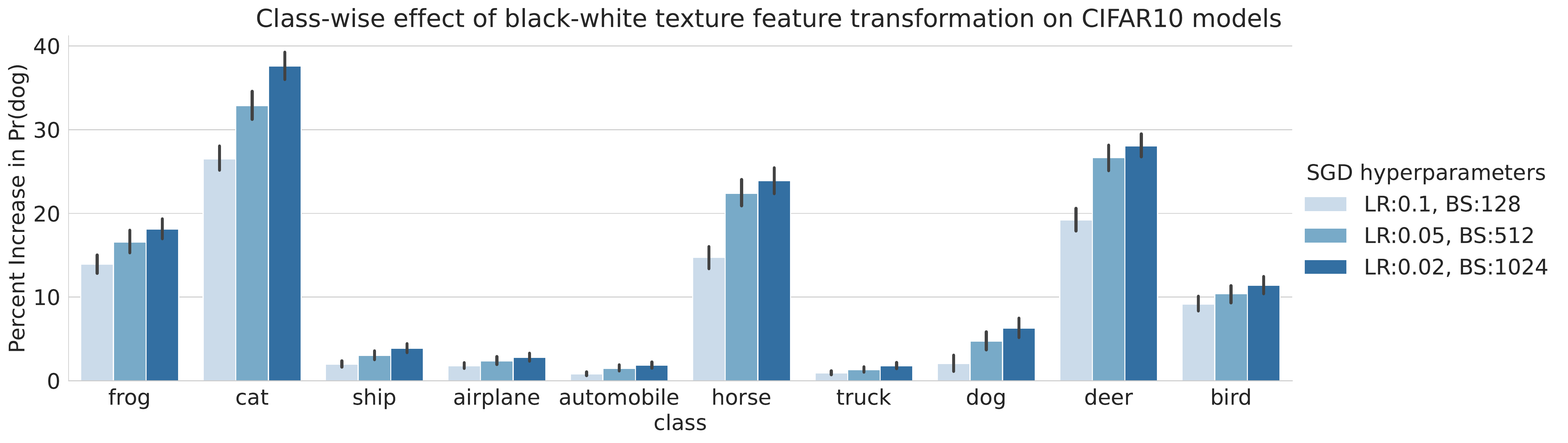}
		\caption{``Black-white texture'' feature}
	\end{subfigure}
	\begin{subfigure}[t]{\textwidth}
		\includegraphics[width=\textwidth]{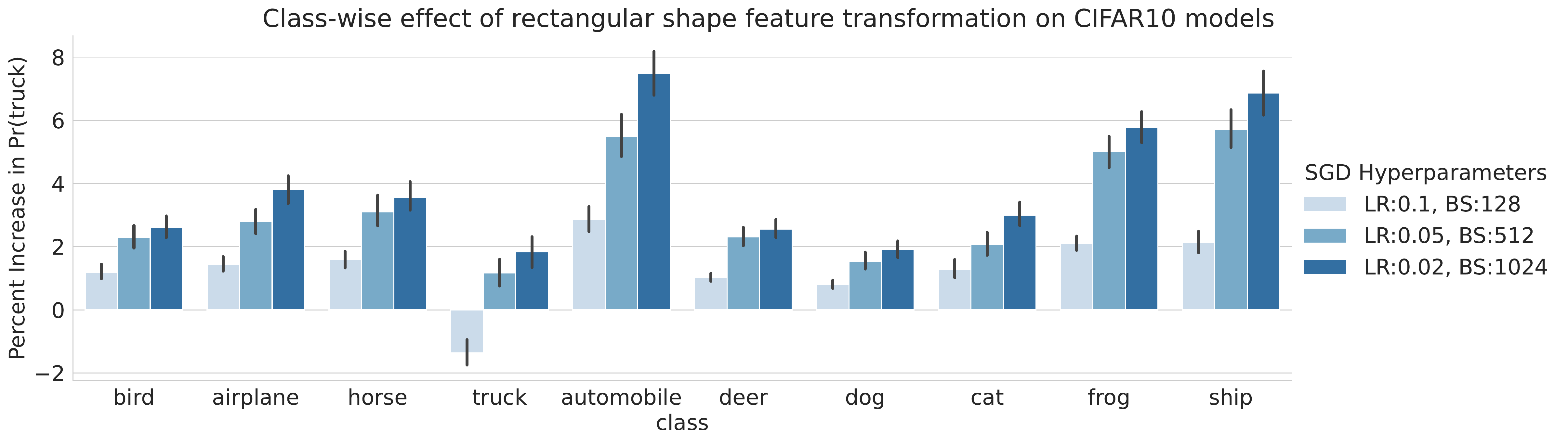}
		\caption{``Rectangular shape'' feature}
	\end{subfigure}
	\vspace{-5px}
	\caption{{\bf Additional evaluation of \textsc{Cifar-10} feature transformations.} The top and bottom row evaluate the effect of ``black-white texture'' and ``rectangular shape'' feature transformations on \textsc{Cifar-10} models trained with high (light blue), medium, and low (dark blue) SGD noise. In both cases, models trained with higher SGD noise are, on average, more sensitive to these transformations across all classes. Furthermore, the effect of the transformations are class-dependent---model predictions on transformed examples from semantically similar classes differ to a greater extent.}
	\label{fig:more-ft:sgd}
\end{figure}

%% file: appendices/baseline.tex
\section{Miscellaneous results}

\subsection{Aggregate metric for algorithm comparison}
\label{app:cosinesim}

We can repurpose our framework as a similarity metric that quantifies the similarity of models trained with different learning algorithms in a more global manner.
A straightforward approach to output a similarity score (or distribution) is to compute the cosine similarity of datamodel vectors.
More concretely, let $\smash{\dma_i}$ and $\smash{\dmb_i}$ denote the datamodels of example $x_i$ with respect to models trained using learning algorithms $\mathcal{A}_1$ and $\mathcal{A}_2$.
Then, the cosine similarity between $\smash{\dma_i}$ and $\smash{\dmb_i}$ measures the extent  to which models trained with $\mathcal{A}_1$ and $\mathcal{A}_2$ depend on the same set of training examples to make predictions on example $x_i$.

We apply this metric to two case studies---pre-training (\textsc{Waterbirds}) and SGD noise (\textsc{Cifar-10})---in~\Cref{fig:cosine-sim}.
Specifically,~\Cref{fig:cosine-sim} plots the distribution of cosine similarity of datamodels for multiple learning algorithms over all test examples.
The left subplot shows that on \textsc{Waterbirds}, ImageNet-pretrained ResNet50 models are, on average, more similar to ImageNet-pretrained ResNet18 models than to ResNet50 models pre-trained on synthetically generated data \cite{baradad2021learning} and models trained from scratch.
The right subplot shows that on \textsc{CIFAR-10}, ResNet9 models trained with high SGD noise are more similar to smaller-width ResNet9 models trained with high SGD noise than to ResNet9 models trained with low SGD noise.

\begin{figure}[h]
    \includegraphics[width=\textwidth]{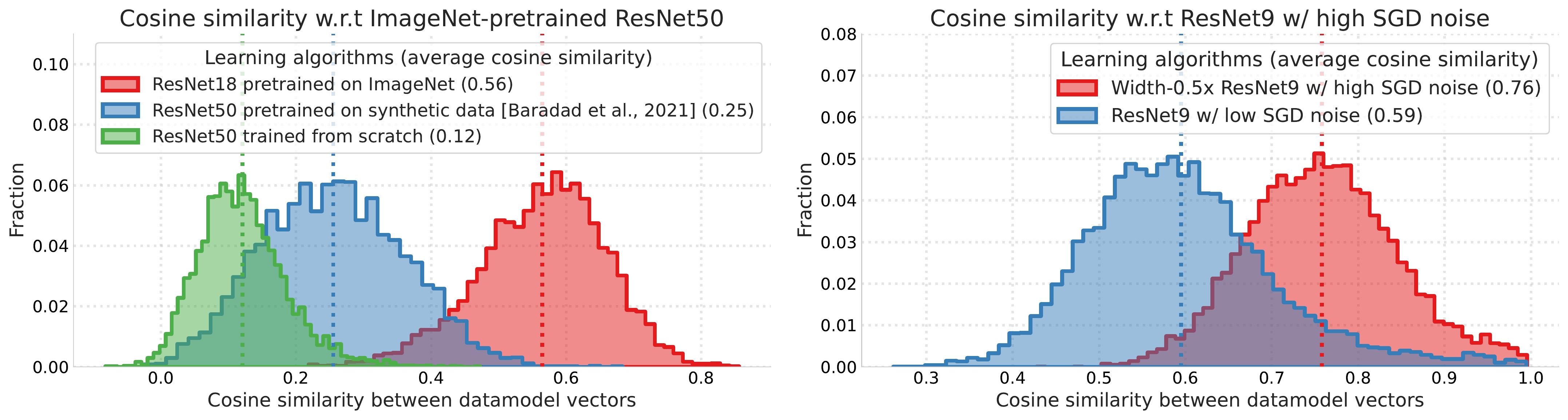}
    \caption{\textbf{Datamodel cosine similarity.} We use cosine similarity between two datamodel vectors as an aggregate metric to quantify the similarity of models trained with different learning algorithms. {\bf (Left)} On \textsc{Waterbirds} data, datamodels of ImageNet-pretrained ResNet50 and ResNet18 models are more similar to each other than to models pre-trained on synthetically generated data and models trained from scratch. {\bf (Right)} On \textsc{Cifar-10} data, ResNet models trained with high SGD noise are more similar to each other to ResNet models trained with low SGD noise.}
    \label{fig:cosine-sim}
\end{figure}

\clearpage
\subsection{Explained variance of residual datamodel principal components}
Recall from~\Cref{sec:approach} that the fraction of variance in datamodel representations $\smash{\{\theta_x^{(i)}\}}$ explained by training direction $v$ signifies the importance of the direction (or, combination of training examples) to predictions of models trained with algorithm $\mathcal{A}_i$.
Through our case studies in~\Cref{sec:examples}, we show that the top $5-6$ principal components (PCs) of residual datamodels $\resdmab$ correspond to training directions that have high explained w.r.t. datamodels of algorithm $\mathcal{A}_1$ but not $\mathcal{A}_2$, and vice versa.
\Cref{fig:more-ev} shows that the top-$100$ PCs of residual datamodel $\resdmab$ (resp., $\resdmba$) have more (resp., less) explained variance on datamodel $\dma$ than on datamodel $\dmb$.

\begin{figure}[h]
    \includegraphics[width=\textwidth]{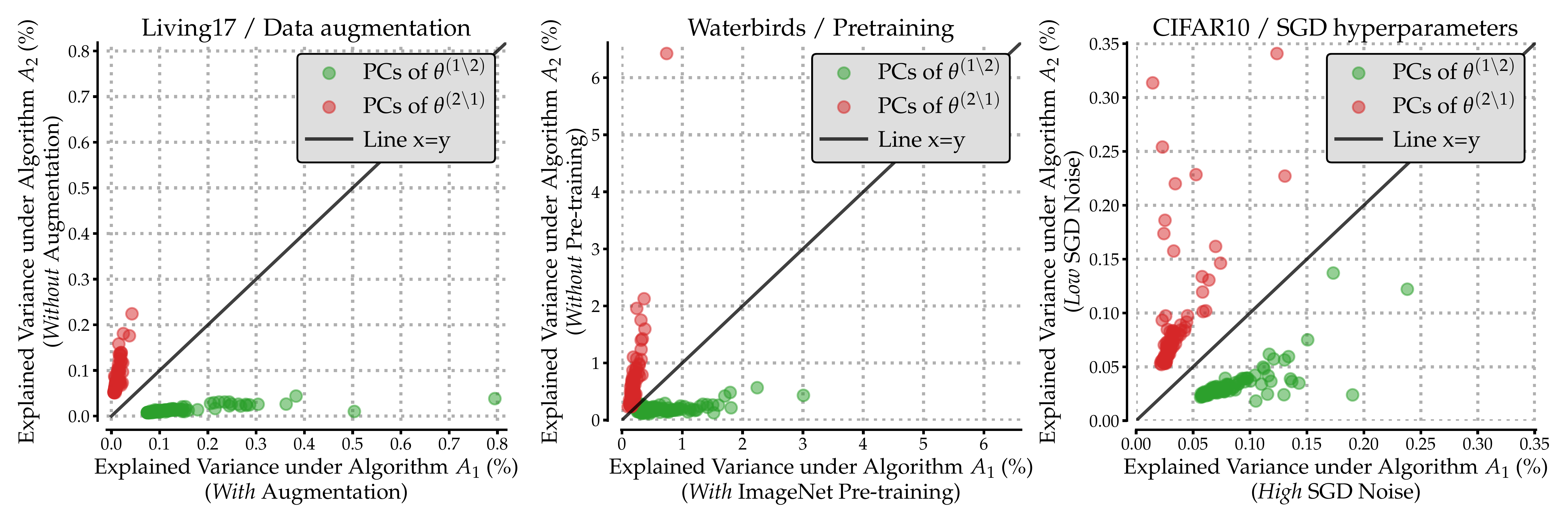}
    \caption{\textbf{Explained variance of residual datamodels' principal components.} Highlighted in green (resp. red), the top-$100$ PCs of residual datamodel $\smash{\resdmab}$ (resp. $\smash{\resdmba}$) explain a larger (resp. smaller) fraction of datamodel variance under algorithm $\mathcal{A}_1$ than under algorithm $\mathcal{A}_2$ across all three case studies.}
    \label{fig:more-ev}
\end{figure}

\clearpage
\subsection{Effect of sample size on datamodel estimation}
\label{app:samp_size}

In this section, we analyze the effect of sample size on datamodel estimation.

\paragraph{Setup.}
Recall from~\Cref{app:datamodels} that a datamodel training set of size $m$ corresponds to training $m$ models on independently sampled training data subsets.
For our case study on ImageNet pre-training in~\Cref{ssec:pretrain}, we estimate datamodels on \textsc{Waterbirds} data with $50,000$ samples (i.e., $m$ ResNet50 models trained on random subsets of the \textsc{Waterbirds} training dataset).
In this experiment, we analyze how the estimated datamodels vary as a function of sample size $m \in \{5000, 10000, 25000, 50000\}$.

\paragraph{Cosine similarity between datamodels.}
Our algorithm comparisons framework uses normalized datamodel representations to compute distinguishing training directions in the first step.
So, we first analyze the alignment between datamodel representations that are estimated with different sample sizes.
Specifically, we evaluate the cosine similarity between $\smash{\theta^{(m_1)}_x}$ and $\smash{\theta^{(m_2)}_x}$, where vector $\smash{\theta^{(m)}_x} \in \mathbb{R}^{|\text{S}|}$ corresponds to the linear datamodel for example $x$ estimated with $m$ samples.
As shown in~\Cref{fig:ss_cs}, the average cosine similarity between datamodels is greater than $0.9$ even when the sample size is reduced by a factor of $10$, from $50000$ to $5000$.

\paragraph{Explained variance of principal components.}
As discussed in~\Cref{sec:examples}, for a given training direction $v$,
the fraction of variance that $v$ explains in datamodel representations
$\smash{\{\theta_x^{(i)}\}}$
captures the importance of direction $v$ (i.e., weighted combination of training
examples) to the predictions of models trained using algorithm $\mathcal{A}_i$.
Here, we show that principal components of datamodel representations trained with smaller sample size (e.g., $m=5000$) have similar explained variance on datamodel representations estimated with larger sample size, and vice-versa.
As shown in~\Cref{fig:ss_ev}, the explained variance of the top-$10$ principal components of datamodels estimated with $m \in \{5000,50000\}$ have similar explained variance on datamodels estimated with $m \in \{5000,10000,25000,50000\}$.

\begin{figure}[b]
	\centering
    \includegraphics[width=0.8\textwidth]{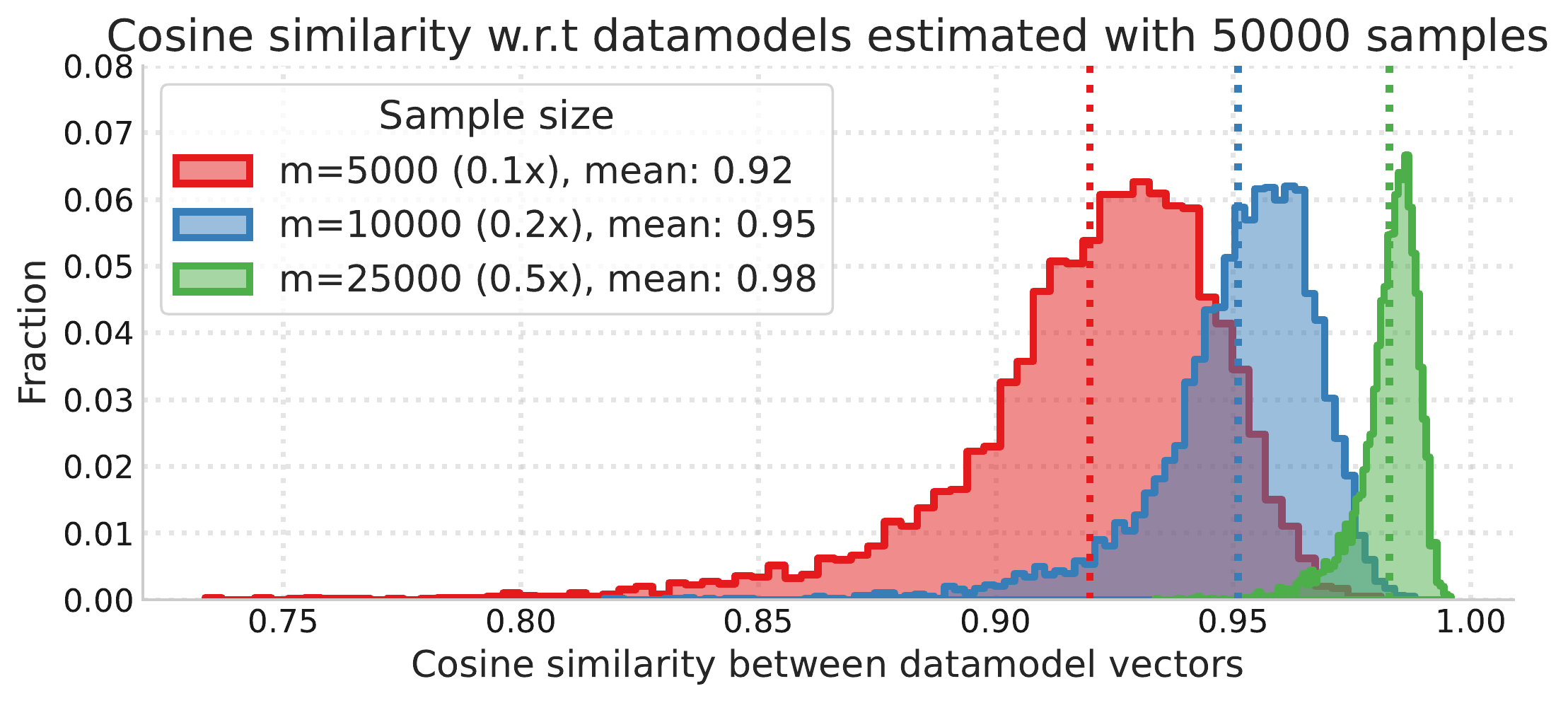}
    \caption{Histogram over cosine similarity between datamodels $\smash{\theta^{(m_1)}_x}$ and $\smash{\theta^{(m_2)}_x}$, where vector $\smash{\theta^{(m)}_x} \in \mathbb{R}^{|\text{S}|}$ corresponds to the linear datamodel for example $x$ estimated with $m \in \{5000, 10000, 25000, 5000\}$ samples.}
    \label{fig:ss_cs}
\end{figure}

\begin{figure}[h]
	\centering
    \includegraphics[width=\textwidth]{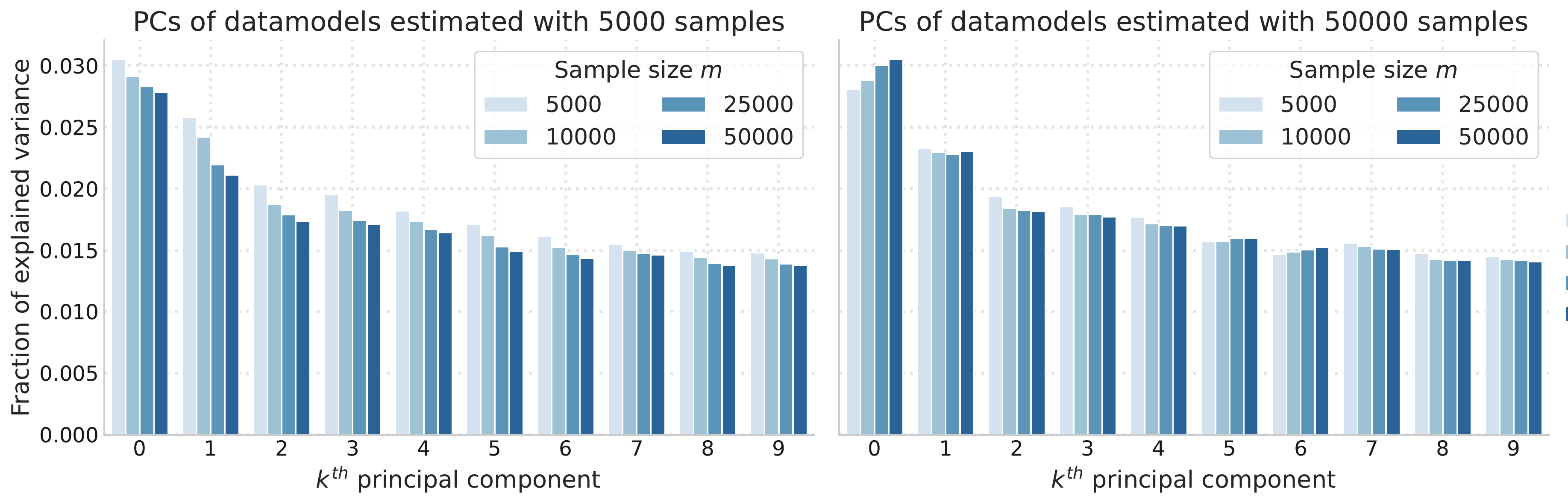}
    \caption{Explained variance of the top-$10$ principal components of datamodels estimated with $m \in \{5000,50000\}$ have similar explained variance on datamodels estimated with sample size $m \in \{5000,10000,25000,50000\}$. }
    \label{fig:ss_ev}
\end{figure}

\clearpage
\subsection{Effect of prediction-level disagreement on distinguishing training directions}
\label{app:disagree}

In this section, evaluate whether differences between algorithms at the model prediction level have a significant effect on the distinguishing training directions surfaced using our framework.
For context, we note that there are no existing methods that analyze prediction-level differences for algorithm comparisons.

We design an experiment to show that example-level differences in predictions of models trained with different algorithms are not necessary to identify subpopulations analysed in our case studies.
The first step of this experiment is to re-run our framework (a) on all test examples and (b) only on test examples on which models trained with different algorithm have the same prediction ``mode'' (taken over multiple runs).
In the second step, we directly compare the alignment between distinguishing training directions before and after controlling for prediction-level differences.

Our results in~\Cref{table:dataset} show that for each case study, our framework identifies similar training directions (i.e., high cosine similarity) even after removing test examples on which model predictions differ on average.
This experiment shows that our framework can identify fine-grained differences between learning algorithms that persist even after controlling for prediction-level disagreement across models trained with different algorithms.

\begin{table}[h]
\centering
\begin{tabular}{llc}
\toprule
\textbf{Dataset / Case study} & Direction & (Absolute) Cosine Similarity \\
\midrule
Living17 / Data augmentation & \textbf{A} (Spider web) & 0.999 \\
 & \textbf{B} (Polka dots) & 0.998 \\ \midrule
Waterbirds / ImageNet pre-training & \textbf{A} (Yellow color) & 0.977 \\
 & \textbf{B} (Human face) & 0.740 \\ \midrule
CIFAR-10 / SGD hyperparameters & \textbf{A} (Black-white texture) & 0.998 \\
 & \textbf{B} (Rectangular shape) & 0.999 \\
\bottomrule
\end{tabular}
\caption{Distinguishing training directions obtained before and after  filtering out high-disagreement test examples (a) exhibit high cosine similarity and (b) surface subpopulations of images that share the same distinguishing feature.}
\label{table:preds}
\end{table}

\clearpage
\subsection{Leveraging \textsc{CLIP} to analyze distinguishing subpopulations}
\label{app:clip}

As discussed in~\Cref{sec:approach} and~\Cref{app:hooman}, we infer distinguishing features candidates through manual inspection of distinguishing subpopulations.
In this section, we demonstrate that for image classifiers, shared vision-language models such as \textsc{CLIP}~\citep{radford2021learning} provide a streamlined alternative to manual inspection of distinguishing subpopulations.

\paragraph{Approach.} Before we describe our approach, note that \textsc{CLIP} is a contrastive learning method that embeds text and natural language into a shared embedding space.
Our approach leverages CLIP embeddings to identify multiple \emph{distinguishing captions}---representative descriptions that best contrast a given subpopulation of images from a set of  images sampled from the same distribution.
In the context of our framework, the \textsc{CLIP}-based approach takes as inputs a distinguishing training direction $v$, a set of images $\mathcal{D}$, and a set of captions $\mathcal{S}$\footnote{We use a filtered list of roughly 20,000 most common English words in order of frequency, taken from \url{https://github.com/first20hours/google-10000-english}.}, and outputs a set of distinguishing captions $\mathcal{S'} \in \mathcal{S}$ in four steps:
\begin{itemize}
	\item \emph{Pre-compute image and text embeddings}. Use the image encoder of a \textsc{CLIP} model to compute a set of normalized embeddings for all images in $\mathcal{D}$. Analogously, use the text encoder of a \textsc{CLIP} model to compute a set of normalized embeddings for all captions in $\mathcal{S}$.
	\item \emph{Record image-text pairwise cosine similarity}. Let vector $C_i \in \mathcal{R}^{|\mathcal{S}|}$ denote the pairwise cosine similarity between the embedding of image $i \in \mathcal{D}$ and all captions $j \in \mathcal{S}$.
	\item \emph{Compute mean cosine similarity over dataset and top-$k$ subpopulation}. Compute the mean cosine similarity vector $\bar{C}=\frac{1}{n}\sum_{i \in \mathcal{D}} C_i$ over all images in $\mathcal{D}$. Similarly, given distinguishing training direction $v$, compute the mean cosine similarity vector $C^{(v)}$ over the top-$k$ images whose residual datamodel vectors are most aligned with $v$.
	\item \emph{Extract distinguishing captions $\mathcal{S'}$}. Use cosine similarity vectors $\bar{C}$ and $C^{(v)}$ to extract captions in $\mathcal{S}$ that have the maximum difference between $C^{(v)}_i$ and $\bar{C}_i$.
\end{itemize}
Intuitively, the set of distinguishing captions $\mathcal{S'}$ correspond to representative captions (or, descriptions) that best contrast the top-$k$ images surfaced by distinguishing direction $v$ from the dataset.

\paragraph{Results.}
We now apply this approach to our case study on ImageNet pre-training, where we compare \textsc{Waterbirds} models trained with and without ImageNet pre-training (see \Cref{sec:examples}).
Specifically, we evaluate whether the \textsc{CLIP}-based approach surfaces distinguishing captions that are similar to distinguishing features ``yellow color'' (direction $\textbf{A}$) and ``human face'' (direction $\textbf{B}$) inferred via manual inspection.
\Cref{fig:clip_yellow} illustrates that for direction $\textbf{A})$, the \textsc{CLIP}-based approach highlights distinguishing captions such as \texttt{yellow}, \texttt{lemon}, and \texttt{sulphur}, all of which are similar to the ``yellow color'' feature that we infer via manual inspection.
Similarly,~\Cref{fig:clip_faces} shows that the distinguishing captions for direction $\textbf{B}$ (e.g., \texttt{florist}, \texttt{faces}, \texttt{counselors}) are similar to the identified ``human face'' feature.

To summarize, we show how the verification step can be easily \emph{specialized} to comparisons of vision classifiers trained on ImageNet-like data via vision-language embeddings such as \textsc{CLIP}.

\begin{figure}[h]
	\centering
	\includegraphics[width=\textwidth]{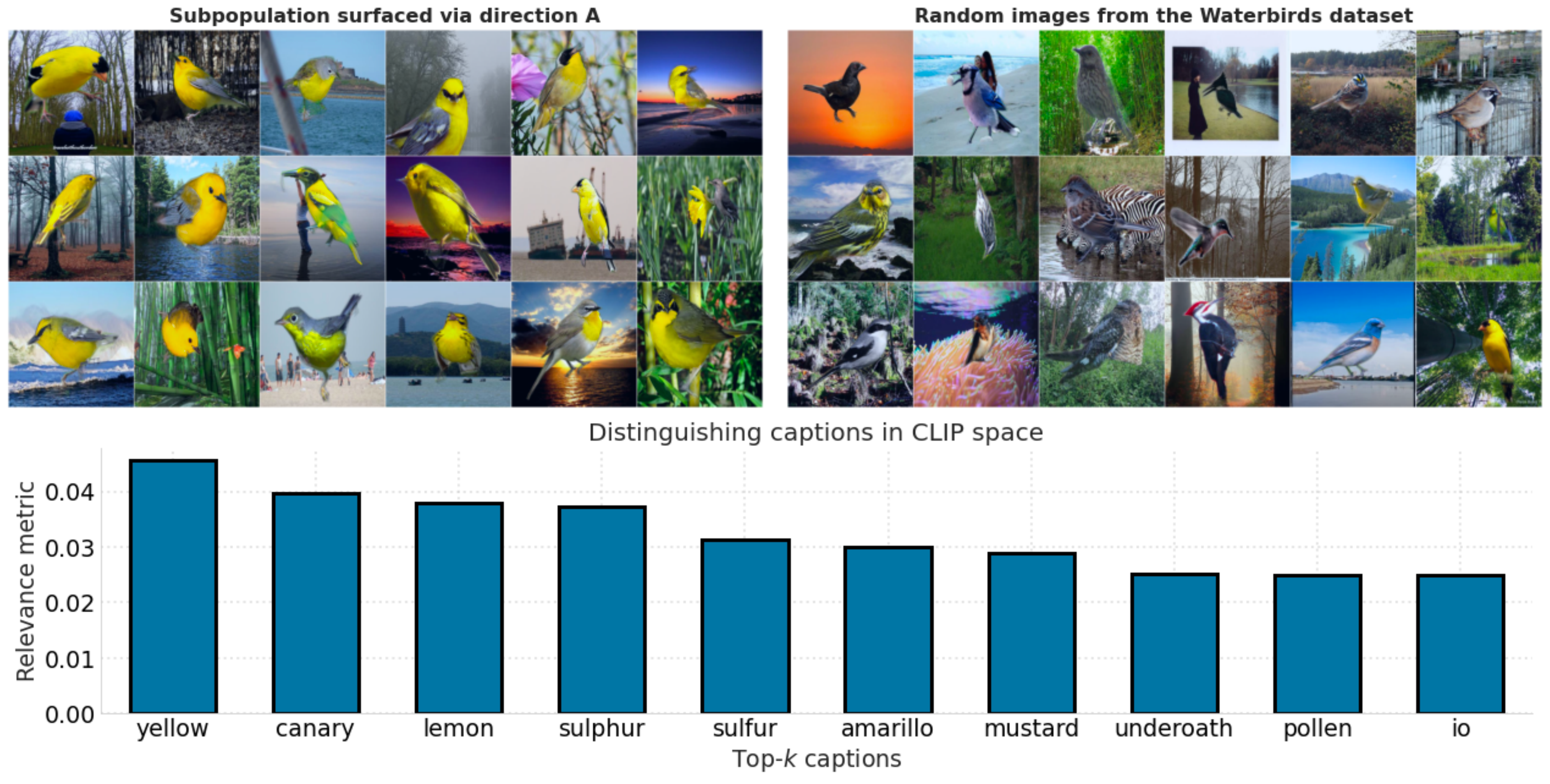}
    \caption{\textbf{Direction A}. The \textsc{CLIP}-based approach extracts distinguishing captions such as \texttt{yellow}, \texttt{lemon}, and \texttt{sulphur}, all of which contrast the residual subpopulation on the left to a set of random images from the \textsc{Waterbirds} dataset on the right. These distinguishing captions match the ``yellow color'' feature that we infer via manual inspection of the distinguishing subpopulation in~\Cref{ssec:pretrain}.}
    \label{fig:clip_yellow}
\end{figure}

\begin{figure}[h]
	\centering
	\includegraphics[width=\textwidth]{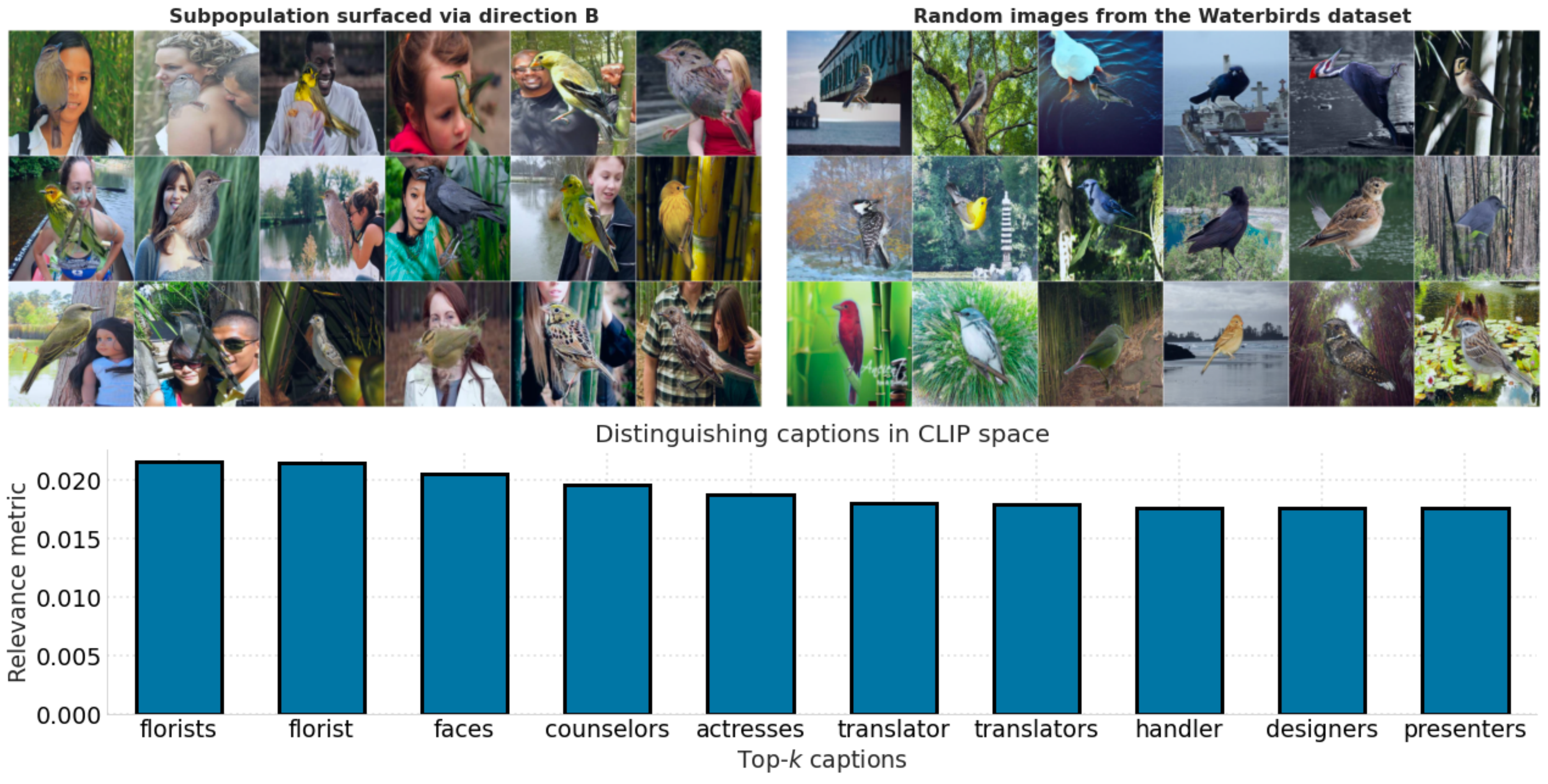}
    \caption{\textbf{Direction B}. The \textsc{CLIP}-based approach extracts distinguishing captions such as \texttt{florists}, \texttt{faces}, and \texttt{counselors}, all of which contrast the residual subpopulation (left) of images with human faces in the background to a set of random images (right) from the \textsc{Waterbirds} dataset. These distinguishing captions match the ``human face'' feature that we infer and counterfactually verify via manual inspection of the distinguishing subpopulation in~\Cref{ssec:pretrain}.
}
    \label{fig:clip_faces}
\end{figure}

\clearpage
\subsection{Subpopulations surfaced by principal components of residual datamodels}

Recall that our framework identifies distinguishing subpopulations via principal components (PCs) of residual datamodels.
Specifically, these subpopulations correspond to test examples whose residual datamodel representations have the most positive (top-$k$) and most negative (bottom-$k$) projection onto a given PC.
Here, we show that the top-$k$ and bottom-$k$ subpopulations corresponding to the top few PCs of residual datamodels considered in~\Cref{sec:examples} surface test examples with qualitatively similar properties.

\begin{figure}[h]
	\centering
    \includegraphics[width=0.76\textwidth]{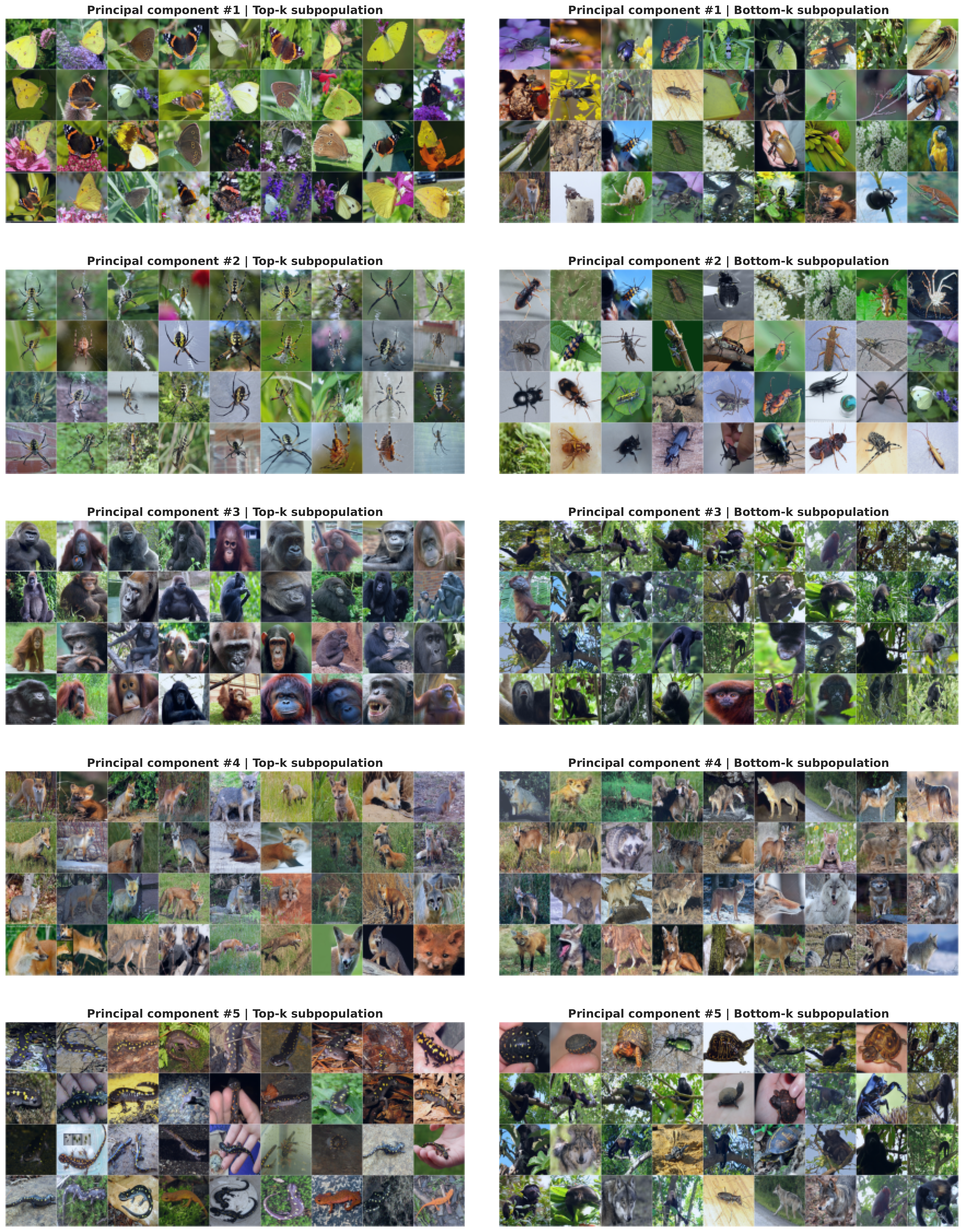}
    \caption{Top five PC subpopulations of \textsc{Living17} residual datamodel $\smash{\resdmab}$, where learning algorithms $\mathcal{A}_1$ and $\mathcal{A}_2$ correspond to training models with and without standard data augmentation respectively. Our case study in~\Cref{ssec:data_aug} analyzes PC \#2 (direction $\textbf{A}$) and PC \#5 (direction $\textbf{B}$).}
    \label{fig:subpop_living17}
\end{figure}

\begin{figure}[h]
	\centering
    \includegraphics[width=\textwidth]{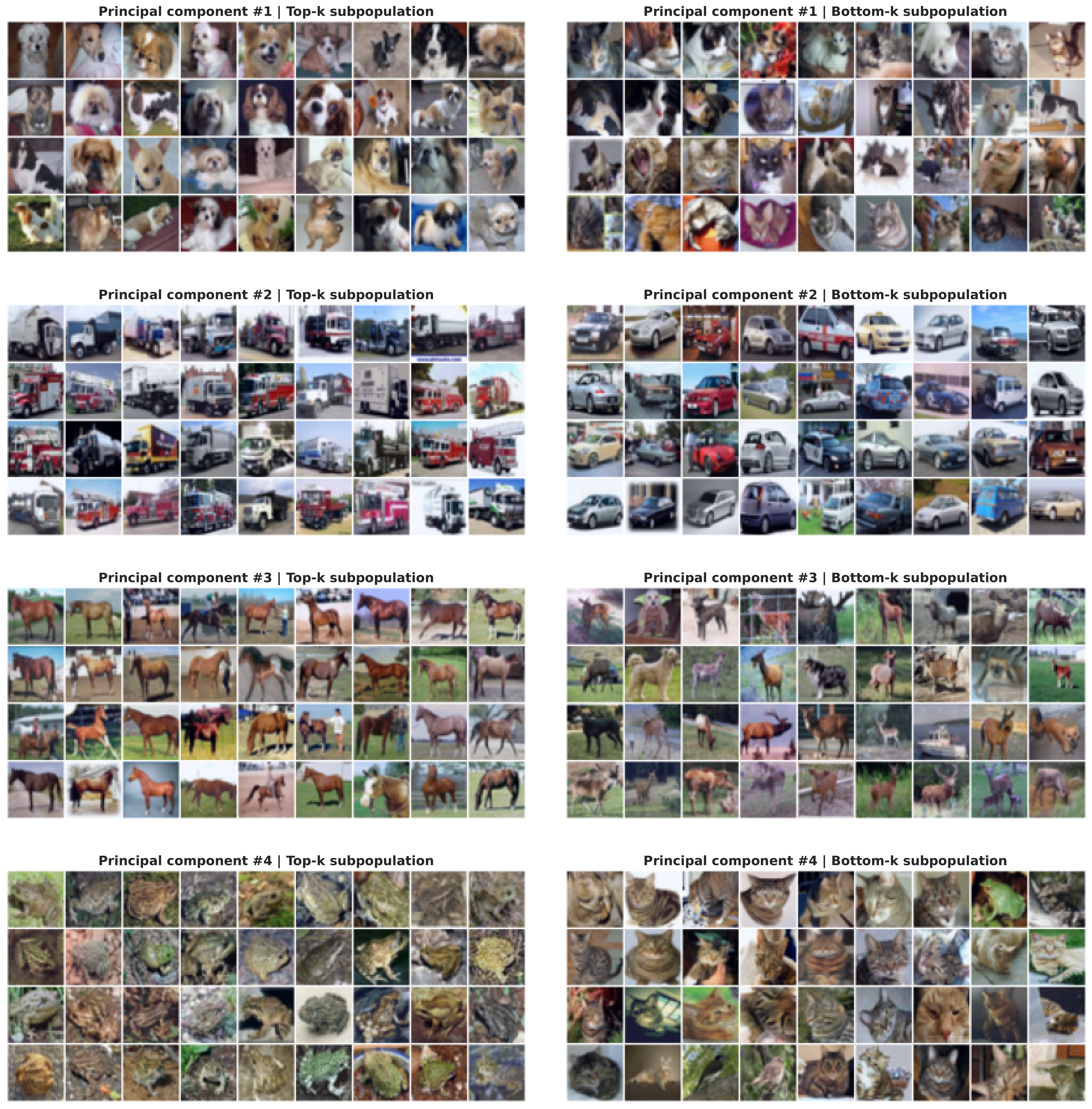}
    \caption{Top four PC subpopulations of \textsc{Cifar-10} residual datamodel $\smash{\resdmba}$, where learning algorithms $\mathcal{A}_1$ and $\mathcal{A}_2$ correspond to training models with high and low SGD noise respectively. Our case study in~\Cref{ssec:lr} analyzes PC \#1 (direction $\textbf{A}$) and PC \#2 (direction $\textbf{B}$).}
    \label{fig:subpop_cifar}
\end{figure}

\begin{figure}[h]
	\centering
	\includegraphics[width=\textwidth]{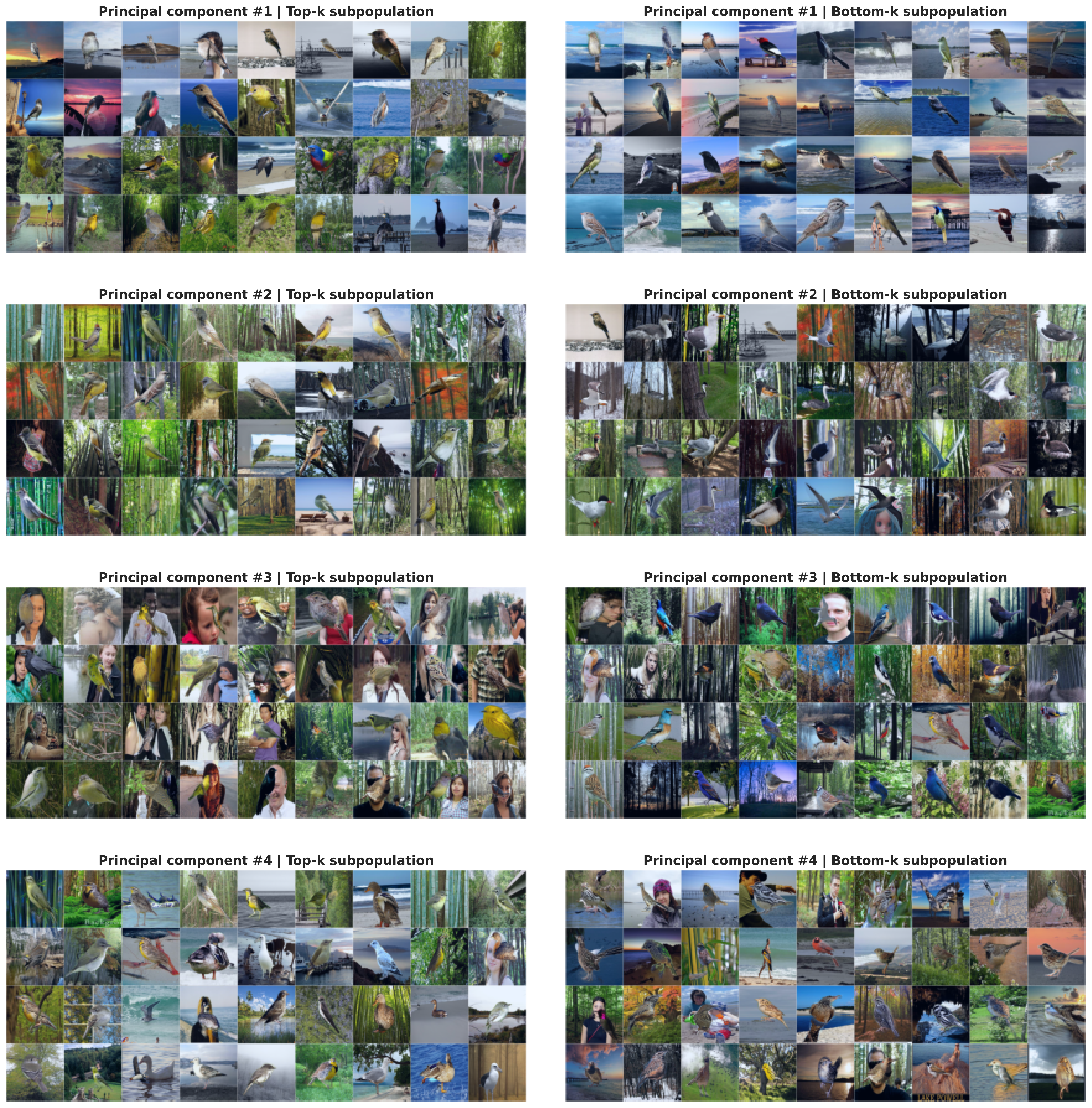}
    \caption{Top four PC subpopulations of \textsc{Waterbirds} residual datamodel $\smash{\resdmab}$, where learning algorithms $\mathcal{A}_1$ and $\mathcal{A}_2$ correspond to training models with and without ImageNet pre-training respectively. Our case study in~\Cref{ssec:data_aug} analyzes PC \#3 (direction $\textbf{A}$).}
    \label{fig:subpop_waterbirds_in}
\end{figure}

\begin{figure}[h]
	\centering
    \includegraphics[width=\textwidth]{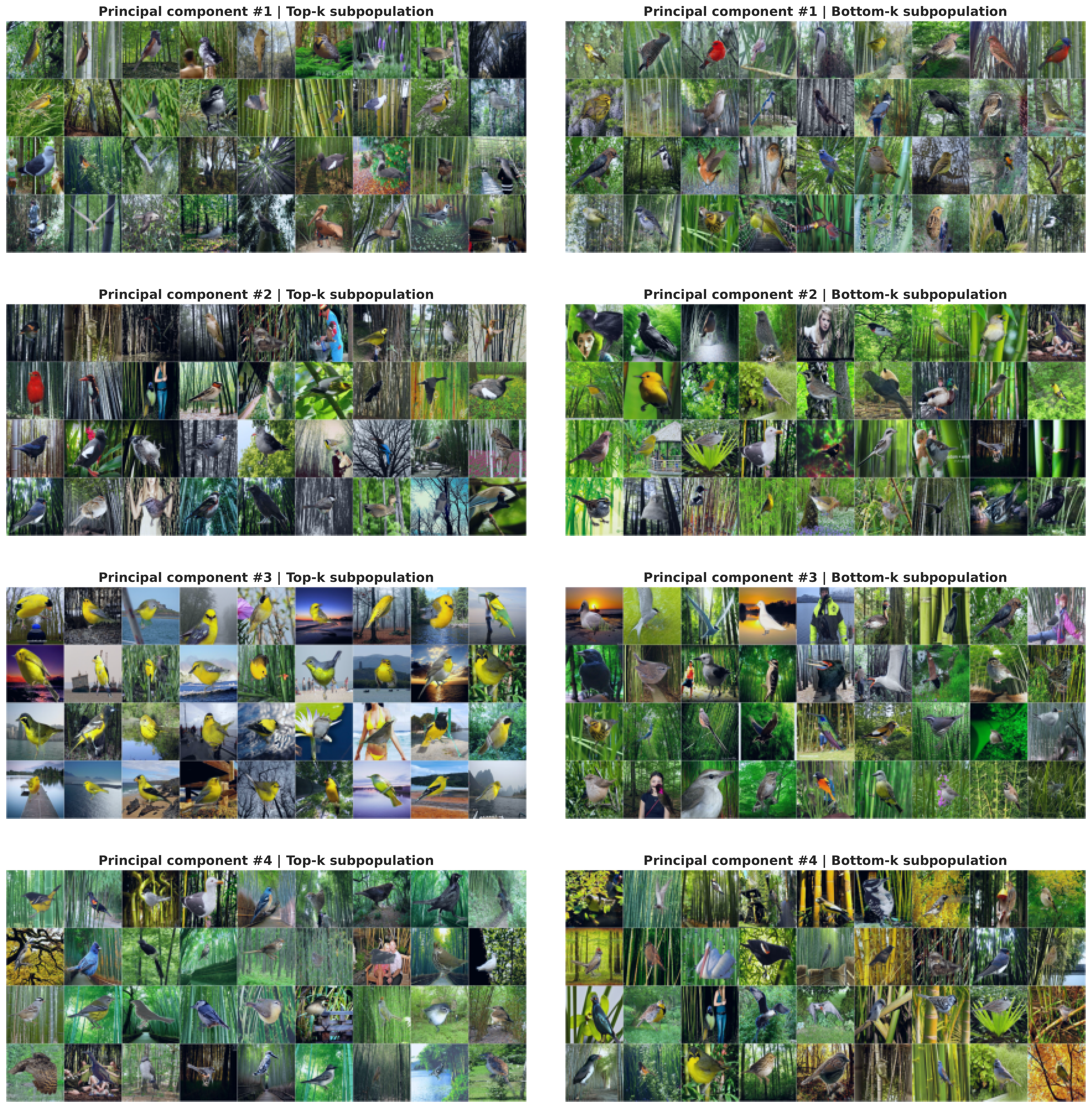}
    \caption{Top four PC subpopulations of \textsc{Waterbirds} residual datamodel $\smash{\resdmba}$, where learning algorithms $\mathcal{A}_1$ and $\mathcal{A}_2$ correspond to training models with and without ImageNet pre-training respectively. Our case study in~\Cref{ssec:data_aug} analyzes PC \#3 (direction $\textbf{B}$).}
    \label{fig:subpop_waterbirds_ri}
\end{figure}